\documentclass{article}
\usepackage{graphicx} 
\usepackage{fullpage}
\usepackage[numbers, sort&compress]{natbib}
\usepackage[hyphens]{url}
\usepackage{hyperref}
\usepackage{authblk}
\usepackage[dvipsnames]{xcolor}
\usepackage{subfigure}
\usepackage{rotating}   
\usepackage{caption}    
\usepackage{algorithm}
\usepackage{algorithmic}
\usepackage{booktabs}
\usepackage[export]{adjustbox}
\usepackage{array} 

\usepackage{amsmath}
\usepackage{amssymb}
\usepackage{mathtools}
\usepackage{amsthm}
\theoremstyle{plain}
\newtheorem{theorem}{Theorem}[section]

\theoremstyle{definition}
\newtheorem{definition}[theorem]{Definition}

\theoremstyle{remark}

\usepackage{pifont}
\definecolor{darkgreen}{RGB}{0,128,0}
\newcommand{\tick}{\textcolor{darkgreen}{\ding{51}}}  
\newcommand{\cross}{\textcolor{red}{\ding{55}}}  

\title{Safe Fairness Guarantees Without Demographics in Classification: Spectral Uncertainty Set Perspective}
\author[1]{Ainhize Barrainkua} 
\author[1]{Santiago Mazuelas} 
\author[1,2,3]{Novi Quadrianto}
\author[1,4]{Jose Antonio Lozano} 

\affil[1]{Basque Center for Applied Mathematics (BCAM), \texttt{\{abarrainkua,smazuelas\}@bcamath.org}}
\affil[2]{Predictive Analytics Lab, University of Sussex, \texttt{n.quadrianto@sussex.ac.uk}}
\affil[3]{Monash University Indonesia}
\affil[4]{University of the Basque Country EHU, \texttt{ja.lozano@ehu.eus}}

\date{}

\begin{document}

\maketitle

\begin{abstract}
As automated classification systems become increasingly prevalent, concerns have emerged over their potential to reinforce and amplify existing societal biases. 
In the light of this issue, many methods have been proposed to enhance the fairness guarantees of classifiers. 
Most of the existing interventions assume access to group information for all instances, a requirement rarely met in practice.
Fairness without access to demographic information has often been approached through robust optimization techniques, which target worst-case outcomes over a set of plausible distributions known as the uncertainty set.
However, their effectiveness is strongly influenced by the chosen uncertainty set.
In fact, existing approaches often overemphasize outliers or overly pessimistic scenarios, compromising both overall performance and fairness.
To overcome these limitations, we introduce \texttt{SPECTRE}, a minimax-fair method that adjusts the spectrum of a simple Fourier feature mapping and constrains the extent to which the worst-case distribution can deviate from the empirical distribution.
We perform extensive experiments on the American Community Survey datasets involving 20 states. 
The safeness of \texttt{SPECTRE} comes as it provides the highest average values on fairness guarantees together with the smallest interquartile range in comparison to state-of-the-art approaches, even compared to those with access to demographic group information.
In addition, we provide a theoretical analysis that derives computable bounds on the worst-case error for both individual groups and the overall population, as well as characterizes the worst-case distributions responsible for these extremal performances.
\end{abstract}

\section{Introduction}

In today's era of data-driven decision-making, machine learning (ML) models have become invaluable tools across various domains, from lending and hiring to criminal justice and healthcare. However, their widespread adoption has raised concerns about fairness, as these models have the potential to perpetuate and even exacerbate existing biases in society. The inadvertent amplification of societal biases and the potential for discriminatory outcomes have highlighted the urgent need to develop strategies that promote fairness within these systems. According to the Article 21 of the EU Charter of Fundamental Rights ``any discrimination based on any ground such as sex, race, color, ethnic or social origin, genetic features, language, religion or belief, political or any other opinion, membership of a national minority, property, birth, disability, age or sexual orientation shall be prohibited.'' 
Consequently, addressing fairness in the context of ML, particularly in classification tasks, remains a critical challenge requiring innovative solutions. This is especially true given that many existing approaches assume access to sensitive information for every instance in the dataset, an assumption that frequently conflicts with real-world scenarios.
In fact, individuals may intentionally withhold such information and, in certain applications, it is not only illegal \citep{elliott2008new} but also unethical to gather sensitive information.

Current approaches that tackle fairness concerns without relying on demographics can be divided into two main types: those employing a proxy sensitive attribute and those adhering to minimax fairness principles. Methods 
in the first family \citep{yan2020fair, grari2021fairness} leverage surrogate group information, such as clustering details and estimated group information, to enforce fairness within these surrogate group partitions. However, a drawback of using proxy attributes is the underlying assumption that the clustering or surrogate group information is correlated with the sensitive attribute(s) of interest. Consequently, achieving fairness across multiple sensitive attributes may be challenging when there is a significant distributional discrepancy among them. 

On the other hand, minimax fairness approaches (e.g. \cite{hashimoto2018fairness, lahoti2020fairness, martinez2021blind}) usually rely on robust optimization frameworks, focusing on improving the worst-case utility. Theoretically, these approaches have several benefits. First, the robust formulation aligns with the concept of Rawlsian minimax fairness \cite{martinez2020minimax}, since the optimization focuses on the worst-case scenario concerning the objective function. Second, improving the predictive performance for the worst-performing group has the property that any model that achieves it Pareto dominates (at least weakly, and possibly strictly) an equalized-error model with respect to group error rates \cite{martinez2020minimax}. In practice, however, improving the worst accuracy among all sensitive groups may be too stringent to achieve significant improvements \cite{HuNiuSatSug18}.  Such considerations lead to hybrid approaches that combine minimax approaches with proxy sensitive attributes, i.e., optimising worst-case performance over a set of subgroups defined by data-driven partitioning techniques based on input features only \cite{SohDunAngGuetal20} or based on both label and input features \cite{MarBerPapRodetal21}.

This paper aims to deepen the understanding of the robust optimization framework underlying minimax fairness approaches—such as, GDRO \cite{sagawa2020distributionally}, MMPF \cite{martinez2020minimax, diana2021minimax}, RLM \cite{hashimoto2018fairness}, ARL \cite{lahoti2020fairness}, and BPF \cite{martinez2021blind} (for a detailed comparison of how these approaches instantiate the characteristics of the robust optimization framework, see Table \ref{tab:summary_rrm})—with a particular focus on fairness without demographic information. Moreover, it addresses a key limitation of existing methods: their tendency to overemphasize outliers, resulting in overly pessimistic classification rules. This often results in low overall and group-specific accuracy, sometimes performing worse than basic methods like standard empirical risk minimization (ERM) with tree-based models \cite{gardner2022subgroup,HarMatVolRumetal23}. Our analysis and the proposed \texttt{SPECTRE} -- SPECTral uncertainty set for Robust Estimation --  approach is designed to overcome this pessimism and improve performance, all without access to sensitive demographic information. \texttt{SPECTRE} is an extension of the Minimax Risk Classifier (MRC) \cite{mazuelas2023minimax} that leverages the favorable properties of the random Fourier basis function mapping for minimax fairness. It operates by mapping the data using these functions and modifying the spectrum to retain only the essential response frequencies. Crucially, \texttt{SPECTRE} constrains the worst-case distribution's deviation from the empirical distribution in the frequency domain, using this adjustment to the MRC's uncertainty set to attain optimal fairness guarantees, measured by worst-group accuracy. 

The reminder of the paper is organized as follows. In Section \ref{sec:rel_works} we discuss the related works. We present the theoretical context for Robust Risk Minimization (RRM) in Section \ref{sec:background_spectre}. In Section \ref{sec:spectre} we introduce \texttt{SPECTRE} and Section \ref{sec:bounds} contains further theoretical analysis on the guarantees of the proposed approach.
Section \ref{sec:spectre_experiments} presents the results of comprehensive experiments on real-world data, highlighting the proposed method's effectiveness in providing safer fairness guarantees in terms of worst-group accuracy than SOTA approaches, while making only a minimal trade-off in overall accuracy. Finally, in Section \ref{sec:conclusions_spectre} we draw the key conclusions, identify the main limitations and outline possible areas of improvement.

\section{Related Works}
\label{sec:rel_works}

\subsection{Algorithmic Fairness}
\label{ap:rel_fairness}

In recent years, algorithmic fairness has gained growing attention in the ML community as awareness has increased about the risk of biased decision-making when models are trained on data reflecting societal inequalities.

\subsubsection{Measuring Algorithmic Fairness}

Algorithmic fairness can be assessed through various metrics, which are commonly categorized into three main groups: individual \cite{dwork2012fairness, li2023accurate, wang2024individual}, statistical (or group) \cite{kamishima2012fairness, hardt2016equality, zafar2017parity, chouldechova2017fair, kleinberg2017inherent, agarwal2018reductions, donini2018empirical, berk2021fairness}, and mini-max fairness. Individual fairness seeks to ensure that similar individuals are treated similarly, which requires a well-defined notion of similarity between instances. However, implementing it can be difficult in practice due to the complexity of defining similarity between real individuals and the high computational cost. In contrast, statistical fairness, evaluates classifier performance across different sensitive groups, offering a more efficient framework. Nonetheless, many of these group fairness definitions are mutually incompatible. 
Besides, enforcing parity can sometimes result in degrading performance for well-performing (often privileged) groups rather than improving outcomes for disadvantaged ones  \cite{martinez2020minimax, diana2021minimax, diana2025minimax}. To address these limitations, minimax fairness has emerged as a compelling alternative: it aims to reduce the error of the worst-performing group.

\subsubsection{Fairness-Enhancing Interventions}

To promote fairness in ML systems \cite{rabonato2025systematic}, researchers have developed a variety of interventions, typically categorized as pre-, in-, or post-processing methods. Pre-processing mitigates bias in the input data, in-processing embeds fairness into the learning algorithm via modified losses or constraints, and post-processing adjusts predictions after training to satisfy fairness criteria. Most of these approaches, however, assume that sensitive information is available for all training instances, an assumption rarely met in practice.

\subsection{Algorithmic Fairness in the (Partial) Absence of Sensitive Information}

Considering the practical limitations of accessing demographic attributes in real-world scenarios, numerous fairness-aware methods have been proposed to operate under partially observed or entirely unknown sensitive information. Some approaches train proxy classifiers to estimate the sensitive attribute \cite{awasthi2021evaluating, kallus2022assessing, diana2022multiaccurate, liang2023fair}, leveraging correlations with non-sensitive features \cite{hajian2012methodology, kilbertus2017avoiding}. While effective, this raises ethical and legal concerns, as it involves constructing a classifier with the purpose of predicting an individual's sensitive information. In contrast, other works assume access to sensitive data only in the validation set and use it to guide fairness-enhancing training without inferring sensitive labels for unseen instances \cite{nam2020learning, liu2021just, taghanaki2021robust, pezeshki2021gradient}. 

More recently, fully agnostic methods have emerged that improve fairness guarantees without relying on any demographic information \cite{hashimoto2018fairness, lahoti2020fairness, yan2020fair, chai2022fairness}. Notably, \citet{hashimoto2018fairness} introduced a DRO-based framework that reweights high-loss instances and \citet{lahoti2020fairness} proposed an adversarially-guided reweighting.  These methods, along with the proposal by \citet{martinez2021blind}, typically rely on surrogate losses (e.g., log-loss) and re-weighing based uncertainty sets, often leading to overly pessimistic classifiers \cite{hu2018short}. Other techniques, like clustering-based resampling \cite{yan2020fair} or teacher-student distillation \cite{chai2022fairness}, offer alternatives but lack theoretical guarantees for generalization. In contrast, our proposed framework introduces a novel RRM approach that optimizes the 0-1 loss directly, defines richer uncertainty sets beyond reweighting, and provides out-of-sample generalization guarantees, achieving minimax fairness without relying on demographic information.

\subsection{Robust Risk Minimization Based Approaches in Algorithmic Fairness}

RRM has been extensively applied in algorithmic fairness, particularly for optimizing minimax fairness objectives \cite{sagawa2020distributionally, martinez2020minimax, diana2021minimax, jung2022re}. While effective, these approaches require access to sensitive attributes for all instances and rely on steep surrogate losses for tractability, making them prone to overfitting outliers and producing pessimistic models. Furthermore, they focus on minimizing the empirical worst-group risk, which does not guarantee favorable worst-group performance on unseen data. In contrast, our method employs RRM in its classical form, aiming to optimize the out-of-sample (OOS) worst-group error. Several works within the DRO framework have also explored OOS fairness guarantees \cite{mandal2020ensuring, rezaei2020fairness, taskesen2020distributionally, du2021robust, du2021fairness}, often modeling uncertainty via reweighted perturbations \cite{mandal2020ensuring, du2021fairness} or Wasserstein balls \cite{taskesen2020distributionally, du2021robust}. However, reweighting fails to capture richer uncertainty structures and is sensitive to outliers, while Wasserstein-based sets can overly constrain the learner, forcing simplistic solutions. Refinements such as marginal constraints \cite{taskesen2020distributionally} or tighter alternatives \cite{rezaei2020fairness} help but still rely on parity-based fairness notions and require full demographic knowledge at training time, limitations that our method overcomes.

\subsection{Spectral Analysis in ML}

Random Fourier Features (RFFs), initially introduced by \citet{rahimi2007random} to accelerate kernel methods, have since proven very relevant across broader domains. Notably, they have deepened our understanding of spectral bias in Deep Neural Networks (DNNs) \cite{RahBarArpDraetal19, xu2019training, xu2024overview, RonMeiAmnDavetal20}. This insight has motivated techniques that either leverage or mitigate this bias. For instance,  \citet{tancik2020fourier} used Fourier feature mappings to improve high-frequency function learning, while \citet{YanAjaAgr22} demonstrated improved convergence in reinforcement learning. In our work, we repurpose RFFs to shape the uncertainty set of a MRC for fairness-aware learning without sensitive demographic information. Unlike standard ERM frameworks, which rely on fixed uncertainty sets centered on empirical distributions, we propose modifying these sets in the frequency domain to adjust the modeled worst-case scenarios, thereby enhancing fairness guarantees.

\section{Theoretical Background}
\label{sec:background_spectre}

\textbf{Notation.} Vectors and matrices are represented by bold lowercase and uppercase letters, respectively, e.g. $\boldsymbol{v}$ and $\boldsymbol{M}$; the transpose of a vector is denoted by $\boldsymbol{v}^T$ and $v_i$ represents its $i$-th component; 
for a probability distribution $p$ over $\mathcal{X} \times \mathcal{Y}$, $p_x$ and $p_y$ stand for the marginal distributions over $\mathcal{X}$ and $\mathcal{Y}$, respectively, and $p_{x|y}$ represents the conditional distribution given $y\in \mathcal{Y}$; additionally, $\succeq$ and $\preceq$ indicate component-wise inequalities; $\mathbb{I}\{ \cdot \}$ denotes the indicator function; finally, $\Delta (\mathcal{Y})$ denotes the set of probability measures on $\mathcal{Y}$.  

\textbf{Problem formulation.} The goal of classification methods is to assign a label in a set $\mathcal{Y}$ to instances in a set $\mathcal{X}$, where $\mathcal{X} \subseteq \mathbb{R}^d$. 
Each instance is defined as a triplet $(\boldsymbol{x},\boldsymbol{s},y)$, where $\boldsymbol{x} \in \mathcal{X}$ represents the non-sensitive attributes,  $\boldsymbol{s} \in \mathcal{S}\subseteq \mathbb{R}^{d'}$ the sensitive demographic attributes and  $y \in \mathcal{Y}$ the class label. We denote by $\Delta(\mathcal{X} \times \mathcal{S} \times \mathcal{Y})$ the set of probability distributions $p$ on $\mathcal{X} \times \mathcal{S} \times \mathcal{Y}$. 
In a fairness-aware supervised classification setting, methods are usually trained on a finite set of $N$ samples $\{(\boldsymbol{x}^{i}, \boldsymbol{s}^{i}, y^{i}) \}_{i=1}^N$ i.i.d. from the true underlying distribution $p^*$ to develop classification rules that minimize the number of errors in prediction while also ensuring fairness guarantees. However, in reality, the access to the sensitive information is limited, due to legal restrictions or personal choices. Therefore, we assume that $s_i$ is not available for training nor for inference in deployment.

Our primary goal is to develop a classification rule that does not depend on sensitive group information for making predictions and has high predictive performance for all the subgroups of the population including  the worst-off group. It is essential to note that merely removing sensitive information does not eliminate bias, as other features may still be closely correlated with this information \cite{dwork2012fairness}. We denote by $T(\mathcal{X},\mathcal{Y})$ the set of all classification rules for which the sensitive information associated to each instance is not eligible to be used to make predictions.
A classification rule is a function mapping from instances to probability distributions on labels. Let $f(y|\boldsymbol{x})$ represent the probability that instance $\boldsymbol{x}$ is set to the label $y$ (if the classifier is deterministic, $f(y|\boldsymbol{x}) \in \{ 0,1 \}$). 

The error of a classification rule $f \in T(\mathcal{X}, \mathcal{Y})$ on the true underlying distribution $p^* \in \Delta(\mathcal{X} \times \mathcal{S} \times \mathcal{Y})$, denoted by $\mathcal{R}(f, p^*)$, is defined as:
\begin{equation}
    \mathcal{R}(f, p^*) = \mathbb{E}_{p^*(\boldsymbol{x},y)} \{ \mathcal{L}(f(\boldsymbol{x}),y) \}
\end{equation}
where $\mathcal{L} \ : \ \Delta
(\mathcal{Y}) \times \mathcal{Y} \rightarrow [0, \infty)$ is a loss function that captures the predictive performance of the classifier. Equivalently, the error of the classification rule $f \in T(\mathcal{X}, \mathcal{Y})$ on the true underlying population of subgroup $\boldsymbol{s} \in \mathcal{S}$,  $p^*_{|\boldsymbol{s}}$ is:
\begin{equation}
    \mathcal{R}(f, p_{|\boldsymbol{s}}^*) = \mathbb{E}_{p^* (\boldsymbol{x},y |S=\boldsymbol{s})} \{ \mathcal{L}(f(\boldsymbol{x}),y) \} .
\end{equation}

In general, $\mathcal{R}(f, p^*)$ cannot be directly optimized since $p^*$ is unknown in practice, 
access being restricted to a set of training samples i.i.d. derived from the actual distribution $p^*$. 

\textbf{Robust Risk Minimization (RRM).}
RRM approaches offer an alternative to deal with distributional uncertainty by optimizing the loss concerning a set of distributions, referred to as the \emph{uncertainty set} or ambiguity set, aligned to varying degrees with the empirical distribution. In essence, RRM approaches minimize the worst-case expected loss with respect to this uncertainty set of distributions~$\mathcal{U}$:
\begin{equation}
    \inf_{f \in T(\mathcal{X}, \mathcal{Y})} \sup_{p \in \mathcal{U}} \mathcal{R}(f, p).
    \label{eq:rrm}
\end{equation}

The RRM framework offers several advantages to enhance the fairness guarantees of classifiers, as RRM's and algorithmic fairness' objectives have some overlaps.
The primary aim of RRM-based approaches is to minimize the worst-case error across distributions close to the empirical distribution. In essence, the goal is to achieve optimal performance in challenging instances, closely aligning with the principles of minimax fairness.
That is, RRM optimizes the tail performance of the model without necessitating demographic information to safeguard against poor performance on data subsets.
Nonetheless, even though RRM  can protect against poor performance on subgroups of the population, solely considering an RRM framework does not inherently guarantee optimal fairness warrants \cite{gardner2022subgroup,HarMatVolRumetal23}. 
Indeed, the performance, and subsequently, the fairness guarantees of classifiers trained within the RRM framework, fluctuate based on the distributions considered within the uncertainty set. This variability arises from different realizations of worst-case scenarios within the uncertainty set, generated by amplifying the relevance of the most challenging instances. Consequently, a broader uncertainty set accentuates the impact of the most vulnerable instances. Yet, if the considered uncertainty set is excessively large, it may lead to highly suboptimal and potentially trivial decision rules. 

Most existing approaches for enhancing minimax fairness through robust optimization frameworks, whether they incorporate sensitive information (e.g., \citet{martinez2020minimax, sagawa2020distributionally, diana2021minimax}) or not (e.g., \citet{hashimoto2018fairness, lahoti2020fairness, martinez2021blind}), define the worst-case distribution as a re-weighting scheme of the empirical training distribution (with variations in how the re-weighting is performed). Moreover, these approaches typically lack explicit bounds on how far the worst-case distribution can deviate from the empirical distribution. The resulting classifiers often exhibit excessive pessimism and are primarily optimized for the empirical training distribution \cite{hu2018short}. This behavior arises from the use of surrogate losses during training for optimization tractability. 
%
In conventional classification tasks, classifier performance is typically evaluated on the deployment environment using the 0-1 loss. However, the mentioned approaches employ surrogate losses (e.g., log-loss) during training for computational feasibility. An optimal distribution according to a surrogate loss may not necessarily correspond to optimality under the 0-1 loss. Additionally, surrogate losses are often steeper, making the resulting classifiers more sensitive to outliers.

In order to overcome these limitations, our objective is to develop a robust optimization-based approach to enhance minimax fairness without relying on sensitive information (nor inferring the sensitive characteristics of individual instances), that guarantees optimal worst-group accuracy while maximizing predictive performance across the entire population. Specifically, we aim to:
\textbf{(a) Enhance Diversity in the Uncertainty Set:} Incorporate distributions beyond simple re-weighting schemes of the empirical distribution, enabling the capture of emerging vulnerable instances outside the empirical data at training time.  
\textbf{(b) Bound Conservativeness and Pessimism:} Reduce reliance on outliers and mitigate over-conservativeness by 
adjusting the response frequency spectrum and optimizing the maximum allowable divergence.

\section{\texttt{SPECTRE}: a SPECTral Uncertainty Set for Robust Estimation}
\label{sec:spectre}

In this section, we introduce \texttt{SPECTRE} 
a novel robust optimization based framework to learn classification rules that maximize predictive performance across the entire population and at subgroup level while remaining oblivious to sensitive information. 
Specifically, our approach involves a novel strategy to modify the uncertainty set of a MRC in the (response) frequency domain. 

\subsection{Robust Classifier with Uncertainty Sets Defined by a Feature Mapping}
\label{sec:spectral_mrc}

As mentioned in the previous section, RRM approaches strive to minimize the expected loss under the worst-case scenario, considering an uncertainty set of distributions (see Equation \ref{eq:rrm}). 
We are interested in the RRM framework called MRC \cite{mazuelas2023minimax} in which the uncertainty set is composed of distributions defined by linear constraints obtained from expectation estimates of a feature mapping. In particular, let $\Phi : \mathcal{X} \times \mathcal{Y} \rightarrow \mathbb{R}^m$ be a feature mapping blind to the sensitive information. The uncertainty set will consist of all distributions where the expected value of the feature mapping is within a distance of at most $\boldsymbol{\lambda} \in \mathbb{R}^m$ from the expected value of the feature mapping of the empirical distribution $\tilde{p}_N$: 
\begin{equation}
    \mathcal{U} = \{ p \in \Delta (\mathcal{X} \times \mathcal{Y}) \; : \; |\mathbb{E}_p\{ \Phi \} - \mathbb{E}_{\tilde{p}_N}\{ \Phi \}| \preceq \boldsymbol{\lambda} \},
    \label{eq:uncertainty}
\end{equation}
where the parameter $\boldsymbol{\lambda} \succeq \boldsymbol{0}$ determines the maximum deviation allowed from the expected value of the feature mapping in the empirical distribution, and consequently, the size of the uncertainty set. Building on the original work by \citet{mazuelas2023minimax}, we decompose the confidence vector as $\boldsymbol{\lambda} = \lambda_0 \sqrt{\frac{\textrm{var}_{\tilde{p}_N}{ \{ \boldsymbol{\Phi} \} }}{N}}$. This formulation assigns higher values to features with greater variance (resulting in looser constraints), and $\lambda_0$ acts as a hyperparameter.

The uncertainty set defined in Equation \ref{eq:uncertainty} offers several advantages: (a) It can contain the true underlying distribution with a
tunable confidence $\boldsymbol{\lambda}$ and thus offers out-of-sample generalization guarantees on group errors (Section \ref{sec:bounds});
(b) the optimization problem in Equation \ref{eq:rrm} is tractable using the 0-1 loss, avoiding the use surrogate losses and the resulting pessimism;
and (c) feature mappings are commonly used in machine learning to represent instance-label
pairs as real vectors. 
Based on (c), we shed light on how response frequencies affect classification rule fairness by using a random Fourier feature mapping (Sections \ref{sec:spectral_uncertainty} and \ref{sec:sigma_toy}). 
For a more detailed discussion on the uncertainty set and the worst-case realizations considered refer to Appendix \ref{ap:modify_uncertainty_set}.

\subsection{Spectral Uncertainty Set}
\label{sec:spectral_uncertainty}

We consider a feature mapping of the original data to a random Fourier basis \cite{rahimi2007random}. RFFs constitute a feature mapping where the resulting features can be seen as a decomposition of the original data into different frequency components. That is, these features characterize the spectral properties of the data. Low-frequency components correspond to patterns that are inherently easier to learn, while high frequency components refer to noisy information. 

According to RFFs, a given datapoint $(\boldsymbol{x}, y)$ is transformed as follows: 
\begin{equation}
\Phi(\boldsymbol{x},y) = \mathbf{e}_y \otimes\nonumber \sqrt{\frac{2}{D}}[\text{cos}(\boldsymbol{\omega}_1^\top \boldsymbol{x}),\text{sin}(\boldsymbol{\omega}_1^\top \boldsymbol{x}),\ldots,\text{cos}(\boldsymbol{\omega}_D^\top \boldsymbol{x}),\text{sin}(\boldsymbol{\omega}_D^\top \boldsymbol{x})],
\end{equation}
where $\mathbf{e}_y$ is one hot encoding of the class label, 
$\otimes\nonumber$ denotes the Kronecker product between $\mathbf{e}_y$ and the Fourier components, 
$D$ refers to the number of frequency components we are considering
and $\boldsymbol{\omega}_1,\ldots,\boldsymbol{\omega}_D$ are i.i.d. samples from a zero-mean Gaussian distribution with covariance matrix $\boldsymbol{\Sigma}$ 
\begin{equation}
    \boldsymbol{\omega} \sim \mathcal{N}(\boldsymbol{0}, \boldsymbol{\Sigma})
\end{equation}
where $\boldsymbol{\Sigma} = \boldsymbol{\sigma}^T \boldsymbol{I}$. In the absence of any strong prior on the frequency spectrum of the signal, an isotropic Gaussian distribution will be considered. That is, we consider $\boldsymbol{\Sigma} = \sigma^2 \boldsymbol{I}$, where $\sigma = \mathbb{R}_{+}$. 
In other words, RFFs are obtained by approximating the radial basis function kernel by the random Fourier feature map; thus, these RFFs characterize a randomized spectral representation of data.
Instead of using fixed basis functions like in classical Fourier analysis, RFF constructs a set of randomized Fourier basis functions.
The resulting feature space consists of a mix of different frequencies\footnote{Here, when we talk about response frequency we refer to how the output is affected by the input. In other words, the frequency in the coordinate $\boldsymbol{\omega}_j$ measures the rate of change of the feature mapping component $\Phi_j$ with respect to $\boldsymbol{x}$. For high values of $\boldsymbol{\omega}_j$, small changes of $\boldsymbol{x}$, induce large changes on the novel feature representation.}, capturing multi-scale structures of the data in an implicit manner. Besides, note that, by changing the value of the scaling parameter $\sigma$, we change the range of frequencies from which the samples are obtained.

RFFs have been used to address spectral bias of DNNs trained under ERM with sufficiently expressive hypothesis classes \cite{RahBarArpDraetal19, xu2019training}.
\emph{Spectral bias} refers to a phenomenon whereby fitting high-frequency components of a
function requires exponentially more gradient update steps than the low-frequency
ones \cite{RahBarArpDraetal19,RonMeiAmnDavetal20}.
In their work, \citet{tancik2020fourier} showed that passing input points through a simple Fourier feature mapping
enables complex ERM classifiers to learn high-frequency functions in low dimensional problem domains. 
Moreover, \citet{YanAjaAgr22} also showed that the introduction of the Fourier features enables faster convergence on higher frequency components of a value function, thereby improving the sample efficiency of off-policy reinforcement learning.

In this paper, we use RFFs to construct the uncertainty set of an MRC. 
To put in context, unlike RRM methods, traditional ERM approaches rely on uncertainty sets that include only the empirical distribution of the training samples, making them non-adjustable. We hypothesize that modifying the frequency components within the uncertainty sets allows us to influence the worst-case scenarios considered, thereby impacting the resulting classification rules and their fairness guarantees. Incorporating too many noisy components may led to worst-case scenarios that are not relevant or are unrealistic and led to classification rules that are too conservative and not useful in practice. So we want to achieve the best possible balance between conservativeness (inclusivity, all groups should be represented as well as possible) and utility. In the following section, we test this hypothesis using synthetic data by varying the scaling parameter $\sigma$ and the tunable confidence parameter $\lambda_0$.

\subsection{The Effect of the Scaling Parameter $\sigma$ and the Tunable Confidence $\lambda_0$}
\label{sec:sigma_toy}

This section analyses the effect of the scaling parameter $\sigma$ at a particular value for the confidence parameter $\lambda_0$ in the resulting decision rule, by means of a toy problem. The toy dataset is comprised of two sensitive groups: a majority ($90\%$) and minority ($10\%$). In both groups the probability of having either value of the class label is 0.5. The instances are described by two features $X_1$ and $X_2$. In the case of the majority group the knowledge of $X_1$ is enough to be able to distinguish the different classes. However, in the case of the minority group, the knowledge about both $X_1$ and $X_2$ is required to derive the optimal classification rule. Besides, we set a higher dispersion on the instances of the minority group. 

\begin{figure}[h!]
\begin{center}
\centerline{\includegraphics[width=0.6\textwidth]{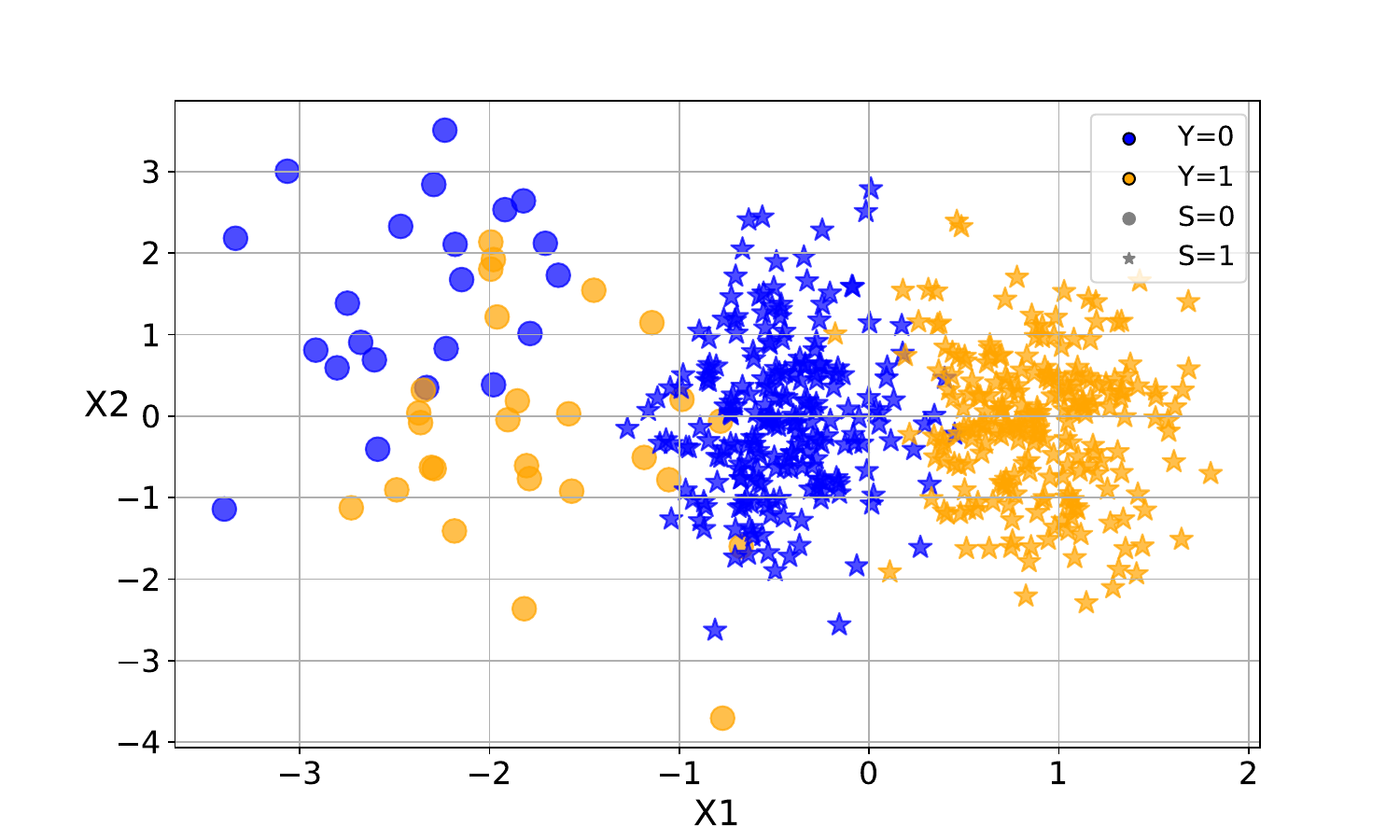}}
\caption{The toy dataset. A binary classification problem with 2 non-sensitive features ($X_1$ and $X_2$), a sensitive attribute $S$ and a binary class label $Y$. The toy dataset is comprised of two sensitive groups: a majority ($90\%$) and minority ($10\%$). In both groups the probability of having either value of the class label is 0.5. In the case of the majority group the knowledge of $X_1$ is enough to be able to distinguish the different classes. However, in the case of the minority group the knowledge about both $X_1$ and $X_2$ is required to derive the optimal classification rule. Besides, the standard deviation of minority instances is higher. }
\label{fig:toy_dataset}
\end{center}
\end{figure}

We first sample the sensitive attributes, knowing that $P(S=0)=0.1$ and $P(S=0)=0.9$. Then we sample the class label $Y$ uniformly at random for both groups, since $P(Y=0|S)=P(Y=1|S) = 0.5$. The conditional distributions of the remaining features are obtained from the following Gaussian distributions: 
\begin{align}
    (X_1, X_2)|A = 1, Y = 1 &\sim \mathcal{N}([6,0], [1,0; 0,1]), \\
    (X_1, X_2)|A = 1, Y = 0 &\sim \mathcal{N}([2,0], [1,0; 0,1]), \\
    (X_1, X_2)|A = 0, Y = 1 &\sim \mathcal{N}([-4,2], [2.5,0; 0,2.5]), \\
    (X_1, X_2)|A = 1, Y = 1 &\sim \mathcal{N}([-2,0], [2.5,0; 0,2.5]). 
\end{align}

We sample 1000 instances and standardize the data. 30\% of the instances comprise the test set and among the remaining instances, 20\% compose the validation set. Figure \ref{fig:toy_dataset} shows a graphical representation of the toy dataset.


Figure \ref{fig:sigma_DB_WG} illustrates the effect of $\sigma$ on the MRC's performance in the toy problem. Smaller $\sigma$ values correspond to lower frequencies in the mapping components, capturing simpler patterns that yield less complex decision boundaries. Specifically, for low $\sigma$ values, the decision boundary is a vertical line that depends almost exclusively on $X_1$, aligning with the pattern observed in the majority group. However, this vertical line is not centered between the blue and orange groups of the majority group, as the classification rule minimizes the worst-case error. 
As $\sigma$ increases, the influence of the $X_2$ feature grows in the resulting classification rule. At very high $\sigma$ values, the decision boundary reflects overly noisy patterns, significantly impacting the classifier’s worst-group accuracy. Both overly simplistic classifiers (resulting from low-frequency mapping components at small $\sigma$) and overly noisy ones (due to high $\sigma$ introducing noise) degrade effectiveness, impairing both overall and group-specific performance.

\begin{figure}[h!] 
    \centering
    \subfigure[]{
    \includegraphics[width=0.8\textwidth]{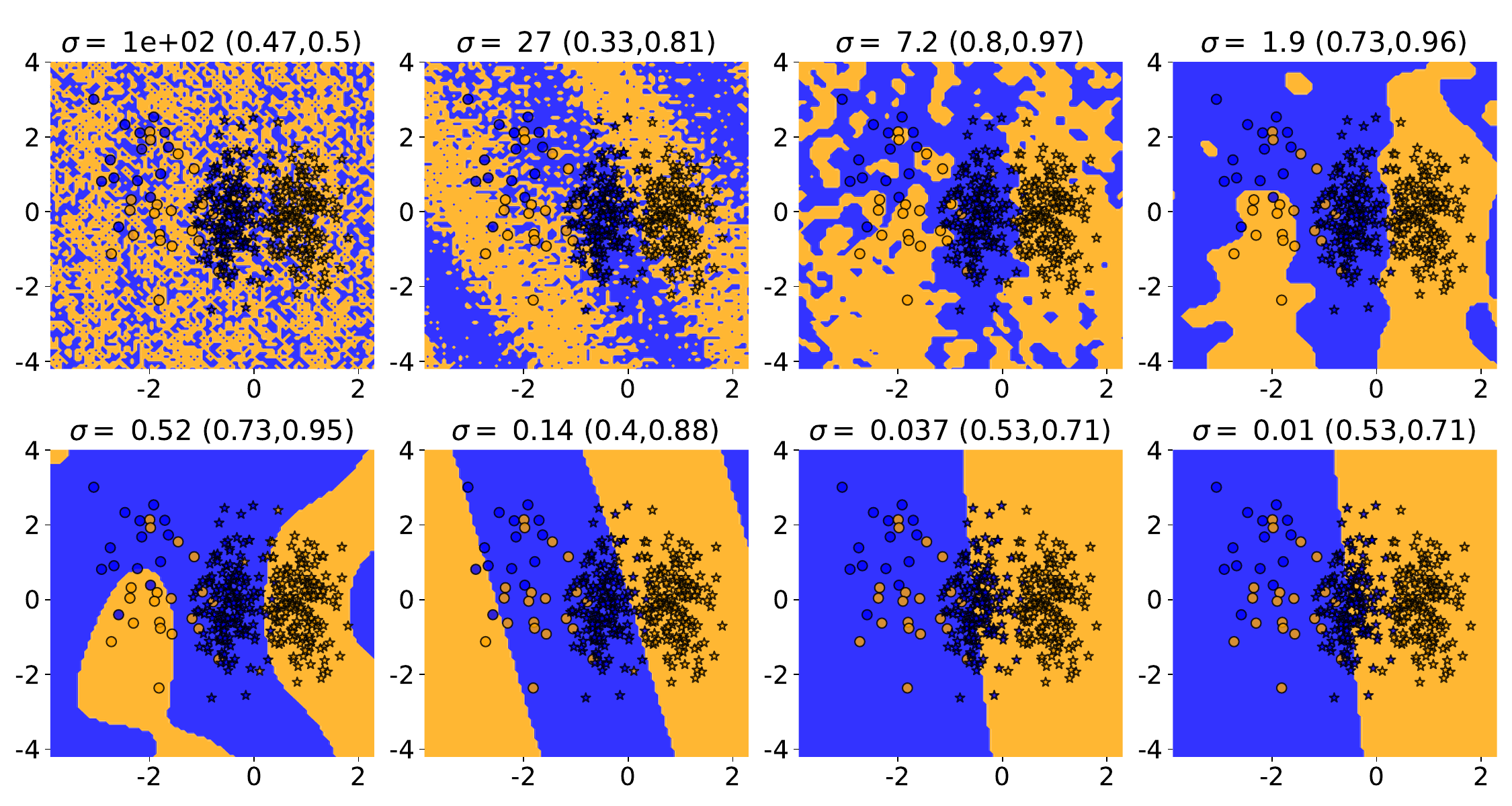}
    \label{fig:sigma_subplot1}
  }
  \hfill 
  \subfigure[]{
    \includegraphics[width=0.65\textwidth]{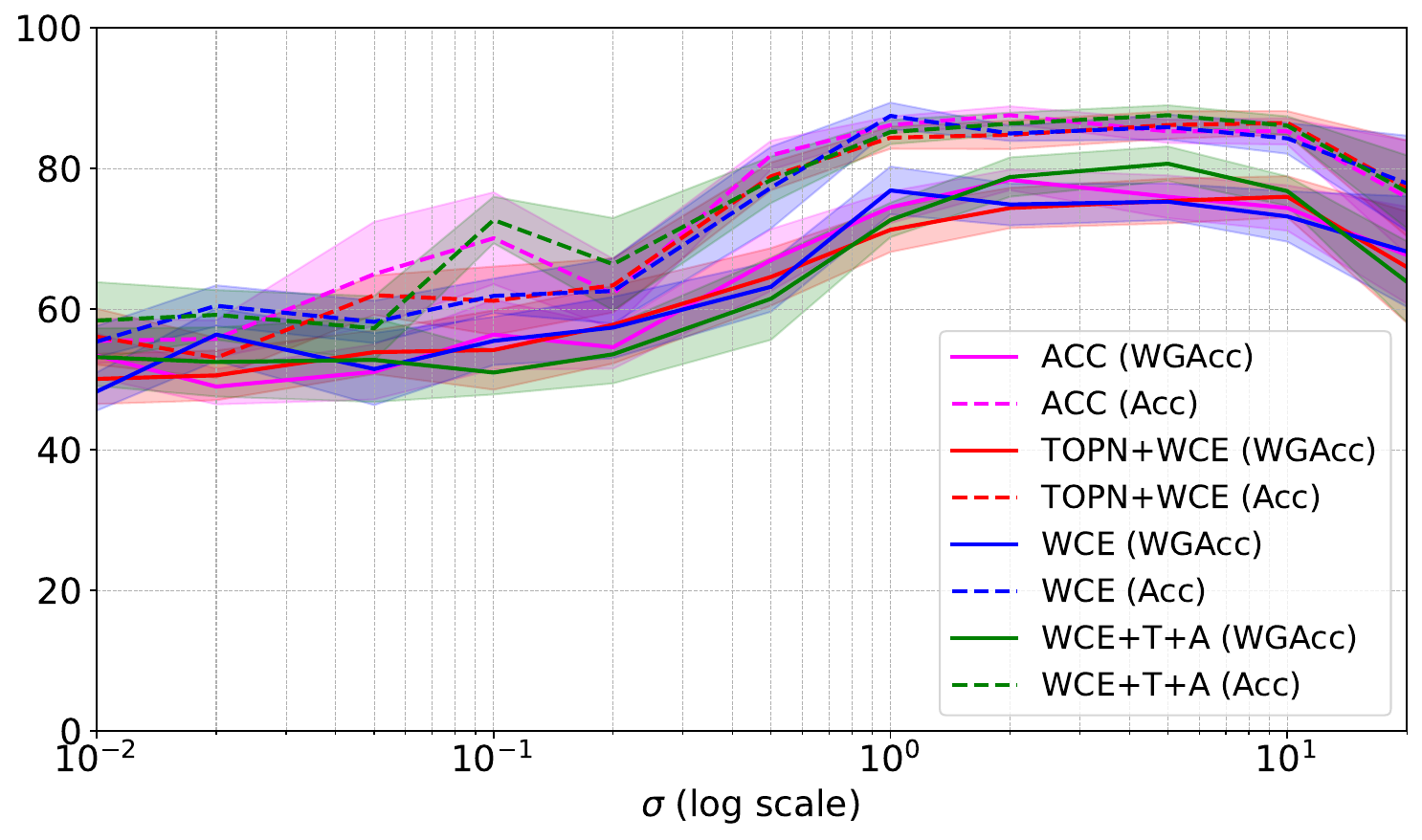}
    \label{fig:sigma_subplot2}
  }
    \caption{The effect of the scaling parameter $\sigma$ on (a) the decision boundaries (where training instances are also shown) and (b) overall accuracy and worst-group accuracy for the toy dataset, with different hyperparameter tuning strategies to get the value of $\lambda_0$ (see Section \ref{sec:experimental_setting} for descriptions on strategies). In (a), we set $\lambda_0 = 0.1$. Above the graphical representations of the decision boundaries, the value of $\sigma$ under consideration is displayed, along with the corresponding worst-group accuracy and overall accuracy (presented in parentheses, in that order). In (b), the curve represents the average value across 10 runs, and the shaded region denotes the standard deviation. 
    }
    \label{fig:sigma_DB_WG}
\end{figure}

We also evaluate other general feature mappings beyond RRFs in Appendix~\ref{ap:other_featmap}. The findings indicate that the Fourier feature mapping achieves the best worst-group accuracy. As our objective is to improve fairness guarantees in terms of worst-group performance, we therefore employ the Fourier feature mapping together with MRC over alternative mappings.

\subsection{\texttt{SPECTRE}}

In summary, \texttt{SPECTRE} is a method that exploits the impact of spectrum modification on the performance and fairness guarantees of a MRC via the uncertainty set to construct minimax-fair classification rules without relying on demographic information. It comprises two main components: first, it adjusts the response frequency spectrum to identify the optimal frequency components for representing the uncertainty set by optimizing the scaling factor $\sigma$; second, it determines the optimal confidence vector value, by tuning $\lambda_0$. Algorithm \ref{alg:pseudocode} provides the pseudocode for \texttt{SPECTRE}. 

\begin{algorithm}[h!]
   \caption{\texttt{SPECTRE}}
   \label{alg:spectre}
    \begin{algorithmic}
   \STATE {\bfseries Input:} $\mathcal{D}_{\text{train}} = \{ (x^i, y^i) \}_{i=1}^{N_{\text{train}}}$, $\mathcal{D}_{\text{val}} = \{ (x^i, y^i) \}_{i=1}^{N_{\text{val}}}$, $\sigma_{\text{values}} = \{ \sigma^j \}_{j=1}^{J}$, $\lambda_{\text{values}} = \{ \lambda_0^k \}_{k=1}^{K}$, $\lambda_0^{\text{init}}$, $\mathit{strategy} \text{ (Section } \ref{sec:experimental_setting})$
   \STATE {\bfseries Output:} $\text{a minimax-fair classification rule } f^*(\boldsymbol{x})$
    \STATE $\lambda_0 \gets \lambda_0^{\text{init}}$
    \STATE $\text{perf}\_\text{res} = [\;]$
    \FOR{$\sigma^j$ {\bfseries in} $\sigma_{\text{values}}$}
    \STATE $f_j \gets MRC(\mathcal{D}_{\text{train}}, \Phi_{\sigma^j}^{RFF}, \lambda_0)$
    \STATE $\text{perf}\_\text{res}_j \gets f_j(\mathcal{D}_{val})$ 
    \ENDFOR
    \STATE $\sigma^* \gets \text{get}\_\text{best}\_\text{hyp}(\text{perf}\_\text{res}, \sigma_{\text{values}}, \mathit{strategy})$ 
    \STATE $\text{perf}\_\text{res} = [\;]$
    \FOR{$\lambda_0^k$ {\bfseries in} $\lambda_{\text{values}}$}
    \STATE $f_k \gets MRC(\mathcal{D}_{\text{train}}, \Phi_{\sigma^*}^{RFF}, \lambda_0^k)$
    \STATE $\text{perf}\_\text{res}_k \gets f_k(\mathcal{D}_{\text{val}})$ 
    \ENDFOR
    \STATE $\lambda_0^* \gets \text{get}\_\text{best}\_\text{hyp}(\text{perf}\_\text{res}, \sigma_{\text{values}}, \mathit{strategy})$ 
    \STATE $f^*(\boldsymbol{x}) \gets MRC(\mathcal{D}_{\text{train}}, \Phi_{\sigma^*}^{RFF}, \lambda_0^*)$
    \end{algorithmic}
\label{alg:pseudocode}
\end{algorithm}

The method takes as input a training set $\mathcal{D}_{\text{train}}$, a validation set $\mathcal{D}_{\text{val}}$, the set of possible values for $\sigma$ (denoted as $\sigma_{\text{values}}$), the set of possible values for $\lambda_0$ (denoted as $\lambda_{\text{values}}$), an intermediate value for $\lambda_0^{\text{init}}$ used to tune the hyperparameter $\sigma$, and the hyperparameter selection strategy ($\mathit{strategy}$, see Section \ref{sec:experimental_setting} for descriptions on strategies). The output is the minimax-fair decision rule.

\texttt{SPECTRE} consists of two main blocks. First, it adjusts the response frequency spectrum to identify the optimal frequency components for representing the uncertainty set. Second, it determines the optimal confidence vector value, which defines the maximum allowed divergence of the worst-case distribution from the empirical distribution. Each step corresponds to tuning a specific hyperparameter: the first involves optimizing the scaling factor $\sigma$ of the distribution used to sample the frequency components for the mapping of RFFs, while the second focuses on optimizing $\lambda_0$.

As a deeper analysis of \texttt{SPECTRE}, the approach begins with initializing the value of $\lambda_0$, ideally setting it to an intermediate value. For example, the value 0.3, used for empirical validation in the original paper by \citet{mazuelas2023minimax}, is a reasonable choice, though other suitable values could also be used. The process then proceeds as follows: for each $\sigma^j$ in the set $\sigma_{\text{values}}$, the training data is mapped into the response frequency domain using RFFs with the scaling parameter $\sigma^j$, after which an MRC model is trained on the transformed dataset.
The confidence vector is defined based on the initialized value $\lambda_0^{\text{init}}$. The resulting classification rule is tested on the validation set, and the relevant performance metrics are recorded ($\text{perf}\_\text{res}_j$). Once all $\sigma$ values have been evaluated, we apply the specified tuning strategy to determine the optimal value, $\sigma^*$. With $\sigma^*$ established, we move on to optimize $\lambda_0$ using a similar procedure. For each value of $\lambda_0^k$ in $\lambda_{\text{values}}$, we train an MRC on the training set, now transformed into the response frequency domain using RFFs with scaling parameter $\sigma^*$. 
The resulting classification rule is tested on the validation set, and the necessary performance metrics are recorded ($\text{perf}\_\text{res}_k$). Using the chosen tuning strategy, we then identify the optimal value, $\lambda_0^*$. Finally, the method outputs the MRC derived from the optimal spectrum adjustment $(\sigma^*)$ and confidence vector value $(\lambda_0^*)$.

We employed the same selection strategy for both hyperparameters, though it is also possible to use different tuning strategies for each of the hyperparameters.

The computational complexity of SPECTRE is given by 
$\mathcal{O}(N_{\lambda} \, N_{\sigma} \, \mathcal{C}_{\text{MRC}}),$
where $N_{\lambda}$ denotes the number of $\lambda_0$ values considered, $N_{\sigma}$ is the number of $\sigma$ values, and $\mathcal{C}_{\text{MRC}}$ represents the computational complexity of training an MRC with given hyperparameter values. The term $\mathcal{C}_{\text{MRC}}$ depends on the strategy used to solve the convex optimization problem for fitting the MRC, which at the same time depends on terms such as the number of instances and the number of Fourier components considered. For example, when using the widely adopted stochastic gradient descent, the complexity scales as 
$\mathcal{C}_{\text{MRC}} \propto N_{\text{iter}} \, N_{\text{inst}} \, D \, |\mathcal{Y}|,$
where $N_{\text{iter}}$ is the number of iterations (determined by batch size and number of epochs), $N_{\text{inst}}$ is the number of instances, $D$ is the number of Fourier components used in the RFF representation, and $|\mathcal{Y}|$ is the number of classes.


Besides, note that since \texttt{SPECTRE} is blind to sensitive information, it does not explicitly optimize any specific notion of fairness. However, it is closely related to minimax fairness and can be regarded as a proxy optimization of it. As an RRM-based approach, \texttt{SPECTRE} optimizes classifier performance under worst-case conditions, aiming to improve the performance in the most challenging instances. On the other hand, minimax fairness focuses on the performance of the worst-performing group, with bias mitigation strategies designed to specifically improve outcomes for that (or those) group(s). While \texttt{SPECTRE} does not explicitly target demographic groups, its emphasis on worst-case performance across instances can indirectly improve the performance of certain groups. Enhancing performance in specific regions of the input space may translate into better outcomes for particular demographics. Moreover, although the connection between \texttt{SPECTRE} and other fairness notions such as equality of opportunity is less direct, improvements in worst-case performance often influence group-based metrics such as TPR, FPR, and others. Furthermore, it is important emphasizing that, as mentioned in the introduction, improving the predictive performance for the worst-performing group has the property that any model that achieves it Pareto dominates (at least weakly, and possibly strictly) an equalized-error model with respect to group error rates \cite{martinez2020minimax}.  In this sense, by improving performance on the most vulnerable instances, \texttt{SPECTRE} can contribute to reducing disparities captured by various fairness metrics.

\section{Out-of-Sample Guarantees on Group Error}
\label{sec:bounds}

Although a classifier $f^*$ may perform well on the training distribution, its performance and fairness guarantees may degrade in deployment if the latter differs from the training environment. Such discrepancies are common in real-world applications due to factors like sample selection bias or environmental non-stationarity. Therefore, before deploying $h^*$ in a specific deployment setting, it is crucial to evaluate the robustness of its performance and fairness guarantees to distribution shifts. As discussed in Section \ref{sec:spectral_mrc}, \texttt{SPECTRE} operates within an RRM framework using the uncertainty set defined in Equation (\ref{eq:uncertainty}). This allows us to derive out-of-sample generalization bounds on group errors, providing a measure of how much worst-group error (and overall error) may degrade under test environment distribution shifts. This analysis provides a systematic approach to determine whether the classifier is safe for deployment ``in the wild''. 

To evaluate the upper and lower bounds of \texttt{SPECTRE}'s performance across different sensitive groups, it is necessary for a portion of the training data to include sensitive information. It is important to emphasize, however, that this requirement stems from the fairness metrics themselves (rather than from our estimation method) as these metrics inherently depend on access to sensitive information.
Formally, the worst-case and best-case errors of the classification rule $f^*$ resulting from \texttt{SPECTRE} on sensitive group $\boldsymbol{s}$ for a general loss function $\ell$ can be defined as:
\begin{equation}
\begin{aligned}
     \mathcal{R}^{\uparrow}_{\mathcal{U}}(f^*, p_{|\boldsymbol{s}}) & = \max_{p \in \mathcal{U}} \mathcal{R}(f^*, p_{|\boldsymbol{s}}) \\ 
     & = \max_{p \in \mathcal{U}} \mathbb{E}_{p (\boldsymbol{x},y |S=\boldsymbol{s})} \{ \mathcal{L}(f^*(\boldsymbol{x}),y) \}
\end{aligned}
\end{equation}

\begin{equation}
\begin{aligned}
     \mathcal{R}^{\downarrow}_{\mathcal{U}}(f^*, p_{|\boldsymbol{s}}) & = \min_{p \in \mathcal{U}} \mathcal{R}(f^*, p_{|\boldsymbol{s}}) \\  
     & = \min_{p \in \mathcal{U}} \mathbb{E}_{p (\boldsymbol{x},y |S=\boldsymbol{s})} \{ \mathcal{L}(f^*(\boldsymbol{x}),y) \}
\end{aligned}
\end{equation}
where the uncertainty set is the one described in Equation (\ref{eq:uncertainty}). In particular, it considers all the distributions for which a the expected value of a given feature mapping applied to that distribution is at most at a distance $\boldsymbol{\lambda}$ from the empirical expectation. 


For every distribution $\hat{p}$ contained in the uncertainty set $\mathcal{U}$, the error of $f^*$ in $\hat{p}$ is bounded by the worst- and best-case errors:
\begin{equation}
    \mathcal{R}^{\downarrow}_{\mathcal{U}}(f^*, p_{|\boldsymbol{s}}) \leq \mathcal{R}(f^*, \hat{p}_{|\boldsymbol{s}}) \leq \mathcal{R}^{\uparrow}_{\mathcal{U}}(f^*, p_{|\boldsymbol{s}}) .
\end{equation}
Consequently, if the true distribution $p^*$ lies within $\mathcal{U}$, the true error is also bounded as 
\begin{equation}
    \mathcal{R}^{\downarrow}_{\mathcal{U}}(f^*, p_{|\boldsymbol{s}}) \leq \mathcal{R}(f^*, p^*_{|\boldsymbol{s}}) \leq \mathcal{R}^{\uparrow}_{\mathcal{U}}(f^*, p_{|\boldsymbol{s}}). 
\end{equation}
The following theorem determines that these upper and lower bounds are the optimal solutions of two linear programs (LP).

\begin{theorem} 

Let $\Phi$ : $\mathcal{X}$ $\times$ $\mathcal{Y}$ $\rightarrow$ $\mathbb{R}^m$ be a feature mapping blind to the sensitive information, $\boldsymbol{\tau} = \mathbb{E}_{p_n}\{ \Phi \}$ $\in \mathbb{R}^m$ and $\boldsymbol{\lambda}$ $\in \mathbb{R}_{+}^m$. Then $\mathcal{R}^{\uparrow}_{\mathcal{U}}(f^*, p_{|s})$ (resp. $\mathcal{R}^{\downarrow}_{\mathcal{U}}(f^*, p_{|s})$) can be calculated as the solution to the following linear program:
\begin{equation}
\begin{aligned}
\max_{\boldsymbol{q} \in \mathbb{R}^N, z\in \mathbb{R}} & \big[ \text{resp. }\min_{\boldsymbol{q} \in \mathbb{R}^N, z\in \mathbb{R}} \big]  \quad \boldsymbol{c}^T \boldsymbol{q} \\
\textrm{s.t.} \quad & z(\boldsymbol{\tau} - \boldsymbol{\lambda}) \preceq \sum_{i=1}^{N} \Phi_i q_i \preceq z(\boldsymbol{\tau} + \boldsymbol{\lambda}),\\
 & \sum_{i=1}^N q_i  = z, \\ 
 & \boldsymbol{e}^T \boldsymbol{q}  = 1, \\
 & 0 \leq q_i \leq z  \quad \forall i, \\
 & z \geq 0  \\
  \label{eq:optim-group}
\end{aligned}
\end{equation}
where $c_i = \mathcal{L}(f^*(\boldsymbol{x}^i),y^i) \mathbb{I}\{\boldsymbol{s}^i = s\}$ and $e_i = \mathbb{I}\{\boldsymbol{s}^i = \boldsymbol{s}\}$. 
\label{th:bounds_main}
\end{theorem}


\begin{proof}
Without loss of generality, we will focus on the worst-case error (maximization problem) in the proof, as deriving the best-case expression by addressing the minimization problem follows an equivalent process. The worst (highest) possible error for each sensitive group $\mathcal{G}_{\boldsymbol{s}}$ (with sensitive attribute $S=\boldsymbol{s}$) under the pre-trained classifier $f \in T(\mathcal{X} \times \mathcal{Y})$ across all distributions in the uncertainty set $\mathcal{U}$ for a general loss function $\mathcal{L}$ is:
\begin{equation}
\begin{aligned}
     \mathcal{R}^{\uparrow}_{\mathcal{U}}(f^*, p_{|\boldsymbol{s}}) & = \max_{p \in \mathcal{U}} \mathcal{R}(f^*, p_{|\boldsymbol{s}}) \\
     & = \max_{p \in \mathcal{U}} \mathbb{E}_{p (\boldsymbol{x},y |S=\boldsymbol{s})} \{ \mathcal{L}(f^*(\boldsymbol{x}),y) \}
\end{aligned}
\end{equation}

Assuming we have access to a dataset $\mathcal{D}=\{(\boldsymbol{x}^i, \boldsymbol{s}^i, y^i) \}_{i=1}^N$, the previous expression can be reformulated as:
\begin{equation}
     \max_{p \in \mathcal{U}} \sum_{(\boldsymbol{x},y) \in \mathcal{G}_{\boldsymbol{s}}} p_{|\boldsymbol{s}}(\boldsymbol{x},y) \; \mathcal{L}(f^*(\boldsymbol{x}), y) 
\end{equation}
where
\begin{equation}
p_{|\boldsymbol{s}} = 
\begin{cases}
  0  &  \textrm{if} \; \; \boldsymbol{x} \notin \mathcal{G}_{\boldsymbol{s}} \\
  \frac{p(\boldsymbol{x},y)}{p(S=\boldsymbol{s})}  &  \textrm{if} \; \; \boldsymbol{x} \in \mathcal{G}_{\boldsymbol{s}}
\end{cases}
\end{equation}

That is,
\begin{equation}
     \max_{p \in \mathcal{U}} \frac{\sum_{(\boldsymbol{x},y)\in \mathcal{G}_{\boldsymbol{s}}}  p(\boldsymbol{x},y) \; \mathcal{L}(f^*(\boldsymbol{x}),y)}  {\sum_{(\boldsymbol{x},y)\in \mathcal{G}_{\boldsymbol{s}}} p(\boldsymbol{x},y)} 
    \label{eq:worst_loss_s}
\end{equation}

Having access to a dataset $\mathcal{D}=\{(\boldsymbol{x}^i, \boldsymbol{s}^i, y^i) \}_{i=1}^N$, we can describe $p$ as a vector $\boldsymbol{p} = (p_1, ..., p_N)$, where the component $p_i$ denotes the probability $p(\boldsymbol{x}^i, \boldsymbol{s}^i, y^i)$. Therefore, reformulating problem (\ref{eq:worst_loss_s}) 
and explicitly adding the constraints, the worst error for group $\mathcal{G}_{\boldsymbol{s}}$ can be obtained by solving the following optimization problem:
\begin{equation}
\begin{aligned}
\max_{\boldsymbol{p}}  \quad & \frac{\boldsymbol{c}^T \boldsymbol{p}}{\boldsymbol{e}^T \boldsymbol{p}} \\
\textrm{s.t.} \quad & \boldsymbol{\tau} - \boldsymbol{\lambda} \preceq \sum_{i=1}^{N} \Phi_i p_i \preceq \boldsymbol{\tau} + \boldsymbol{\lambda}\\
  &  \sum_{i=1}^N p_i  = 1    \\
  &  0 \leq p_i \leq 1  \quad \forall i   \\
  \label{eq:optim-group2}
\end{aligned}
\end{equation}
where
\begin{equation}
c_i = 
\begin{cases}
  \mathcal{L}(f^*(\boldsymbol{x}^i),y^i) &  (\boldsymbol{x}^i, y^i) \in \mathcal{G}_{\boldsymbol{s}} \\
  0  &   (\boldsymbol{x}^i, y^i) \notin \mathcal{G}_{\boldsymbol{s}}
\end{cases}
\end{equation}
and
\begin{equation}
e_i = 
\begin{cases}
  1  &   (\boldsymbol{x}^i, y^i) \in \mathcal{G}_{\boldsymbol{s}} \\
  0  &   (\boldsymbol{x}^i, y^i) \notin \mathcal{G}_{\boldsymbol{s}}
\end{cases}
\end{equation}

The optimization problem in \eqref{eq:optim-group2} is \textit{linear-fractional} in $p$. However, assuming that $\boldsymbol{e}^T \boldsymbol{p}>0$ it is possible to convert this \textit{linear-fractional} program in $p$ into a LP in with the following change of variables:
\begin{equation*}
    \boldsymbol{q} = \frac{\boldsymbol{p}}{\boldsymbol{e}^T \boldsymbol{p}} \quad z = \frac{1}{\boldsymbol{e}^T \boldsymbol{p}}
\end{equation*}

Therefore, the resulting LP reads as:

\begin{equation}
\begin{aligned}
\max_{\boldsymbol{q} \in \mathbb{R}^N, z\in \mathbb{R}}  \quad & \boldsymbol{c}^T \boldsymbol{q} \\
\textrm{s.t.} \quad & z(\boldsymbol{\tau} - \boldsymbol{\lambda}) \preceq \sum_{i=1}^{N} \Phi_i q_i \preceq z(\boldsymbol{\tau} + \boldsymbol{\lambda}),\\
 & \sum_{i=1}^N q_i  = z, \\ 
 & \boldsymbol{e}^T \boldsymbol{q}  = 1, \\
 & 0 \leq q_i \leq z  \quad \forall i, \\
 & z \geq 0  \\
  \label{eq:optim-group-3}
\end{aligned}
\end{equation}

where $c_i = \mathcal{L}(f^*(\boldsymbol{x}^i),y^i) \mathbb{I}\{\boldsymbol{s}^i = \boldsymbol{s} \}$ and $e_i = \mathbb{I}\{\boldsymbol{s}^i = \boldsymbol{s}\}$.

\end{proof}

Furthermore, the solutions to the optimization problem in (\ref{eq:optim-group}) also reveal the distributions that derive the worst- and best-case performances across the different groups.

\begin{definition}[Extremal distribution for group $\boldsymbol{s}$]
    Let  $\boldsymbol{q}^{\uparrow *}_{\boldsymbol{s}}$ and $z^{\uparrow *}_{\boldsymbol{s}}$ (resp. $\boldsymbol{q}^{\downarrow *}_{\boldsymbol{s}}$ and  $z^{\downarrow *}_{\boldsymbol{s}}$) and be the solutions of the maximization (resp. minimization) problem (\ref{eq:optim-group}) for group $\boldsymbol{s}$, then $\boldsymbol{p}^{\uparrow *}_{\boldsymbol{s}} = \frac{1}{z^{\uparrow *}_{\boldsymbol{s}}}\boldsymbol{q}^{\uparrow *}_{\boldsymbol{s}}$ (resp. $\boldsymbol{p}^{\downarrow *}_{\boldsymbol{s}} = \frac{1}{z^{\downarrow *}_{\boldsymbol{s}}}\boldsymbol{q}^{\downarrow *}_{\boldsymbol{s}}$) will represent the distribution that maximizes (resp. minimizes) the error for group $\boldsymbol{s}$. 
\end{definition}

Moreover, we can derive \textbf{out-of-sample generalization guarantees for the overall risk}
by solving the optimization problem (\ref{eq:optim-group}), setting $z = 1$, and considering $c_i = \mathcal{L}(f^*(\boldsymbol{x}^i), y^i)$ for all $(\boldsymbol{x}^i, y^i) \in \hat{\mathcal{D}}$. Besides, the value of $\boldsymbol{q}$ that optimizes (\ref{eq:optim-group}) in that case, will refer to the worst distribution for the overall accuracy (since in this case $\boldsymbol{p} \equiv \boldsymbol{q}$ as $z = 1$). Note that the computation of the bounds on the overall risk does not require the knowledge of the sensitive information of any instance.

\section{Experimental Evaluation}
\label{sec:spectre_experiments}

This section presents the empirical evaluation of the proposed framework.\footnote{The source code for reproducing these experiments is available at \url{https://github.com/abarrainkua/SPECTRE}.} We first examine the performance of \texttt{SPECTRE} and SOTA approaches across real-world classification tasks, focusing primarily on methods that improve fairness without relying on demographic data. For completeness, we also include comparisons with methods that have full access to the sensitive attribute. In addition, to gain further insight into \texttt{SPECTRE}'s behavior, we study the out-of-sample guarantees on both group and overall errors, and analyze the computational cost of \texttt{SPECTRE} relative to SOTA methods.

\subsection{Experimental Setting}
\label{sec:experimental_setting}

\subsubsection{Models}

We conduct the empirical validation of \texttt{SPECTRE} through various SOTA approaches: (\textbf{i}) two baselines, namely Logistic Regression (LR) and XGBoost (XGB) \cite{chen2016xgboost}; (\textbf{ii}) fairness enhancing-interventions that are blind to the sensitive information by \citet{hashimoto2018fairness} (RLM), \citet{lahoti2020fairness} (ARL),  \citet{martinez2021blind} (BPF), \citet{chakrabarti2023sure} (SURE) and \citet{ko2023fair}(FairEns); and (\textbf{iii})  fairness-enhancing interventions that assume complete knowledge of the sensitive information and optimize worst-group performance (i.e., minimax fairness), including the works by \citet{martinez2020minimax} (MMPF), \citet{sagawa2020distributionally} (GDRO) and \citet{diana2021minimax} (MMPFrel). 
It is worth noting that, although XGBoost is treated as a baseline, \citet{gardner2022subgroup} showed that its tree-based structure derived strong subgroup robustness, even in comparison to methods explicitly designed to enhance robustness and fairness. In the following, we present more detailed descriptions of the methods:

\paragraph{Logistic regression (LR) classifier.} We consider LR as a baseline due to its popularity as a conventional ML method. As a straightforward unconstrained classifier, LR serves as a crucial reference point to illuminate biases inherent in the training data. We implemented this method utilizing the code from the \verb|sklearn| Python package.

\paragraph{XGBoost.} The XGBoost classifier \cite{chen2016xgboost}, is a tree-based method that employs a scalable end-to-end tree boosting system. We opted to incorporate this method into our empirical evaluation based on recent experiments that demonstrated its performance robustness within subgroups \cite{gardner2022subgroup}. It showcased superior performance, even when compared to methods aimed at enhancing robustness and fairness. Notably, it exhibits a lower training cost and is less sensitive to hyperparameter configurations. We explore various hyperparameter configurations as investigated in the comprehensive empirical evaluation by \citet{gardner2022subgroup} and present the outcomes with the highest accuracy and optimal worst-group accuracy. We leverage the code from the \verb|xgboost| Python package.

\paragraph{Repeated loss minimization (RLM).} \citet{hashimoto2018fairness} suggests using a general DRO framework with a $\chi^2$ uncertainty set to improve worst-group performance when sensitive information is unavailable. Specifically, for methods trained using gradient descent, they introduce an efficient formulation based on the dual formulation. Essentially, they reformulate the loss, considering only instances whose loss exceeds a specified value $\eta$. We tune the value of $\eta$ as detailed in their paper.\footnote{\url{https://bit.ly/2sFkDpE}} 

\paragraph{Adversarially reweighted learning (ARL).}  \citet{lahoti2020fairness} introduce an adversarial learning framework designed to enhance the predictive performance of a classifier for unprivileged groups, even in scenarios where the sensitive information is not accessible. In this framework, the \textit{learner} focuses on optimizing the primary task by learning an effective classification rule to minimize the overall expected loss. Simultaneously, the \textit{adversary} learns a function that maps to computationally identifiable regions characterized by high loss and assigns weights to maximize the expected loss of the learner. We have used the code made available by the authors\footnote{\url{https://github.com/google-research/google-research/tree/master/group_agnostic_fairness}}  and we adopt the network architecture employed by the authors for both the learner and the adversary components.

\paragraph{Blind pareto fairness (BPF).} \citet{martinez2021blind} propose a variant of the MMPF method to achieve minimax fairness when the sensitive attribute is completely unknown. In particular, their main goal is to learn a classification rule that has the best
performance on the worst-group risk of the worst possible realization of the sensitive group assignment, subject to a constraint in the minimum group size. This framework works when both the loss function and the hypothesis set are convex. Besides, the problem at hand to be well-defined, there's a necessity to assume a binary identification mapping, representing the worst-group risk membership. We consider the code provided by the authors.\footnote{\url{https://github.com/natalialmg/BlindParetoFairness}} We execute the model for various values of $\rho$ (minimum group size), selecting the range of values as presented in their experimental section, and present the outcome with the optimal worst-group accuracy. In instances where multiple solutions exhibit equal worst-group accuracy, we prioritize the one with the highest overall accuracy.

\paragraph{Significant Unfairness Risk Elimination (SURE).}
\citet{chakrabarti2023sure} propose an iterative reweighting algorithm. At each iteration, SURE identifies high-risk regions in the feature space; that is, zones where the risk of unfair classification is statistically significant. The algorithm then adjusts the instance weights in the loss function by increasing the weights of samples within these identified zones, followed by normalization of all weights. SURE involves only one hyperparameter, which we set to the same default value adopted in the original study to ensure consistency. 

\paragraph{Fair Ensemble (FairEns).}
\citet{ko2023fair} observe that fairness can naturally emerge from ensembling deep neural networks. Specifically, they show that in a homogeneous ensemble—where all individual models share the same architecture, training data, and design choices—the performance of minority groups tends to improve disproportionately as the number of models increases, thereby enhancing overall fairness. 
Motivated by these findings, we include FairEns in our empirical evaluation as a representative fairness-enhancing ensemble approach.
The diversity within the ensemble is introduced from the stochasticity introduced by random parameter initialization, as proposed by the authors. 

\paragraph{Minimax pareto fairness (MMPF).} \citet{martinez2020minimax} propose a minimax approach, framed as a multi-objective optimization problem, where the objectives correspond to the classifier's performance on individual sensitive groups. The aim is to identify the Pareto optimal classification rule (with respect to group errors) that maximizes worst-group accuracy. However, strictly  optimizing for minimax fairness may compromise overall accuracy. Given that problem, \citet{diana2021minimax} suggest a relaxed version (which we refer to as \textbf{MMPFrel}), where they propose to consider $\epsilon$-optimal models (in terms of minimax fairness) that achieve significantly higher overall accuracy, given a maximum group error bound $\gamma$. We consider the code provided by the authors for both the original\footnote{\url{https://github.com/natalialmg/MMPF}} and relaxed\footnote{\url{https://github.com/amazon-science/minimax-fair}} versions. In the case of the latter, we execute the model for various values of $\gamma$ (selecting the range of values as presented in their experimental section) and present the outcome with the optimal worst-group accuracy. In instances where multiple solutions exhibit equal worst-group accuracy, we prioritize the one with the highest overall accuracy.

\paragraph{Group distributionally robust optimization (GDRO).} \citet{sagawa2020distributionally} introduce a technique to enhance worst-group accuracy through DRO. The objective is to identify the classification rule that maximizes worst-group accuracy within an uncertainty set of distributions. The uncertainty set is constructed as a weighted combination of the empirical distributions of distinct groups. We have implemented their model using the code provided by the authors.\footnote{\url{https://github.com/kohpangwei/group_DRO/}} 

These three latter approaches (MMPF, MMPFrel and GDRO) operate under the assumption that sensitive information is available for both training and validation instances. Therefore, they are given competitive advantage in the experimental setting, and have access to the sensitive information of all the instances. 

Furthermore, given the nature of the tabular classification tasks under consideration, all minimax fair methods—whether or not they use sensitive information—employ, unless otherwise specified, a base model consisting of a 2-layer NN with sizes 64 and 32, as it has consistently demonstrated superior performance. Besides, for all approaches involving NNs to construct classification rules, we adopt a learning rate of $10^{-3}$, conducting training for 1000 epochs, and employ Adam optimizer. We aim to ensure that all models are trained under identical conditions to establish a fair empirical evaluation setting.

\subsubsection{Hyperparameter Tuning} 

\texttt{SPECTRE} has two hyperparameters ($\sigma$ and $\lambda_0$), for which we explore 4 different hyperparameter tuning strategies that do not rely on sensitive information. These methods are applied to a validation set, and the hyperparameter is selected based on the following criteria: \textbf{(ACC)} selecting the configuration that maximizes accuracy; \textbf{(WCE)} selecting the configuration that minimizes the worst class error; \textbf{(WCE+T+A)} selecting the configuration that maximizes accuracy among those with a worst class error within a specified tolerance; and \textbf{(TOPN+WCE)} selecting the configuration with the minimum worst class error from the top-$5$ configurations ranked by accuracy. We begin by tuning $\sigma$ with $\lambda_0 = 0.3$ (the value used for the empirical evaluation of MRC in the original work by \citet{mazuelas2023minimax}), as $\sigma$ has the most significant impact on the performance of our approach. Once the optimal $\sigma$ value is determined, we then tune $\lambda_0$ using this previously optimized $\sigma$.
We tune the parameter $\lambda_0$ in the range [0.01, 1.0], and $\sigma$ in the logaritmic range $[0.1*\sigma_{scale}, 10*\sigma_{scale}]$, where $\sigma_{scale} = \sqrt{\frac{2}{D \cdot var_{\tilde{p}_N}\{X\}}}$ is a commonly used characterization of the scaling parameter from the literature. 
The hyperparameters of the rest of the methods are tunned as specified in their original formulations. Indeed, all the methods except ARL require hyperparameter tuning, and we have considered the same amount of possible values (20) for each. In our case 10 for $\sigma$ and 10 for $\lambda_0$.

\subsubsection{Datasets} 

We examine several widely used datasets in the field of algorithmic fairness: 


\paragraph{ ACSEMPLOYMENT \cite{ding2021retiring}} The \textsc{ACSEmployment} dataset, sourced from the American ACS data, is derived from publicly accessible information offered by the US Census. Specifically designed for predicting the employment status of individuals aged between 16 and 90, this dataset encompasses data from all fifty U.S. states and Puerto Rico. The dataset spans five years of data collection, covering the period from 2014 to 2018, inclusive. Notably, \textsc{ACSEmployment} incorporates sensitive attributes, including \textit{gender}, \textit{race}, and \textit{disability}. We preprocess the data in accordance with the specifications outlined in \cite{ding2021retiring}. We conduct extensive experiments using 2018 data from multiple states with distinct characteristics, exploring different variations of the classification problem. Specifically, we consider two alternative definitions of the sensitive groups: one based on \textit{race} (multi-valued one-dimensional sensitive attribute) and another on intersectional groups defined by both \textit{race} and \textit{gender}.
For instance, Figure \ref{fig:acs_2states} displays the distribution of various racial groups for the two states under examination in the experimental section, namely Hawaii and Indiana, with race identified as the sensitive attribute. The figure demonstrates that the state of Hawaii exhibits greater racial diversity.

\begin{figure*}[ht]
\begin{center}
\centerline{\includegraphics[width=0.9\textwidth]{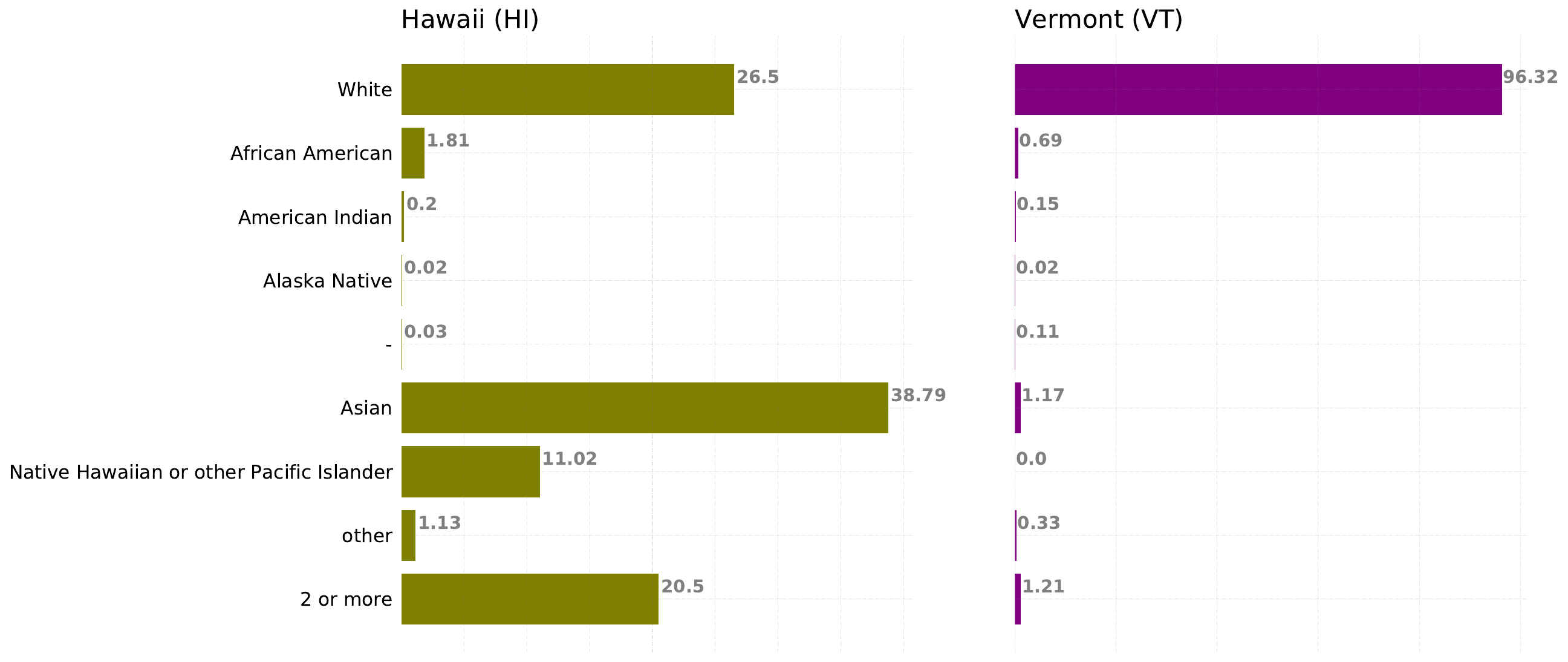}}
\caption{ The distribution of the sensitive attribute \textit{race} for the states of Hawaii and Indiana in the ACSEmployment dataset. The state of Hawaii exhibits greater racial diversity. }
\label{fig:acs_2states}
\end{center}
\end{figure*}

\paragraph{ACSINCOME \cite{ding2021retiring}} This dataset is also sourced from the ACS data, and involves the task of predicting whether whether the yearly income of a US working adults is above \$50K. As in the previous task, we perform extensive experiments using 2018 data from multiple states with diverse characteristics, examining different variations of the classification problem regarding the characterization of the sensitive groups. We consider two definitions of the sensitive groups: one based on \textit{race} (multi-valued one-dimensional sensitive attribute) and another on intersectional groups defined by both \textit{race} and \textit{gender}.


\paragraph{COMPAS} This ProPublica dataset encompasses information gathered on the utilization of the \textsc{COMPAS} risk assessment tool in Broward County, Florida. The dataset is composed of 6,167 individuals, each of which is described by attributes like the count of juvenile felonies, the degree of the current arrest charge, and sensitive attributes such as \textit{race} and \textit{gender}. Each individual in the dataset is associated with a binary "recidivism" outcome, indicating whether they were rearrested within two years after the charge recorded in the data. Hence, the objective of this classification task is to anticipate whether an individual is likely to recidivate within a two-year period.

\paragraph{GERMAN CREDIT} The German Credit dataset collects information about several individuals created from a German bank\textquotesingle s data from 1994. It contains details about the socioeconomic situation of individuals: namely its employment, housing, savings, etc. Besides, the set of features includes some sensitive information as well, such as, gender, age and marital status. In this classification task, the objective is to predict whether an individual should obtain a good or bad credit score. This dataset is considerably smaller than the previous ones, containing only 1,000 instances and 20 features, and it is publicly available in the UCI repository.\footnote{\url{http://archive.ics.uci.edu/ml/index.php}} In our experiments, we have considered the sensitive attributes personal status (which is a combination between gender and marital status) and age. Regarding the sensitive attribute age, we have divided the instances into two groups based on whether their age is greater than 25 or less than or equal to 25. This distinction reflects real socioeconomic and financial differences relevant to credit risk modeling. Younger individuals often lack established credit histories and income stability, whereas older individuals tend to have more financial experience and responsibilities.

\subsubsection{Evaluation} 

We design a set of experiments to evaluate \texttt{SPECTRE} from multiple perspectives. First, we compare its performance against SOTA methods across real-world classification tasks (Section \ref{sec:spectre_sota_results}). Results on overall accuracy and fairness guarantees are averaged over 10 random train/validation/test splits, with 30\% of each dataset reserved for testing and 20\% of the training data used for validation. Fairness is assessed through worst-group accuracy. Note that methods assuming full knowledge of sensitive attributes (MMPF \cite{martinez2020minimax}, GDRO \cite{sagawa2020distributionally}, and MMPFrel \cite{diana2021minimax}) enjoy a competitive advantage.

We further analyze the theoretical guarantees of our framework, validating both the proposed bounds and the corresponding worst-case distributions for the group errors (Section \ref{sec:experiments_bounds}) and the overall error (Section \ref{sec:experiments_bounds_overall}). Finally, we as well examine the computational cost of the proposed framework (Section \ref{sec:spectre_cost}). For all these experiments, we employ the same experimental setup as above unless otherwise specified.

\subsubsection{Further Details}
\label{ap:resources}

All the experiments have been run in a cluster of 1000 Intel Xeon Cascade Lake and Skylake CPUs. Information of the used assets and their associated licenses are available in Table \ref{tab:licenses_spectre}.

\begin{table}[h]
    \caption{Licenses of assets used.}
    \centering
    \begin{tabular}{cc}
        \toprule
        Asset & License \\
        \midrule
        folktables & MIT \\
        sklearn & BSD 3-Clause \\
        MRCpy & GPL 3.0 \\
        ARL & Apache 2.0 \\
        BPF & GPL 3.0 \\
        MMPF & GPL 3.0 \\
        MMPF(rel) & Apache 2.0 \\
        GDRO & MIT \\
        \bottomrule
        \addlinespace
    \end{tabular}
    \label{tab:licenses_spectre}
\end{table}

\subsection{Results}
\label{sec:spectre_sota_results}

\begin{figure}[h!] 
    \centering
    \subfigure[]{
    \includegraphics[width=0.89\textwidth]{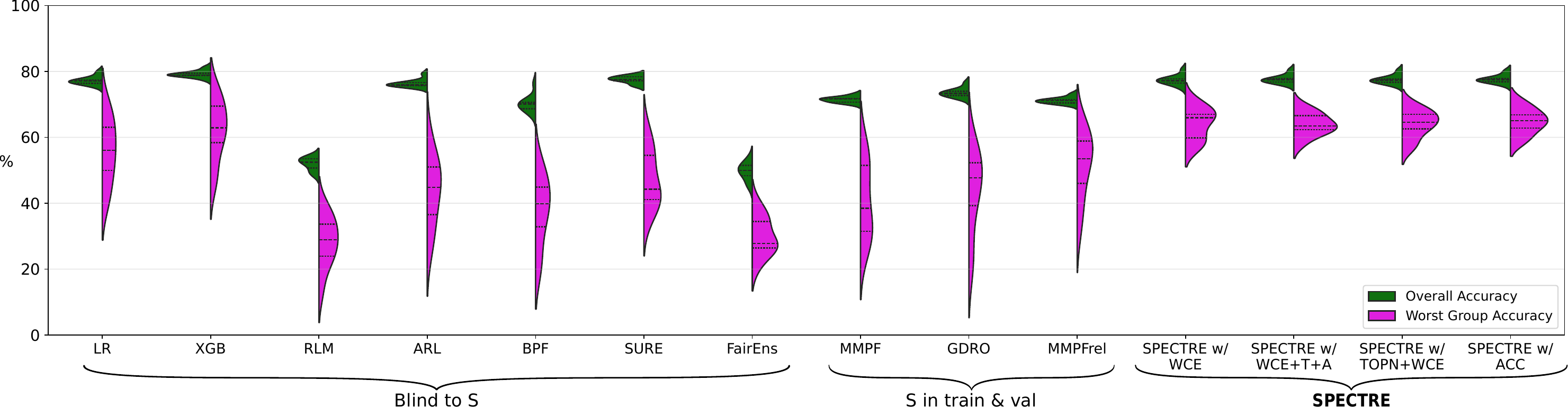}
    \label{fig:genexp_subplot1}
  }
  \hfill 
  \subfigure[]{
    \includegraphics[width=0.89\textwidth]{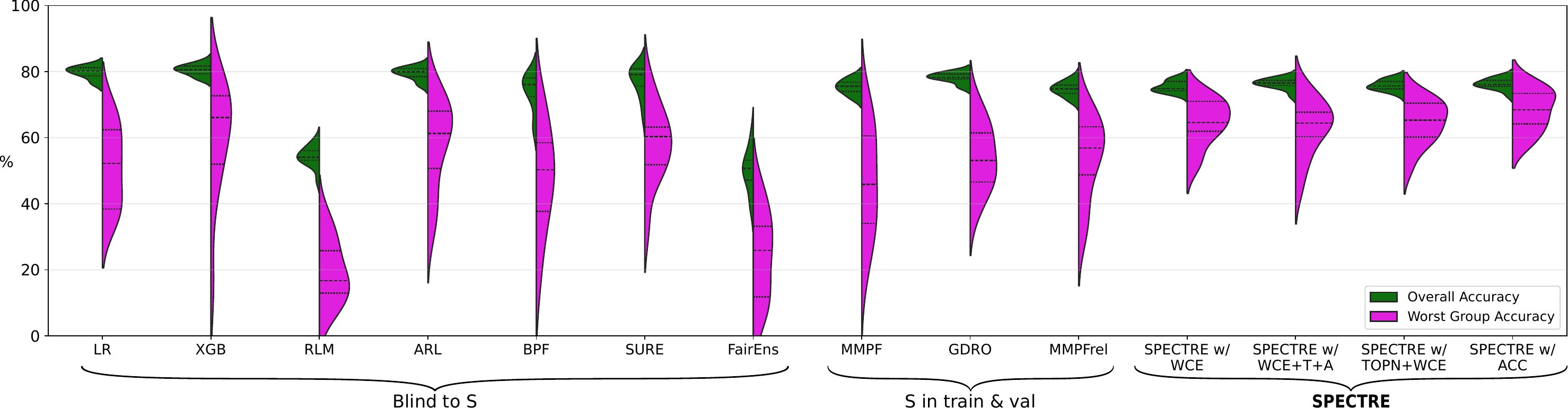}
    \label{fig:genexp_subplot2}
  }
  \hfill 
  \subfigure[]{
    \includegraphics[width=0.89\textwidth]{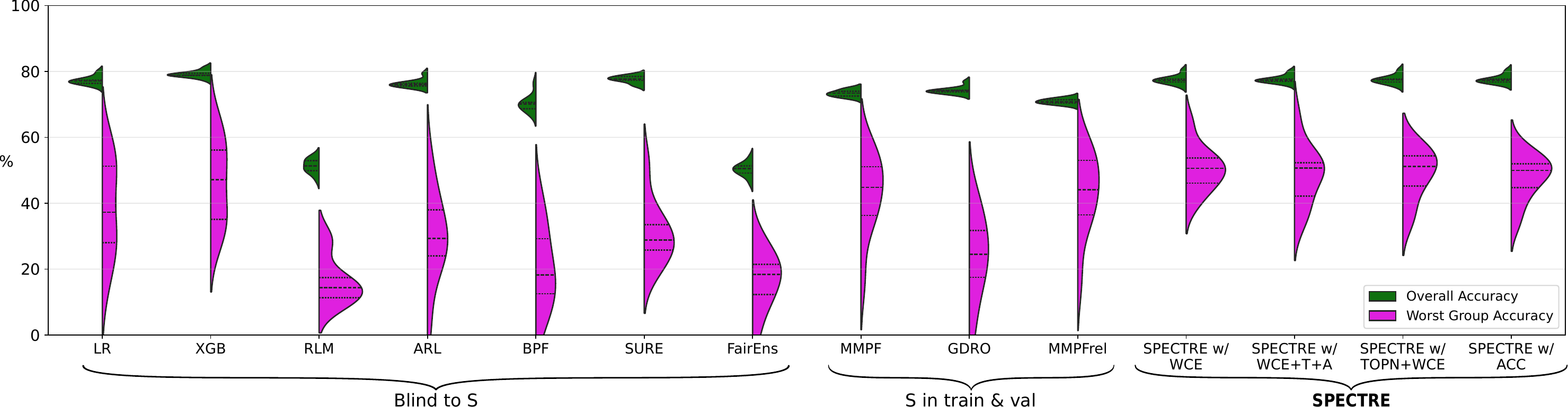}
    \label{fig:genexp_subplot3}
  }
  \hfill 
  \subfigure[]{
    \includegraphics[width=0.89\textwidth]{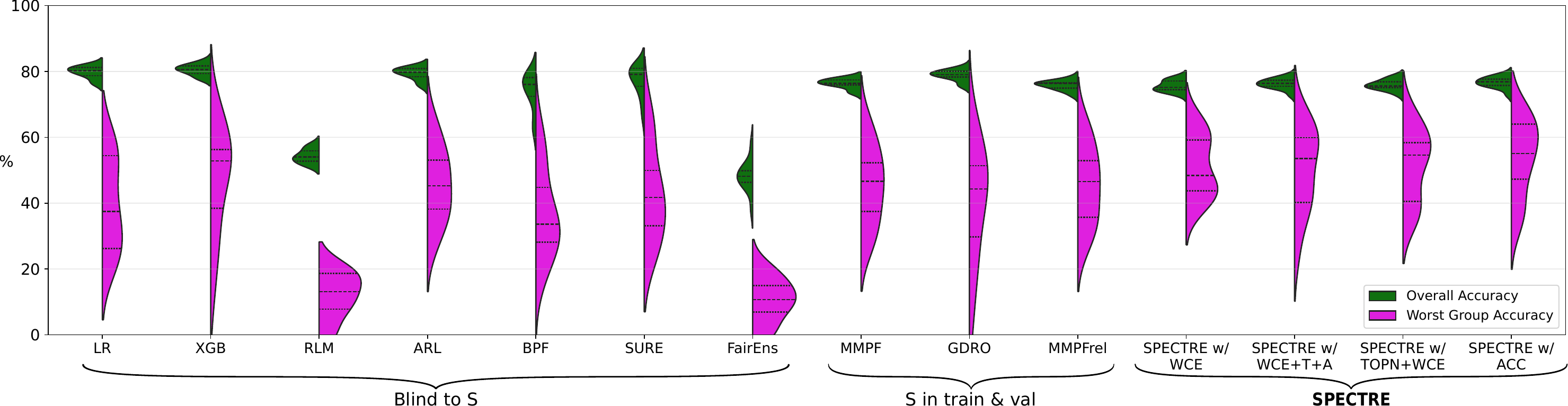}
    \label{fig:genexp_subplot4}
  }
    \caption{The distributions of the average \color{OliveGreen} \textbf{overall accuracy} \color{black} and \color{Magenta} \textbf{worst group accuracy} \color{black} of SOTA approaches and \texttt{SPECTRE} in (a,c) ACSEmployment and (b,d) ACSIncome datasets, across different collections of 20 randomly selected US states, with (a,b) \textit{race} as the sensitive attribute and (c,d) considering intersectional groups based on \textit{race} and \textit{gender}. \texttt{SPECTRE} consistently outperforms SOTA methods in worst-group accuracy with a minimal overall accuracy reduction with respect to the best performing method (XGB). 
    }
    \label{fig:general_results}
\end{figure}

\paragraph{Results for ACS datasets.} Figure \ref{fig:general_results} presents the results for the ACS-based classification tasks across 20 randomly sampled US states (40 different classification tasks in total), using \textit{race} as the sensitive attribute (see Figures \ref{fig:genexp_subplot1} and \ref{fig:genexp_subplot2}) and considering intersectional groups based on \textit{race} and \textit{gender} (see Figures \ref{fig:genexp_subplot3} and \ref{fig:genexp_subplot4}). 
In both tasks, \texttt{SPECTRE} consistently outperforms SOTA methods, delivering strong fairness guarantees across all examined states while maintaining a minimal overall accuracy reduction with respect to the classifier with the highest accuracy (XGBoost). In fact, it provides the highest average values on fairness guarantees together with the smallest interquartile range in comparison to SOTA approaches, even compared to those with access to demographic group information.
More precisely, \texttt{SPECTRE} provides fairness gains up to 20-40\% in the most vulnerable sensitive groups with a maximum accuracy drop of 1-4\% in the overall population, with respect to the classifier with the highest accuracy.

Moreover, these findings highlight the limitations of RRM approaches (with or without demographics) that constitute re-weighting schemes of the empirical training distribution (RLM, ARL, BPF, MMPF, GDRO and MMPFrel), often exhibiting excessive pessimism and over-reliance on outliers. Indeed, in some states, their performance on worst-group accuracy deteriorates significantly. In contrast, \texttt{SPECTRE} maintains stable worst-group accuracy across the different states, avoiding such degradations and ensuring more reliable performance and safe fairness guarantees. This holds true even when compared to minimax-fair approaches that have access to the sensitive information for all instances. 

As expected, worst-group accuracy tends to be lower when considering intersectional groups compared to analyses based solely on race. This underscores the importance of examining algorithmic fairness through an intersectional lens to reveal intra-group disparities. \texttt{SPECTRE} addresses this effectively, as it does not require explicitly defining all possible groupings in advance.

For a deeper exploration of \texttt{SPECTRE}’s behavior, and that of SOTA methods, on the ACS datasets, we refer the readers to Appendix \ref{ap:additional_results}: 
Appendix \ref{ap:detailed_table} provides a more detailed, state-level analysis on the performance; 
and Appendix \ref{ap:other_fairness_metrics} extends the study of fairness guarantees beyond worst-group accuracy to include two popular group fairness notions, EOp \cite{hardt2016equality} and DP \cite{dwork2012fairness}.

\paragraph{Results for COMPAS dataset.} We also study the performance of \texttt{SPECTRE} for the COMPAS dataset (see Table \ref{tab:COMPAS_race}).
Since it has been demonstrated that improving a classifier's performance for subgroups based on one sensitive attribute can negatively impact the balance of subgroup accuracy for another sensitive attribute, we consider two distinct sensitive attributes: \textit{race} (R) and \textit{gender} (G). In particular, we report the average ($\pm$ standard deviation) for overall accuracy (AV. ACC),  the worst-group accuracy (W ACC (R)) maximum accuracy disparity (MAX $\Delta$ ACC (R)) across groups based on race and the worst-group accuracy (W. ACC (G)) maximum accuracy disparity (MAX $\Delta$ ACC (G)) across groups based on gender.

\begin{table*}[ht!]
\caption{ Average ($_{\pm \textrm{std}}$) results on the COMPAS dataset, with the sensitive attribute being \textit{race}. Among the methods that are blind to demographics, the best result for each metric (column) is emphasized in \textbf{bold}, and the second-best result is \underline{underlined}. }
\label{tab:COMPAS_race}
\begin{center}
\begin{small}
\begin{sc}
\scalebox{0.95}{
\begin{tabular}{c>{\raggedright\arraybackslash}m{11em}ccccc}
\toprule
& & av.acc  & w.acc(r) & max$\Delta$acc(r) & w.acc(g)  & max$\Delta$acc(g)  \\
& Method & $(\uparrow)$ & $(\uparrow)$& $(\downarrow)$ & $(\uparrow)$ &  $(\downarrow)$  \\
\midrule
Base - & LR & $66.9_{\pm 1.0}$ & $\boldsymbol{59.5_{\pm 9.3}}$ & $33.9_{\pm 14.5}$ & $66.0_{\pm 1.5}$ & $\underline{3.0_{\pm 2.0}}$  \\
line & XGBoost & $67.1_{\pm 1.0}$ & $56.6_{\pm 12.6}$ & $\boldsymbol{28.9_{\pm 11.3}}$ & $\underline{66.4_{\pm 1.2}}$ & $3.9_{\pm 2.7}$ \\
\midrule
& RLM & $51.7_{\pm 10.2}$ & $36.6_{\pm 17.3}$ & $31.8_{\pm 10.8}$ & $46.0_{\pm 13.6}$ & $9.2_{\pm 3.9}$ \\
& ARL & $\underline{68.1_{\pm 1.1}}$ & $55.3_{\pm 19.6}$ & $\underline{29.0_{\pm 16.8}}$ & $67.2_{\pm 1.3}$ & $4.5_{\pm 1.5}$  \\
w/o & BPF & $58.6_{\pm 1.3}$ & $55.1_{\pm 2.2}$ & $35.5_{\pm 8.9}$ & $57.5_{\pm 1.4}$ & $6.0_{\pm 1.6}$  \\
Demo - & \texttt{SPECTRE} \& WCE & $67.9_{\pm 1.2}$ & $56.1_{\pm 16.0}$ & $32.2_{\pm 14.7}$ & $67.2_{\pm 1.2}$ & $3.2_{\pm 1.6}$ \\
graphics & \texttt{SPECTRE} \& WCE+T+A & $67.6_{\pm 1.1}$ & $\underline{59.0_{\pm 5.9}}$ & $29.3_{\pm 9.0}$ & $67.0_{\pm 1.1}$ & $\boldsymbol{2.8_{\pm 1.6}}$ \\
& \texttt{SPECTRE} \& TOPN+WCE & $\boldsymbol{68.2_{\pm 1.1}}$ & $53.9_{\pm 15.3}$ & $34.4_{\pm 14.9}$ & $\boldsymbol{67.5_{\pm 1.1}}$ & $3.3_{\pm 2.3}$ \\
& \texttt{SPECTRE} \& ACC & $67.9_{\pm 1.0}$ & $56.0_{\pm 15.9}$ & $32.6_{\pm 14.2}$ & $\underline{67.4_{\pm 1.1}}$ & $\boldsymbol{2.8_{\pm 2.1}}$ \\
\midrule
S in & MMPF  & $66.9_{\pm 1.0}$ & $59.5_{\pm 9.3}$ & $33.9_{\pm 14.6}$ & $66.1_{\pm 1.4}$ & $2.9_{\pm 2.0}$ \\ 
train & GDRO & $67.1_{\pm 1.3}$ & $56.6_{\pm 10.3}$ & $27.7_{\pm 15.4}$ & $66.1_{\pm 1.3}$ & $5.4_{\pm 1.6}$  \\
\& val & MMPFrel  & $66.9_{\pm 1.0}$ & $59.5_{\pm 9.3}$ & $33.9_{\pm 14.6}$ & $66.1_{\pm 1.4}$ & $2.9_{\pm 2.0}$ \\
\end{tabular}}
\end{sc}
\end{small}
\end{center}
\end{table*}

\begin{figure}[h!]
\begin{center}
\centerline{\includegraphics[width=0.7\textwidth]{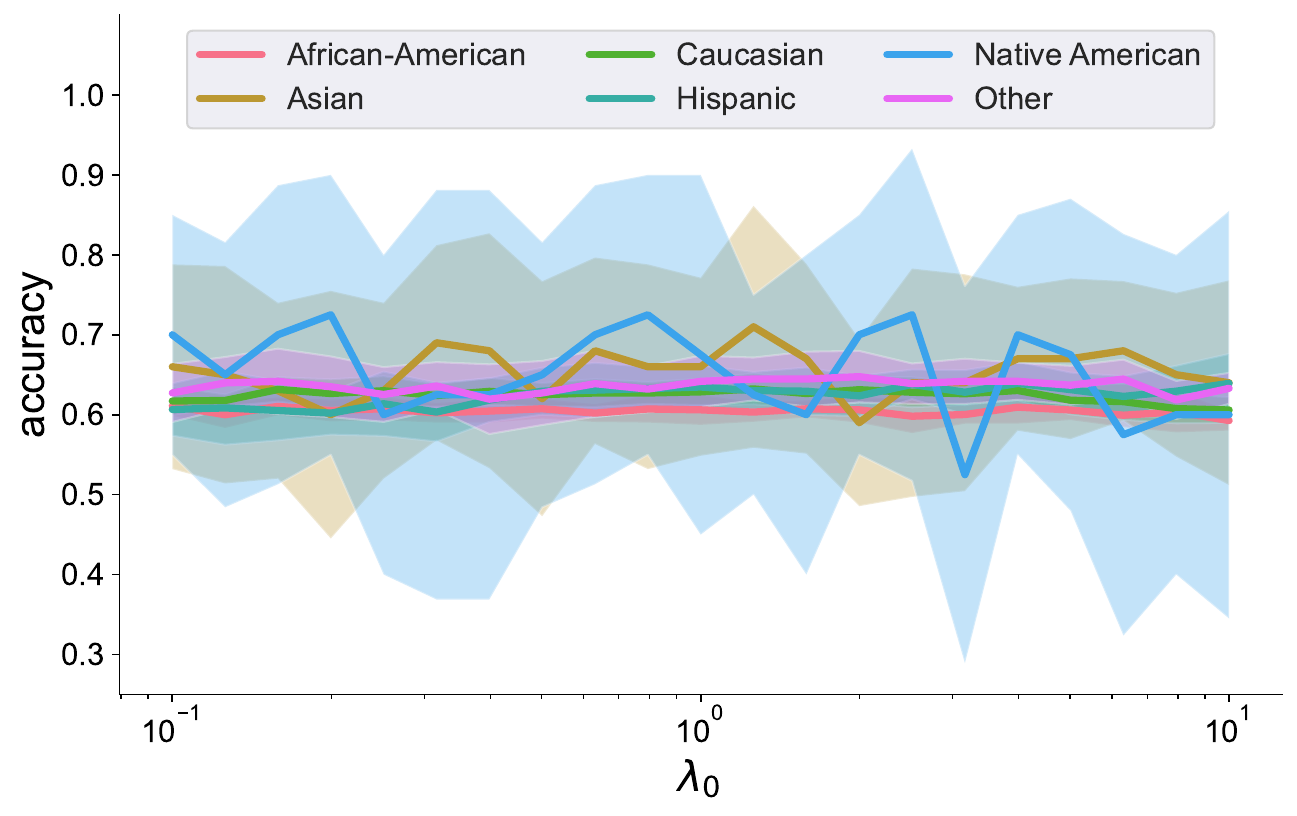}}
\caption{The predictive performance of MRC on the \textsc{COMPAS} dataset trained and tested individually for the different racial groups for varying value of $\lambda_0$. The worst-group accuracy does not experience significant enhancement with \texttt{SPECTRE}, as the MRC trained individually within the groups struggles to achieve a high level of predictive performance.}
\label{fig:compas_lambda}
\end{center}
\end{figure}

Among the blind methods (including baselines), \texttt{SPECTRE} attains the highest overall accuracy and second best worst-group accuracy, only 0.5\% lower than the best performing LR. However, we see that as the other RRM based approaches it faces challenges in significantly improving the worst-group accuracy.
\emph{What accounts for this phenomenon?}
As explained by \citet{slowik2021algorithmic}, a RRM framework can enhance the worst-group accuracy only when the model attains satisfactory performance during training on instances exclusively from that particular group.
Figure \ref{fig:compas_lambda} illustrates the performance of the traditional MRC trained on individual subgroups, with varying $\lambda_0$. It is evident that, for most of the racial groups, the MRC struggles to identify an effective classification rule when trained only on instances of that group. Consequently, when employing \texttt{SPECTRE} on the entire population, it struggles to enhance the worst-group performance.

\paragraph{Results for GERMAN CREDIT dataset.} We also study the performance of  \texttt{SPECTRE} and SOTA methods on the \textsc{German Credit} dataset. For one-dimensional sensitive attributes, we consider personal status (which combines gender and marital status) as the primary sensitive attribute for assessing fairness guarantees, as it exhibited the highest level of unfairness in terms of worst-group accuracy. We additionally evaluate fairness guarantees with respect to age groups, and extend the analysis to intersectional groups defined by the combination of age and personal status.
The performance of \texttt{SPECTRE} and the SOTA methods is evaluated in terms of overall accuracy and worst-group accuracy, along with two widely adopted fairness metrics: equality of opportunity (EOp) and demographic parity (DP).

The results for one-dimensional sensitive groups are presented in Table \ref{tab:results_german_per}, while those for intersectional groups are reported in Table \ref{tab:results_german_int}. These results show that \texttt{SPECTRE} achieves the highest worst-group accuracy across different group characterizations, including intersectional groups, with only a 3\% decrease compared to the method with the highest overall accuracy (ARL). When considering one-dimensional sensitive attributes, the performance of ARL and \texttt{SPECTRE} is comparable: ARL attains slightly higher overall accuracy (by 3\%) but lower worst-group accuracy (by 2\%). Other methods such as LR and XGB achieve marginally higher overall accuracy (by 1–2\%) than \texttt{SPECTRE} but exhibit substantially lower worst-group accuracy (by 7–8\%). However, when intersectional groups are taken into account, the worst-group accuracy of ARL, LR, and XGB decreases markedly—by approximately 20–30\%. This underscores the importance of evaluating multiple group characterizations to obtain a more informative and comprehensive understanding of the fairness guarantees across different dimensions. Besides, these results demonstrate that \texttt{SPECTRE} remains effective even when applied to smaller datasets.

Regarding the parity-based metrics EOp and DP, disparities in these measures are generally high, which is expected since the evaluated methods are designed to optimize worst-group accuracy and improve performance across all groups rather than enforce parity. Among the minimax fairness-oriented approaches, \texttt{SPECTRE} shows the lowest disparity values after SURE and FairEns, while achieving higher overall and worst-group accuracy than both. This suggests that strong results on parity-based metrics may conceal suboptimal performance for all sensitive groups and, consequently, for the population as a whole. These observations underscore the importance of considering multiple fairness criteria to obtain a more comprehensive and informative evaluation of the fairness guarantees of a model in empirical analyses.

\begin{table*}[ht!]
\caption{ Average ($\pm$ std) results on the \textsc{German Credit} dataset, considering personal status (gender + marital status) as the main sensitive attribute and age as the secondary sensitive attribute. We also report gender-group metrics to evaluate how gains in one form of sensitive attribute characterization affect another. Among the methods that are blind to demographics, the best result for each metric (column) is emphasized in \textbf{bold}, and the second-best result is \underline{underlined}. }
\label{tab:results_german_per}
\begin{center}
\begin{small}
\begin{sc}
\scalebox{0.78}{\begin{tabular}{clccccccc}
\toprule
& & av.acc  & w.acc  (p) & max$\Delta$acc(p) & EOp (p) & DP (p) & w.acc(a)  & max$\Delta$acc(a)  \\
& Method & $(\uparrow)$ & $(\uparrow)$& $(\downarrow)$ & $(\downarrow)$ & $(\downarrow)$ & $(\uparrow)$ &  $(\downarrow)$  \\
\midrule
Base- & LR & $72.7 \pm 2.0$ & $54.7 \pm 4.5$ & $25.8 \pm 4.5$ & $27.4 \pm 5.1$ & $21.4 \pm 9.9$ & $58.4 \pm 5.3$ & $16.8 \pm 6.3$ \\
lines & XGB & $\underline{73.5 \pm 2.2}$ & $55.1 \pm 8.4$ & $23.3 \pm 9.2$ & $20.7 \pm 7.4$ & $13.0 \pm 6.8$ & $\underline{65.0 \pm 5.1}$ & $10.0 \pm 5.7$ \\
\midrule
& RLM & $54.9 \pm 1.2$ & $46.5 \pm 6.7$ & $21.4 \pm 5.2$ & $25.3 \pm 10.2$ & $26.3 \pm 13.3$ & $48.4 \pm 7.2$ & $11.1 \pm 7.2$ \\
& ARL & $\boldsymbol{74.0 \pm 3.1}$ & $60.9 \pm 8.3$ & $19.5 \pm 9.0$ & $19.8 \pm 8.3$ & $21.4 \pm 7.9$ & $64.2 \pm 3.8$ & $11.5 \pm 4.6$ \\
& BPF & $47.1 \pm 1.4$ & $41.4 \pm 2.4$ & $16.8 \pm 5.6$ & $35.7 \pm 14.0$ & $38.0 \pm 15.6$ & $43.8 \pm 2.1$ & $10.0 \pm 5.1$ \\
w/o & SURE & $65.9 \pm 7.3$ & $54.5 \pm 8.6$ & $15.2 \pm 3.8$ & $\boldsymbol{6.2 \pm 12.4}$ & $\boldsymbol{4.6 \pm 9.2}$ & $58.4 \pm 5.8$ & $9.8 \pm 5.4$ \\
Demo- & FairEns & $52.7 \pm 15.1$ & $44.0 \pm 10.6$ & $\underline{13.5 \pm 3.4}$ & $\underline{8.2 \pm 6.3}$ & $\underline{9.6 \pm 9.2}$ & $49.2 \pm 13.1$ & $9.0 \pm 5.4$ \\
graphics & \texttt{SPECTRE}+WCE & $70.5 \pm 3.6$ & $61.4 \pm 7.9$ & $16.2 \pm 9.5$ & $22.1 \pm 8.1$ & $22.4 \pm 7.7$ & $63.0 \pm 0.7$ & $9.8 \pm 2.5$ \\
 & \texttt{SPECTRE}+WCE+T+A & $70.4 \pm 3.3$ & $61.6 \pm 6.7$ & $19.6 \pm 9.2$ & $26.4 \pm 10.9$ & $18.9 \pm 3.8$ & $61.1 \pm 5.8$ & $9.6 \pm 6.2$ \\
& \texttt{SPECTRE}+TOPN+WCE & $70.6 \pm 3.7$ & $\underline{62.0 \pm 6.7}$ & $\boldsymbol{13.0 \pm 7.9}$ & $16.5 \pm 11.0$ & $18.4 \pm 10.6$ & $\boldsymbol{68.4 \pm 3.9}$ & $\boldsymbol{3.6 \pm 1.7}$ \\
& \texttt{SPECTRE}+ACC & $71.1 \pm 3.4$ & $\boldsymbol{62.9 \pm 10.3}$ & $17.0 \pm 12.2$ & $20.9 \pm 10.7$ & $22.6 \pm 12.1$ &  $64.3 \pm 4.6$ & $\underline{8.0 \pm 2.3}$ \\
\midrule
S in & MMPF & $71.3 \pm 4.7$ & $49.8 \pm 8.9$ & $30.1 \pm 4.9$ & $24.5 \pm 8.3$ & $18.9 \pm 9.4$ & $62.9 \pm 6.3$ & $13.2 \pm 5.1$ \\
train & GDRO & $74.1 \pm 2.3$ & $54.5 \pm 7.5$ & $27.0 \pm 4.9$ & $22.3 \pm 7.1$ & $17.4 \pm 8.5$ & $63.3 \pm 4.3$ & $12.8 \pm 4.1$ \\
\& val & MMPFrel & $71.8 \pm 5.8$ & $56.7 \pm 8.7$ & $25.1 \pm 5.7$ & $25.3 \pm 9.2$ & $17.3 \pm 10.9$ & $63.1 \pm 8.5$ & $12.4 \pm 7.2$ \\
\end{tabular}}
\end{sc}
\end{small}
\end{center}
\end{table*}

\begin{table*}[ht!]
\caption{ Average ($\pm$ std) results on the \textsc{German Credit} dataset considering intersectional sensitive groups based on personal status and age. Among the methods that are blind to demographics, the best result for each metric (column) is emphasized in \textbf{bold}, and the second-best result is \underline{underlined}. }
\label{tab:results_german_int}
\begin{center}
\begin{small}
\begin{sc}
\scalebox{0.8}{\begin{tabular}{clccccc}
\toprule
& & av.acc  & w.acc & max$\Delta$acc & EOp & DP  \\
& Method & (${\uparrow}$) & (${\uparrow}$) & (${\downarrow}$) & (${\downarrow}$) & (${\downarrow}$) \\
\midrule
Base- & LR & $72.7 \pm 2.0$ & $12.5 \pm 15.8$ & $71.1 \pm 15.2$ & $73.3 \pm 30.4$ & $63.9 \pm 28.3$ \\
lines & XGB & $\underline{73.5 \pm 2.2}$ & $20.0 \pm 24.5$ & $64.8 \pm 29.5$ & $71.7 \pm 31.9$ & $64.4 \pm 23.7$ \\
\midrule
& RLM & $54.9 \pm 1.2$ & $29.6 \pm 19.0$ & $58.3 \pm 17.8$ & $62.9 \pm 23.2$ & $51.9 \pm 19.0$  \\
& ARL & $\boldsymbol{74.0 \pm 3.1}$ & $20.0 \pm 24.5$ & $61.4 \pm 26.4$ & $75.7 \pm 29.7$ & $70.8 \pm 19.6$ \\
 & BPF & $47.1 \pm 1.4$ & $34.5 \pm 9.0$ & $54.3 \pm 14.4$ & $75.5 \pm 16.1$ & $66.0 \pm 19.9$ \\
w/o& SURE & $65.9 \pm 7.3$ & $40.8 \pm 9.3$ & $45.9 \pm 19.2$ & $\underline{39.1 \pm 24.2}$ & $\underline{35.8 \pm 25.5}$ \\
Demo- & FairEns & $52.7 \pm 15.1$ & $25.9 \pm 14.0$ & $47.5 \pm 20.3$ & $\boldsymbol{36.5 \pm 41.4}$ & $\boldsymbol{32.8 \pm 37.5}$ \\
graphics & \texttt{SPECTRE}+WCE & $70.5 \pm 3.6$ & $\underline{42.3 \pm 21.2}$ & $\underline{45.2 \pm 22.6}$ & $54.5 \pm 22.6$ & $46.1 \pm 16.9$  \\
 &  \texttt{SPECTRE}+WCE+T+A & $70.4 \pm 3.3$ & $41.1 \pm 21.0$ & $46.4 \pm 24.3$ & $52.8 \pm 24.5$ & $45.2 \pm 16.5$ \\
&  \texttt{SPECTRE}+TOPN+WCE & $70.6 \pm 3.7$ & $41.5 \pm 21.9$ & $45.3 \pm 25.1$ & $48.4 \pm 23.7$ & $44.9 \pm 16.3$ \\
&  \texttt{SPECTRE}+ACC & $71.1 \pm 3.4$ & $\boldsymbol{44.1 \pm 20.5}$ & $\boldsymbol{43.7 \pm 21.7}$ & $54.8 \pm 23.1$ & $45.2 \pm 15.5$  \\
\midrule
S in & MMPF & $71.3 \pm 3.4$ & $17.8 \pm 24.5$ & $64.7 \pm 23.7$ & $75.3 \pm 32.4$ & $69.3 \pm 23.8$ \\
train & GDRO & $73.7 \pm 2.6$ & $18.9 \pm 23.2$ & $63.4 \pm 24.8$ & $74.2 \pm 31.7$ & $67.7 \pm 22.7$ \\
\& val & MMPFrel & $72.8 \pm 5.1$ & $19.2 \pm 24.1$ & $62.1 \pm 25.2$ & $72.1 \pm 33.8$ & $66.5 \pm 23.4$ \\
\end{tabular}}
\end{sc}
\end{small}
\end{center}
\end{table*}

\color{black}

\subsection{Performance Guarantees and Worst-Case Distributions on Group Errors}
\label{sec:experiments_bounds}

In this section, we validate the performance bounds of \texttt{SPECTRE} on group-specific errors and provide additional insights into the worst-case distributions that derive the extremal performance of each group. We consider the same toy dataset from Section \ref{sec:sigma_toy} and split it into train, validation, and test sets as in previous experiments. \texttt{SPECTRE} is trained using the training and validation sets, and its performance is evaluated on the test set, which we refer to as the true error.
The bounds are calculated using a subset (30\%) of the training set, where sensitive group information is assumed to be available. Moreover, we employ the 0-1 loss to solve optimization problem (\ref{eq:optim-group}).
These results demonstrate that, with a portion of the training data containing available sensitive information, it is possible to confidently estimate a feasible range for \texttt{SPECTRE}'s performance across different sensitive groups. Additionally, this allows for identifying the worst-case distributions that lead to extremal performance in \texttt{SPECTRE}.

\subsubsection{Bounds on Group Error} 

\begin{figure}[!h] 
    \centering
    \subfigure[]{
            \includegraphics[width=0.75\textwidth]{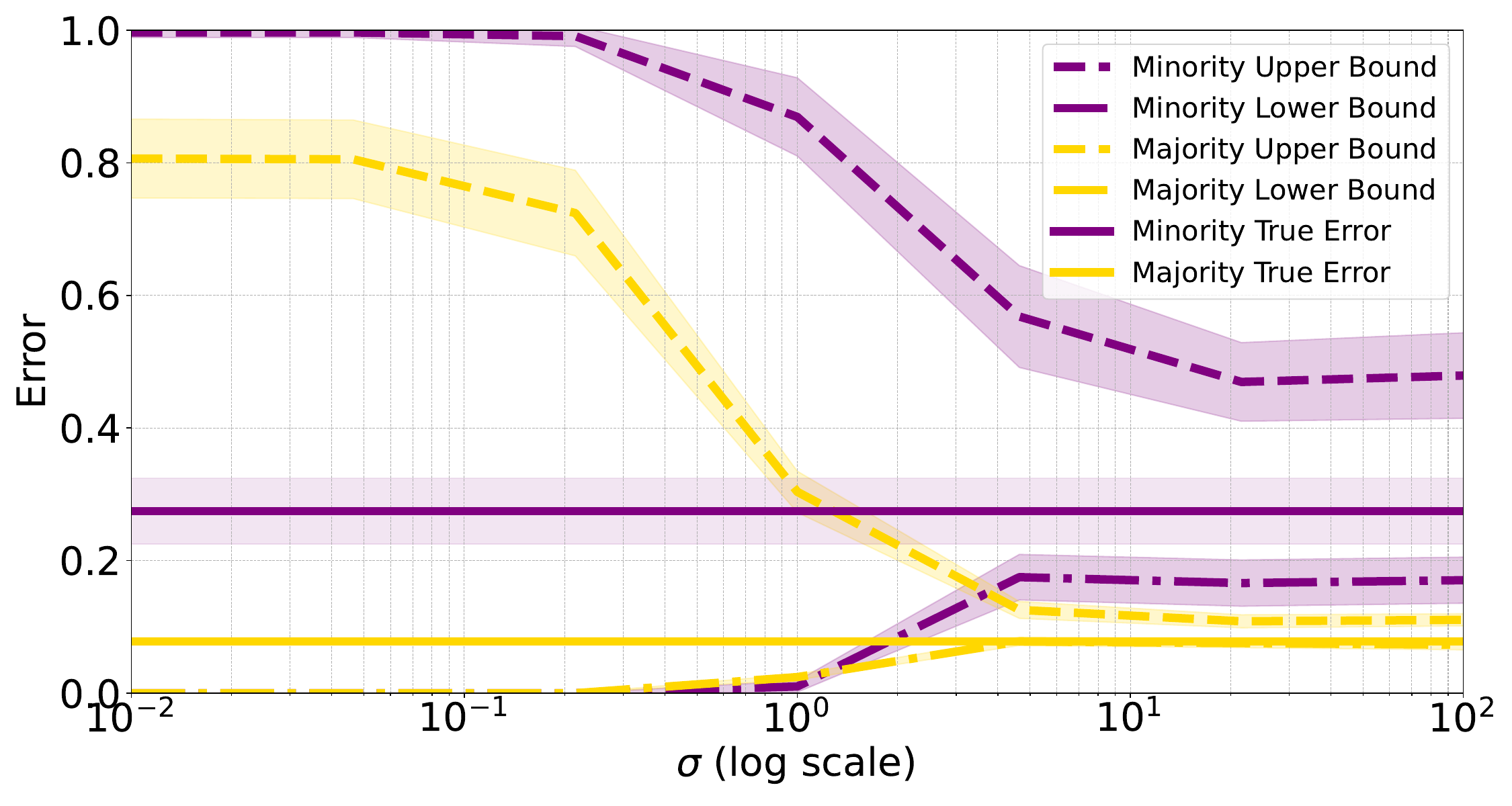}
            \label{fig:bounds_toy_subplot1}
        }
        \hfill
        \subfigure[]{
            \includegraphics[width=0.75\textwidth]{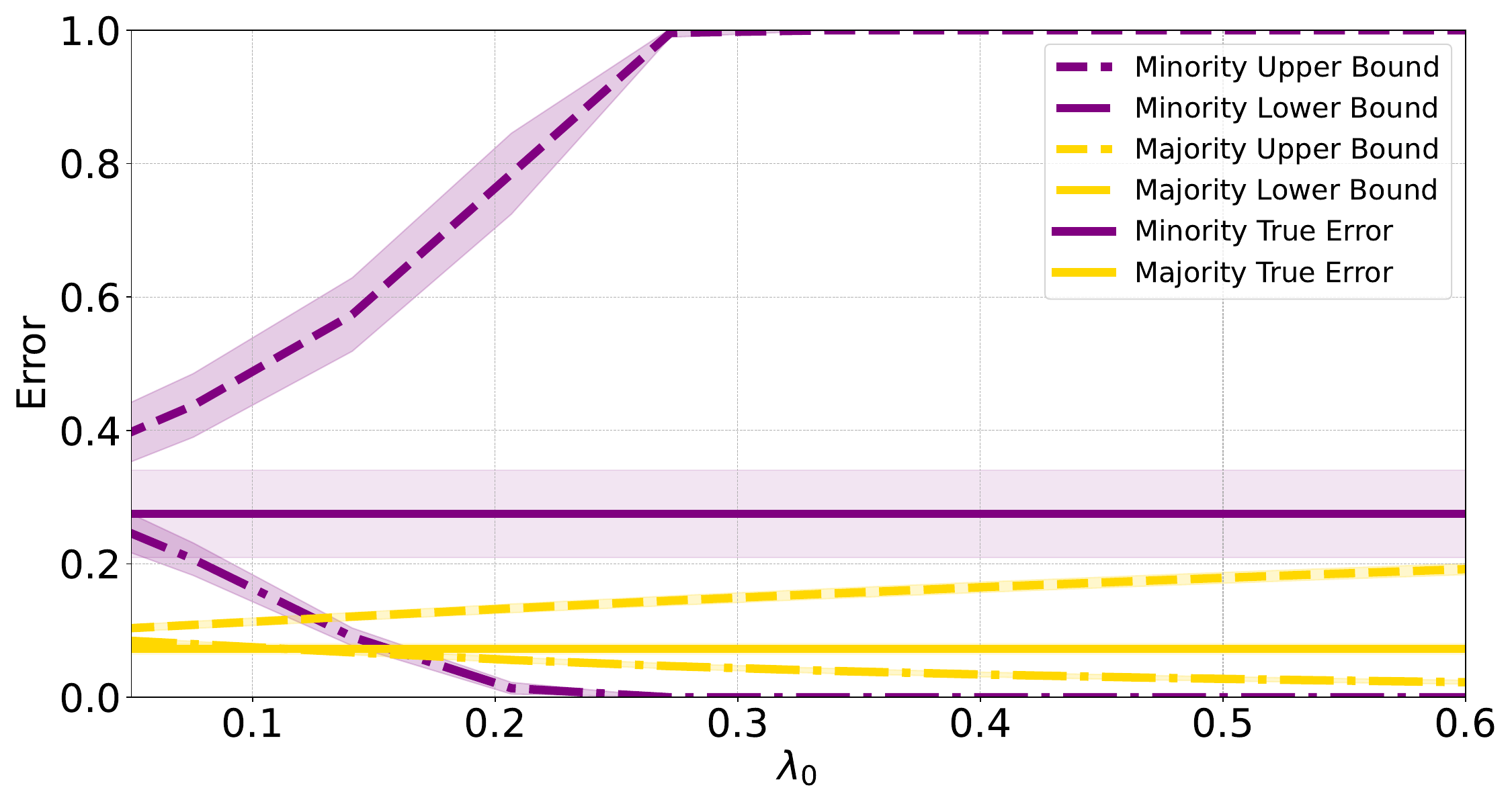}
            \label{fig:bounds_toy_subplot2}
        }
    \caption{The performance bounds of \texttt{SPECTRE} for the \color{Purple} \textbf{minority group ($\boldsymbol{S=0}$)} \color{black} and the \color{Goldenrod} \textbf{majority group ($\boldsymbol{S=1}$)} \color{black} under varying (a) scaling parameter ($\sigma$) and (b) maximum allowed divergence ($\lambda_0$). For (a), we fix $\lambda_0 = 0.1$, and for (b), we fix $\sigma = 10$. 
    }
    \label{fig:bounds_toy}
\end{figure}

Figure \ref{fig:bounds_toy} illustrates the group error bounds for varying values of $\sigma$ and $\lambda_0$. In all cases, the bounds are tighter for the majority group, reflecting greater confidence in \texttt{SPECTRE}'s performance for this group, as it is better characterized. On the contrary, \texttt{SPECTRE}'s performance on the minority group is more vulnerable to shifts in deployment environments. For the same value of $\lambda_0$ and, therefore, $\boldsymbol{\lambda}$ (which defines the maximum deviation from the empirical distribution), the worst-case loss is consistently higher for the unprivileged group.
Additionally, Figure \ref{fig:bounds_toy_subplot1} shows that the values yielding the tightest bounds align closely with the near-optimal $\sigma$ values for \texttt{SPECTRE} (see Figure \ref{fig:sigma_subplot2}). This suggests that the optimal spectrum characterization used during \texttt{SPECTRE}'s training can also be leveraged to compute effective performance bounds.
Moreover, higher values of $\lambda_0$ lead to looser bounds, as broader confidence vectors allow for greater deviation of the worst-case distribution from the empirical distribution. For further empirical validation of the bounds on real-world datasets, see Appendix \ref{ap:bounds_real}.

\subsubsection{Worst-Case Distributions}

In this section, we analyze the worst-case distributions that drive the extremal performance of \texttt{SPECTRE} for both minority and majority groups under varying values of $\lambda_0$ in the toy dataset. The original dataset statistics are presented in Figure~\ref{fig:toy_stats}, while Figure~\ref{fig:extremal_toy} illustrates the corresponding worst-case distributions for each group across the different $\lambda_0$ settings.

\begin{figure}[h!]
\begin{center}
\centerline{\includegraphics[width=0.5\textwidth]{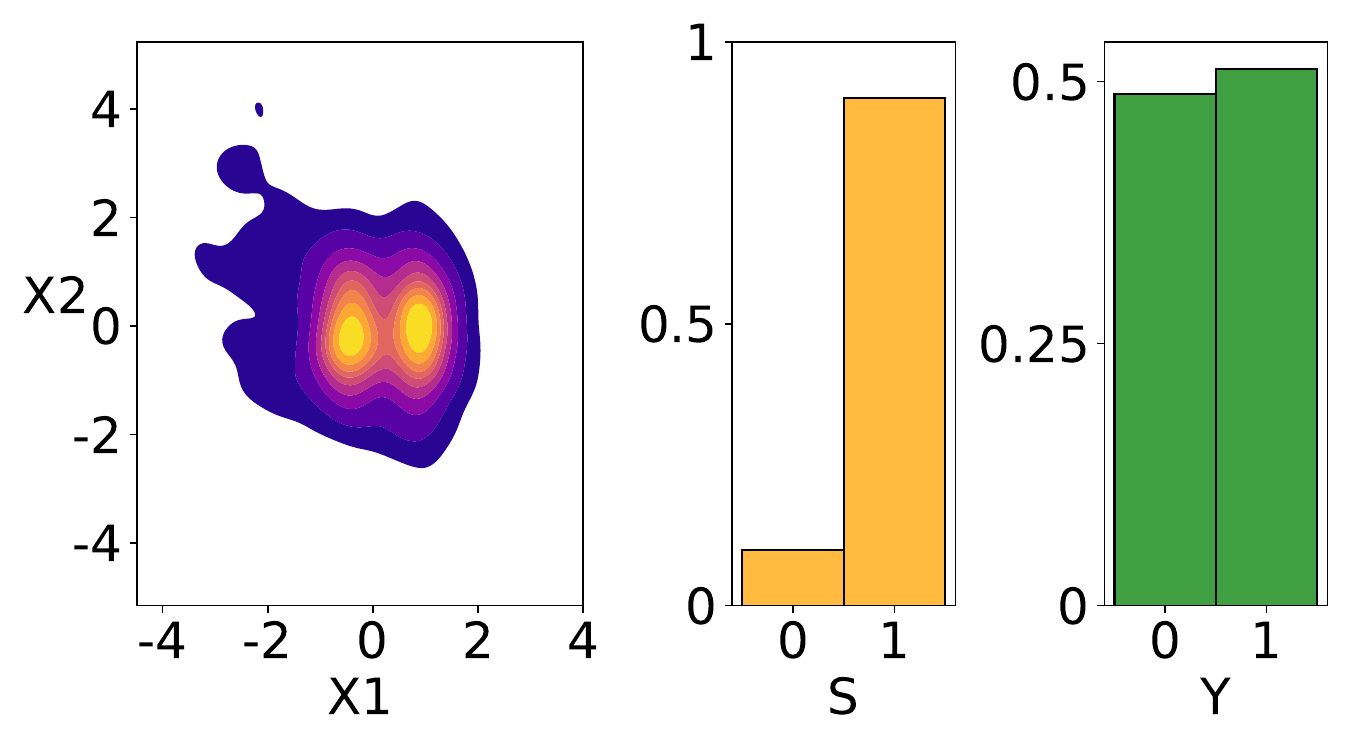}}
\caption{The statistics of the original toy dataset. From left to right the figures show the joint distribution of $X_1$ and $X_2$, and the marginal distributions of $S$ and $Y$.}
\vskip -0.2in
\label{fig:toy_stats}
\end{center}
\end{figure}

\begin{figure}[h!] 
    \centering
    \subfigure[$\lambda_0 = 0.1$]{
    \includegraphics[width=0.48\textwidth]{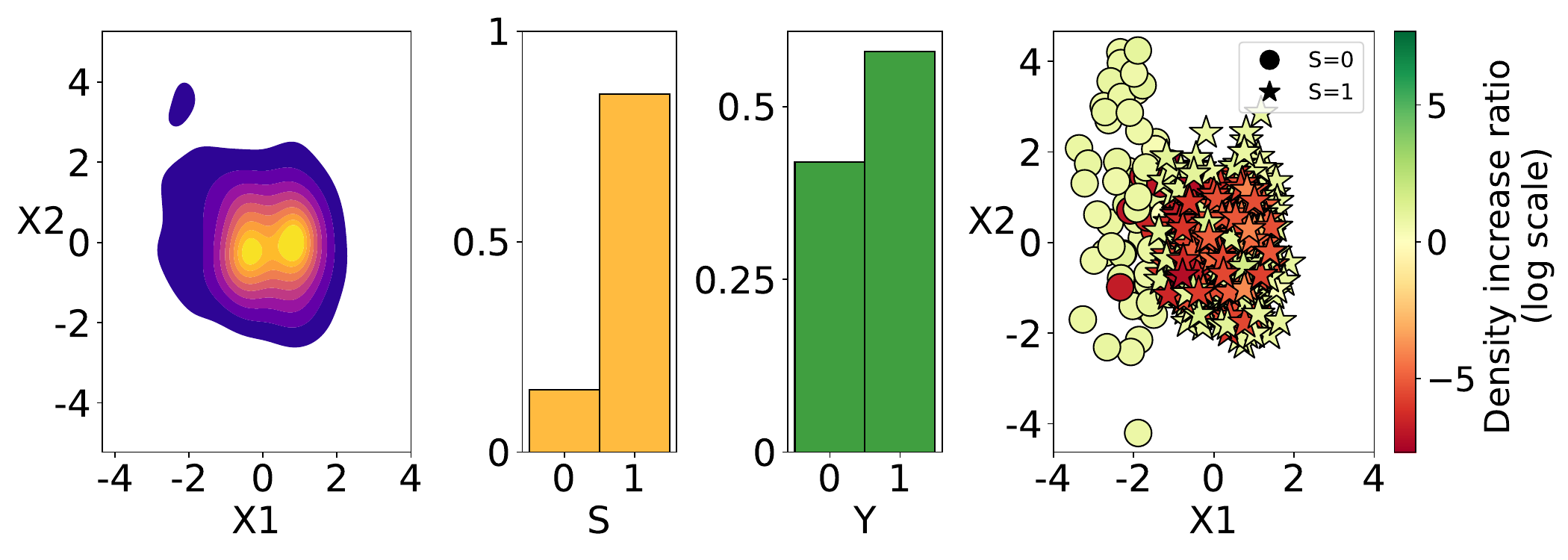}
    \label{fig:subplot1_extreme_minority}
  }
  \hfill 
  \subfigure[$\lambda_0 = 0.1$]{
    \includegraphics[width=0.48\textwidth]{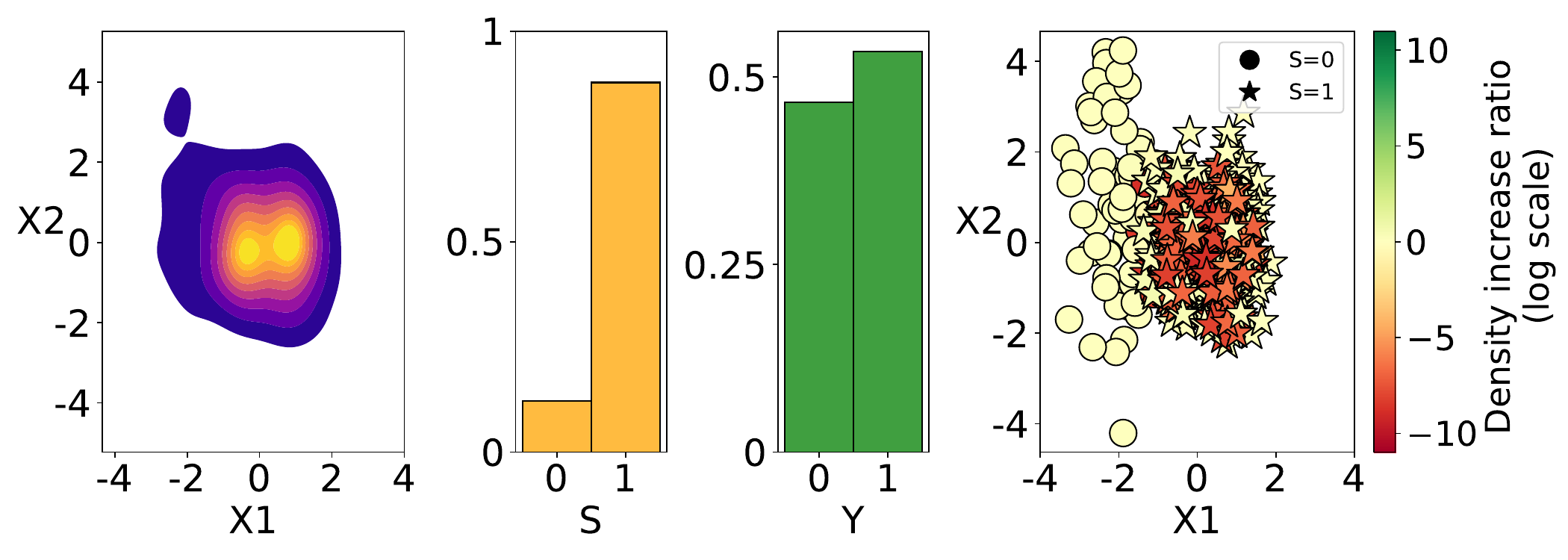}
    \label{fig:subplot1_extreme_majority}
  }
  \hfill 
  \subfigure[$\lambda_0 = 0.5$]{
    \includegraphics[width=0.48\textwidth]{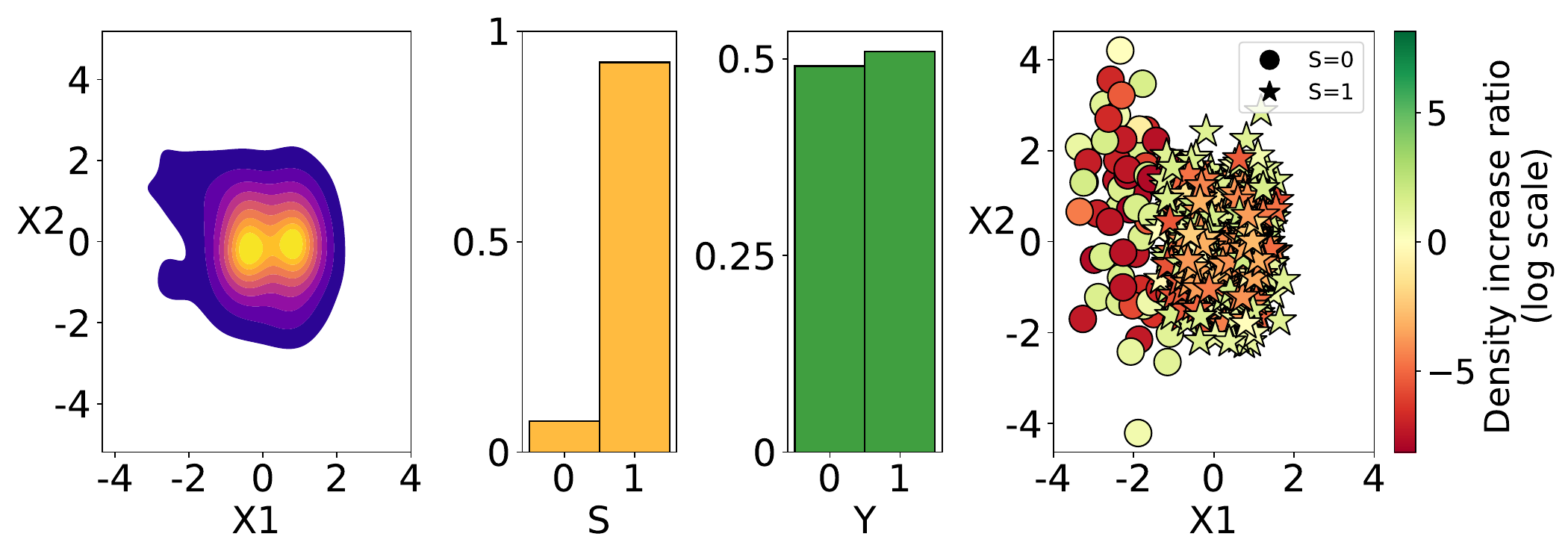}
    \label{fig:subplot2_extreme_minority}
  }
  \hfill 
  \subfigure[$\lambda_0 = 0.5$]{
    \includegraphics[width=0.48\textwidth]{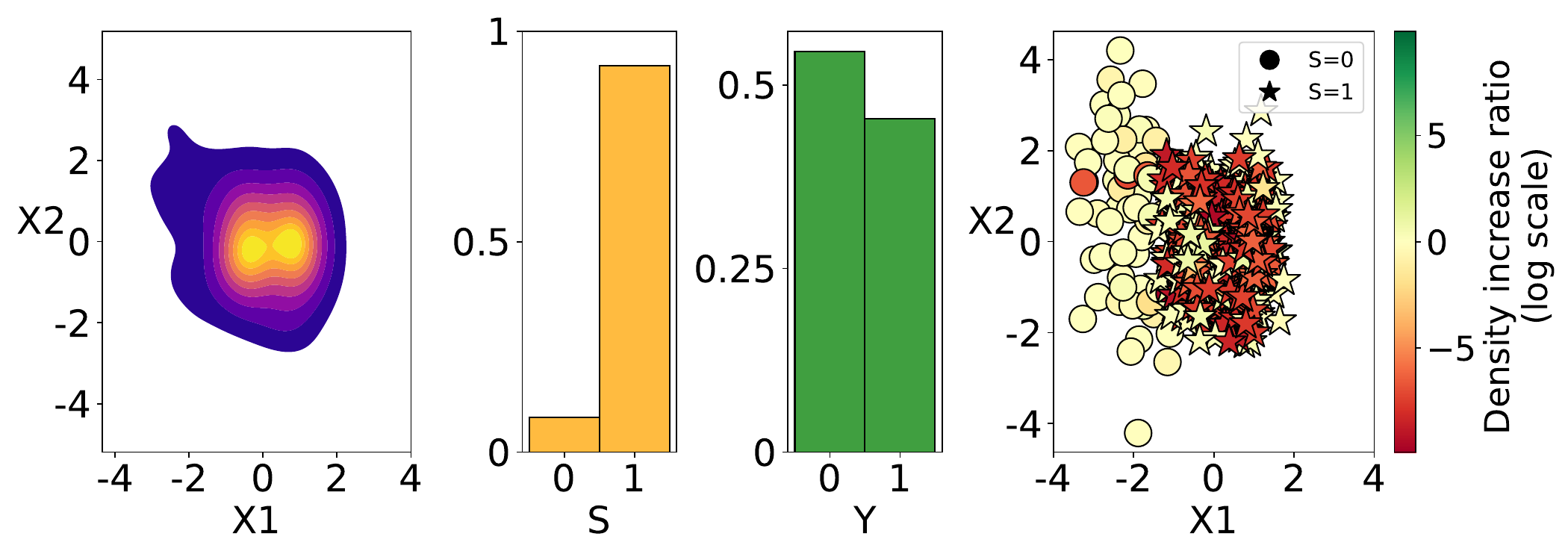}
    \label{fig:subplot2_extreme_majority}
  }
  \hfill
  \subfigure[$\lambda_0 = 1.0$]{
    \includegraphics[width=0.48\textwidth]{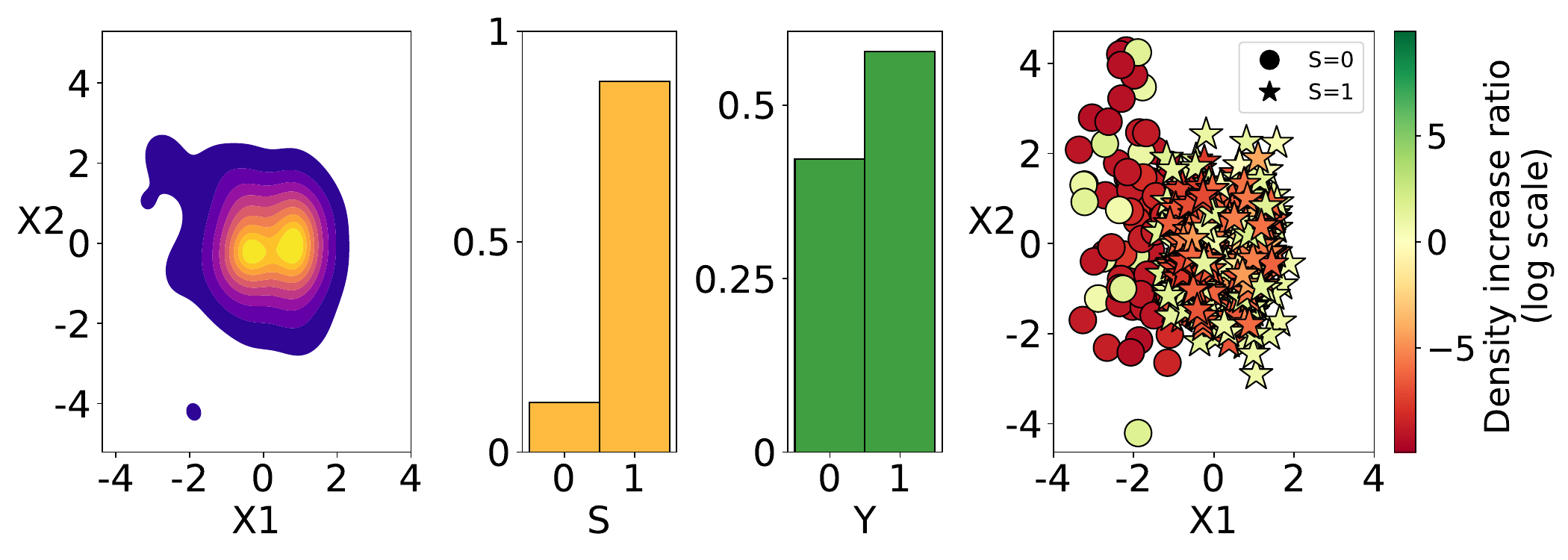}
    \label{fig:subplot3_extreme_minority}
  }
  \hfill 
  \subfigure[$\lambda_0 = 1.0$]{
    \includegraphics[width=0.48\textwidth]{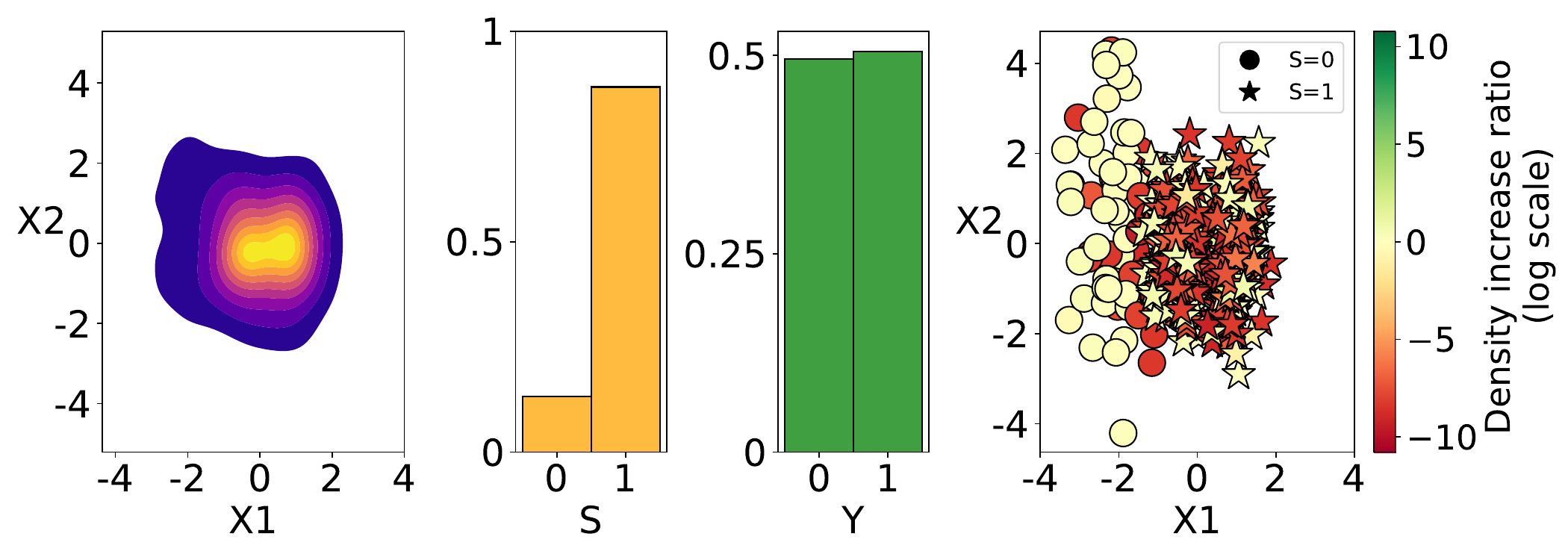}
    \label{fig:subplot3_extreme_majority}
  }
  \hfill 
  \subfigure[$\lambda_0 = 5.0$]{
    \includegraphics[width=0.48\textwidth]{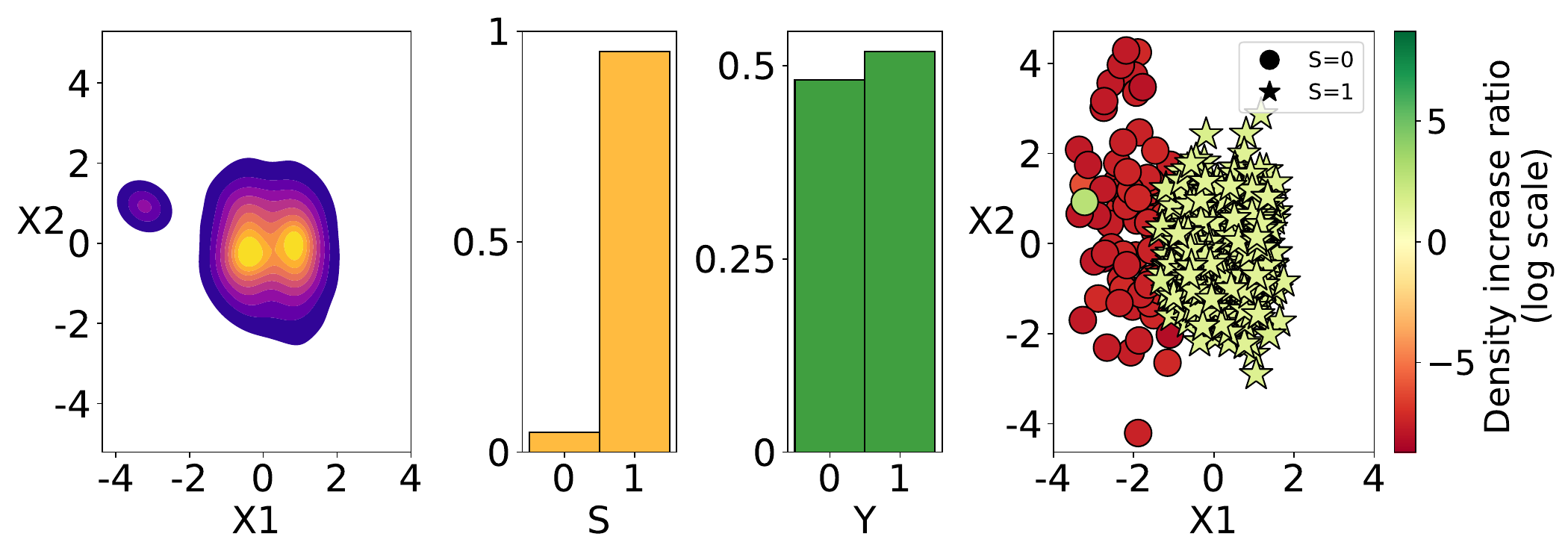}
    \label{fig:subplot4_extreme_minority}
  }
  \hfill 
  \subfigure[$\lambda_0 = 5.0$]{
    \includegraphics[width=0.48\textwidth]{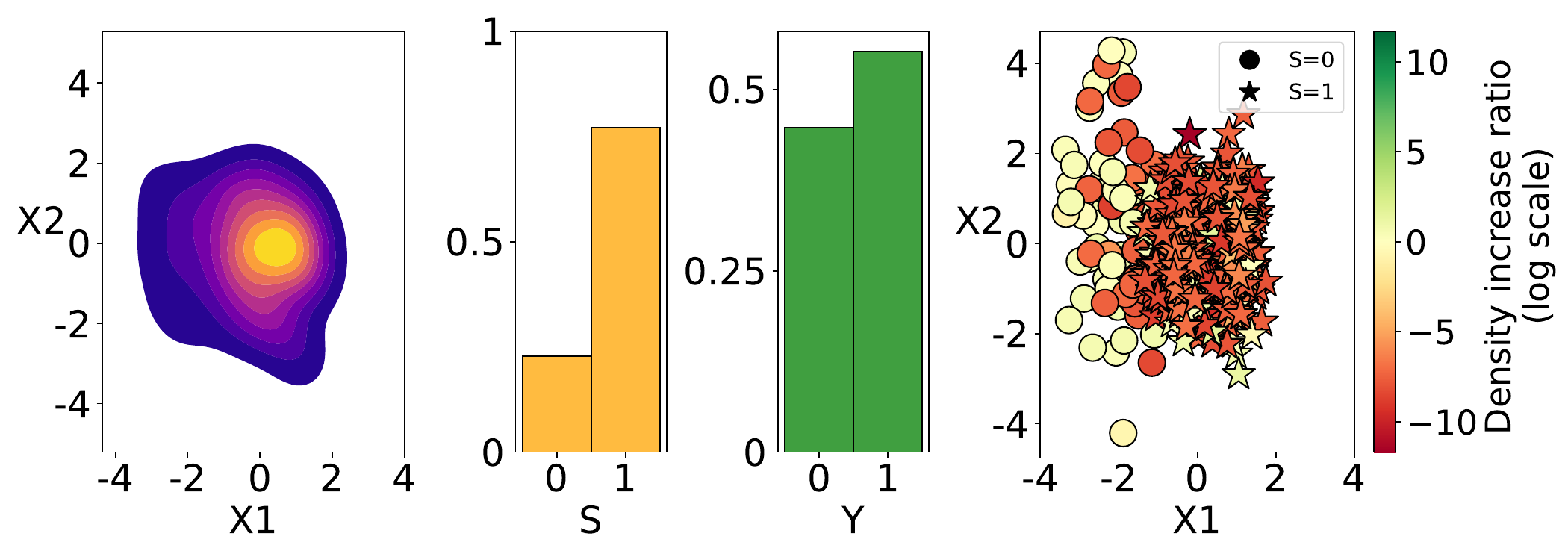}
    \label{fig:subplot4_extreme_majority}
  }
    \caption{The statistics of the worst-case distributions w.r.t. \texttt{SPECTRE} for the minority (left column) and majority (right column) groups, with $\lambda_0 = \{ 0.1, 0.5, 1.0, 5.0\}$.
    In all the cases we set $\sigma = 5$. From left to right the figures show the joint distribution of $X_1$ and $X_2$, the marginal distributions of $S$ and $Y$ and the up-weighing (green) and down-weighing (red) of the instances to create the worst-case distribution. 
    }
    \label{fig:extremal_toy}
\end{figure}

Intuitively, higher values of $\lambda_0$ allow for greater discrepancies between the worst-case and the empirical distributions of the groups, as evidenced by greater variations in the probability given to the points.
The figure also demonstrates that, when analyzing the worst-case scenario for a specific group, the probabilities of other groups are also adjusted to satisfy the constraints of the overall distribution (Equation \ref{eq:uncertainty}). In particular, the worst-case distribution does not simply up-weight instances from the target group while down-weighting all others; rather, it modifies the probabilities across groups to maintain feasibility.  
This effect is especially interesting in the case of \( \lambda_0 = 5.0 \), where, when computing the worst-case distribution for the minority group, nearly all minority group instances are down-weighted except for a single point, while all majority group instances are slightly up-weighted. 
That is, the extremal distribution that provides the upper bound for the error of minority group does not solely focus on the most outlier instance; it also increases the probabilities of many privileged instances to satisfy that the distance between the empirical distribution and this worst-case distribution is at most $\boldsymbol{\lambda}$.

With this, the estimated bounds are less influenced by outliers while maintaining a global population context.
This showcases the advantage of \texttt{SPECTRE} over SOTA approaches, particularly in mitigating pessimism and reducing over-reliance on outliers.

Besides, the effect of $\lambda_0$ is most significant in the worst-case distribution of the minority group. Small variations in $\lambda_0$ result in larger discrepancies between the worst-case distribution and the empirical distribution for the minority group. This occurs because the constraints on the maximum allowable deviation apply to the entire population, with the majority group carrying the greatest weight. The majority group compensates for substantial changes in the minority group's distribution.
In other words, it is easier to create extremal distributions that deviate significantly from the empirical distribution for the minority group, as the majority group helps satisfy the overall constraint. By contrast, for the majority group’s worst-case distribution, higher values of $\lambda_0$ are required to produce meaningful changes in its distribution when creating the worst-case scenario.

\clearpage

\subsection{Performance Guarantees and Worst-Case Distributions on Overall Error}
\label{sec:experiments_bounds_overall}

In the previous section, we analyzed the worst-case distributions for group errors, providing an explicit evaluation of the fairness guarantees offered by the classification rule derived from \texttt{SPECTRE}. However, it is also possible to compute the worst-case distributions for the overall risk, which accounts for the entire population. This allows us to study the out-of-sample generalization guarantees of the overall error and identify the worst-case distributions that lead to these worst-case errors. It is important to emphasize that these worst-case distributions correspond to the distributions the classification rule optimizes against. By examining them, we can directly observe the worst-case scenario that influenced the learned classification rule and analyze its characteristics. Notably, computing the worst-case error and worst-case distribution for the overall error does not require knowledge of the sensitive attributes of the instances. This means that all available training points can be utilized to compute the bounds and analyze the worst-case distributions.

\subsubsection{Worst-Case Distributions on Overall Error}
\label{ap:overall_extremal_distri}

Figure \ref{fig:all_extremal_lambda_ap} illustrates the worst-case distributions for the overall population across different values of \( \lambda_0 \). As \( \lambda_0 \) increases, the worst-case distribution deviates further from the empirical distribution, and the probability assigned to unprivileged instances increases. This demonstrates how modifying the confidence vector directly impacts the fairness guarantees of the classifier in terms of minimax fairness. 
In fact, in this case, the minority group represents the unprivileged group where the classifier incurs a high loss. Therefore, in the minimax formulation, the classifier focuses on these instances as much as the constraints allow.
However, even for very high values of \( \lambda_0 \), we observe that the worst-case distribution does not solely increase the probability of difficult points from the minority group. Instead, it also increases the probability of multiple instances of varying nature, which addresses the pessimism and over-reliance on outliers seen in existing RRM-based minimax-fair approaches (with or without access to demographic information). In other words, the worst-case scenario is not an unrealistic or outlier-driven realization, but one that still maintains a meaningful context.

\begin{figure}[h!] 
    \centering
    \subfigure[$\lambda_0 = 0.1$]{
    \includegraphics[width=0.48\textwidth]{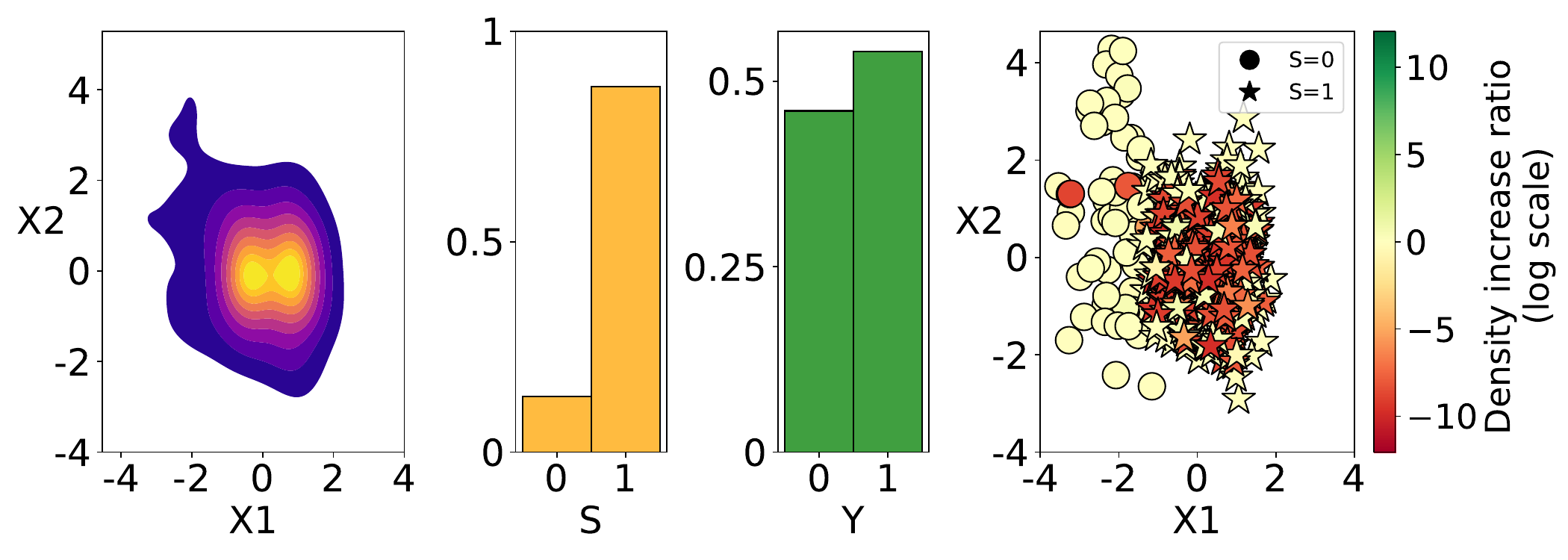}
    \label{fig:subplot1_extreme_all}
  }
  \hfill 
  \subfigure[$\lambda_0 = 0.5$]{
    \includegraphics[width=0.48\textwidth]{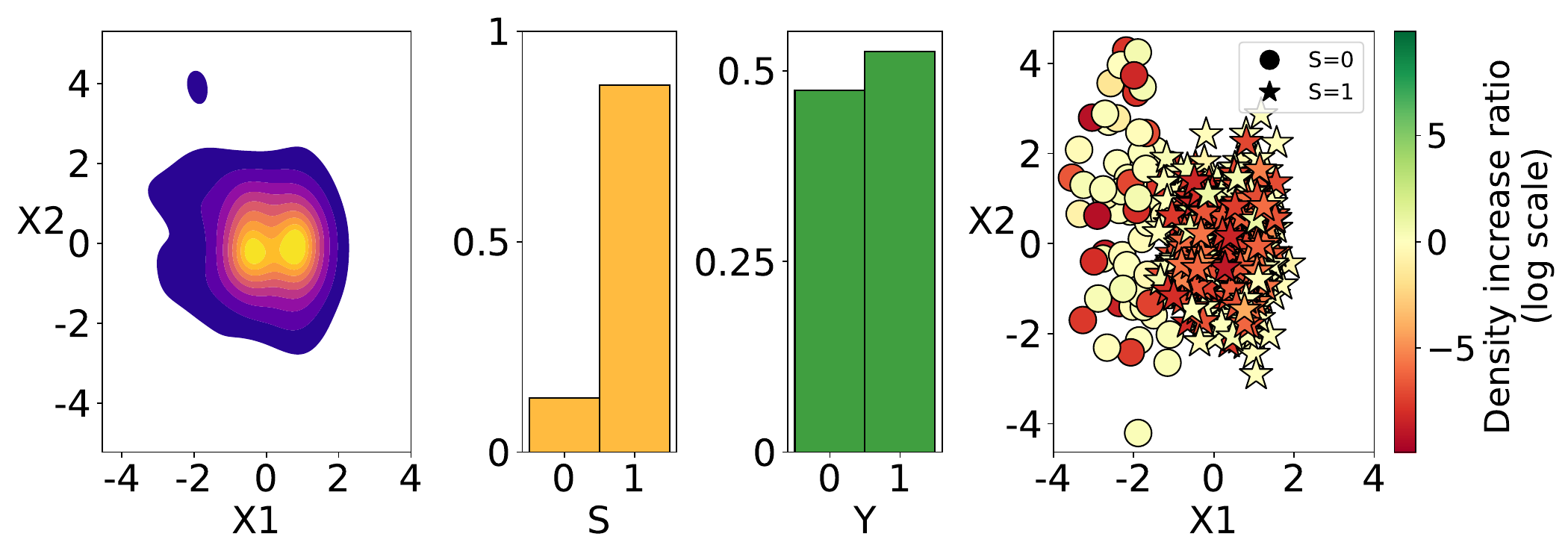}
    \label{fig:subplot2_extreme_all}
    }
  \hfill 
  \subfigure[$\lambda_0 = 1.0$]{
    \includegraphics[width=0.48\textwidth]{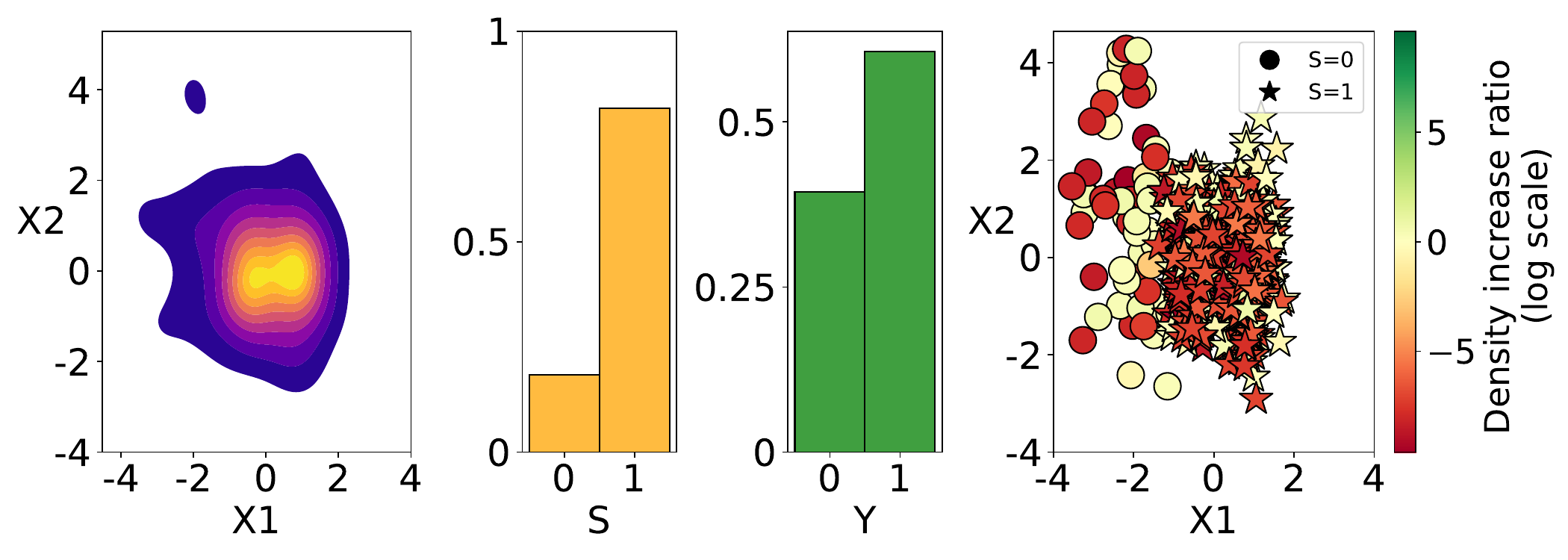}
    \label{fig:subplot3_extreme_all}
    }
  \hfill 
  \subfigure[$\lambda_0 = 5.0$]{
    \includegraphics[width=0.48\textwidth]{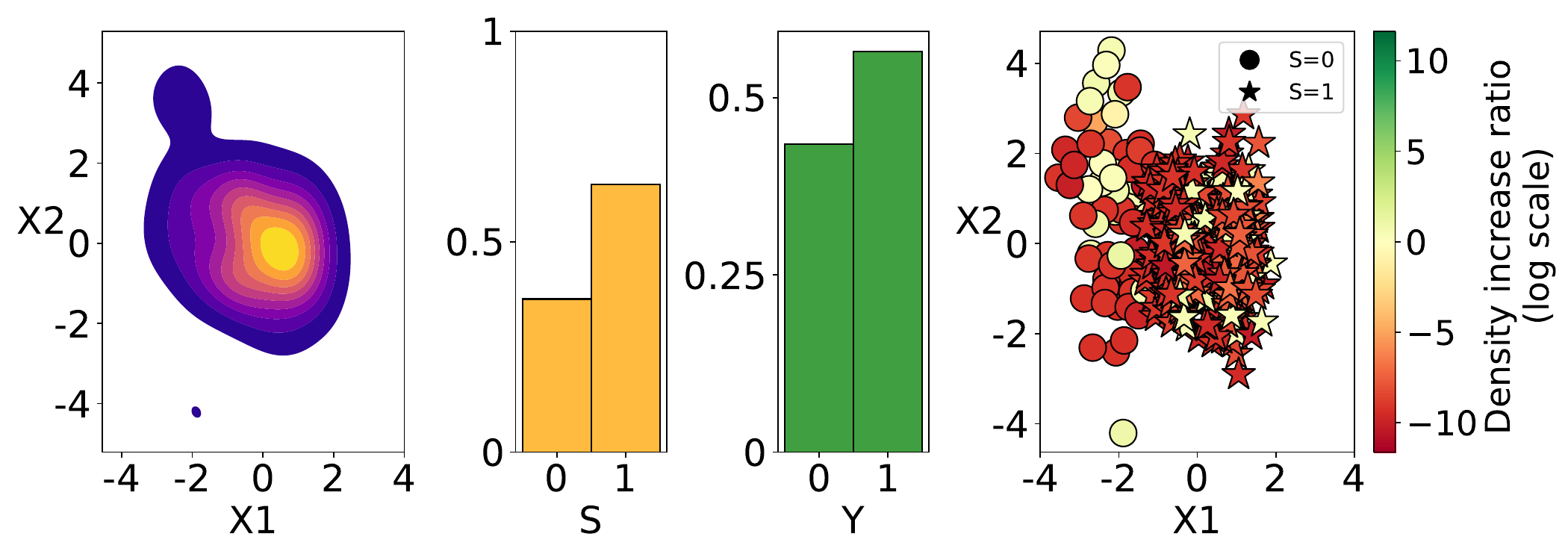}
    \label{fig:subplot4_extreme_all}
    }
    \caption{The statistics of the worst-case distributions w.r.t. \texttt{SPECTRE} for the overall population, for different values of $\lambda_0$. In all the cases we set $\sigma = 5$. Within each subfigure, from left to right the figures show the joint distribution of $X_1$ and $X_2$, the marginal distributions of $S$ and $Y$ and the up-weighing (green) and down-weighing (red) of the instances to create the worst-case distribution.
    }
    \label{fig:all_extremal_lambda_ap}
\end{figure}

\subsubsection{Out-of-Sample Guarantees on Overall Risk}
\label{ap:overall_bounds}

In this section, we show the out-of-sample performance guarantees of \texttt{SPECTRE} in terms of overall risk for the toy problem. We follow the same experimental setup as in Section~\ref{sec:experiments_bounds}, and the results are presented in Figure~\ref{fig:overal_error_bounds}. Notably, computing bounds on the overall risk does not require knowledge of the sensitive attributes of the instances.  

The patterns observed in the bounds of the overall error closely mirror those of the majority group, which is expected given that they constitute the largest portion of the data. Besides, as in the case of the out-of-sample guarantees on group errors, the tightest bounds are achieved with the \( \sigma \) value that provides near-optimal or optimal performance for \texttt{SPECTRE}. Additionally, as \( \lambda_0 \) increases—corresponding to a higher confidence vector—the bounds become looser.

\begin{figure}[!h] 
    \centering
    \subfigure[]{
            \includegraphics[width=0.65\textwidth]{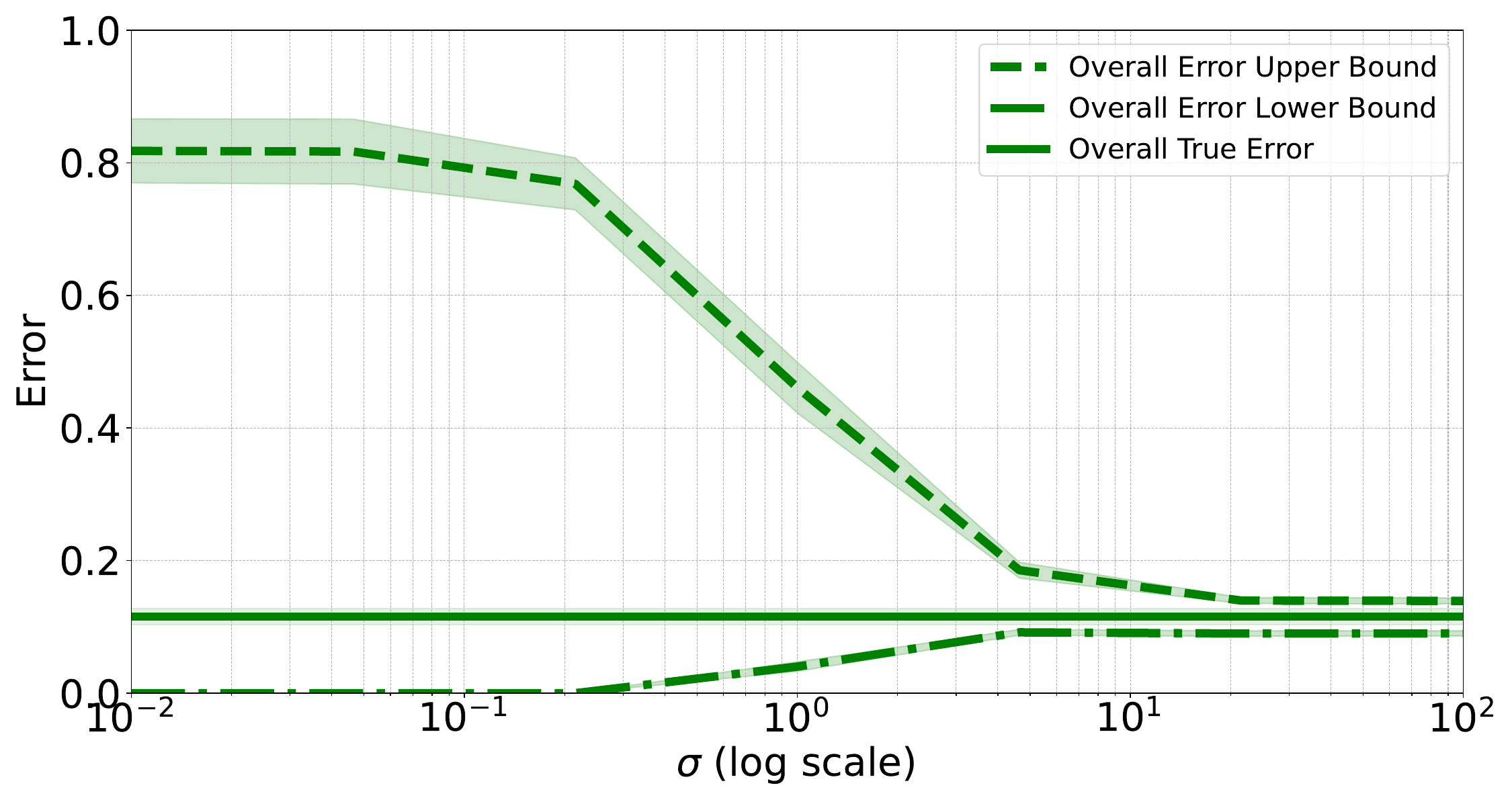}
            \label{fig:overall_bounds_sigma}
        }
        \hfill
        \subfigure[]{
            \includegraphics[width=0.65\textwidth]{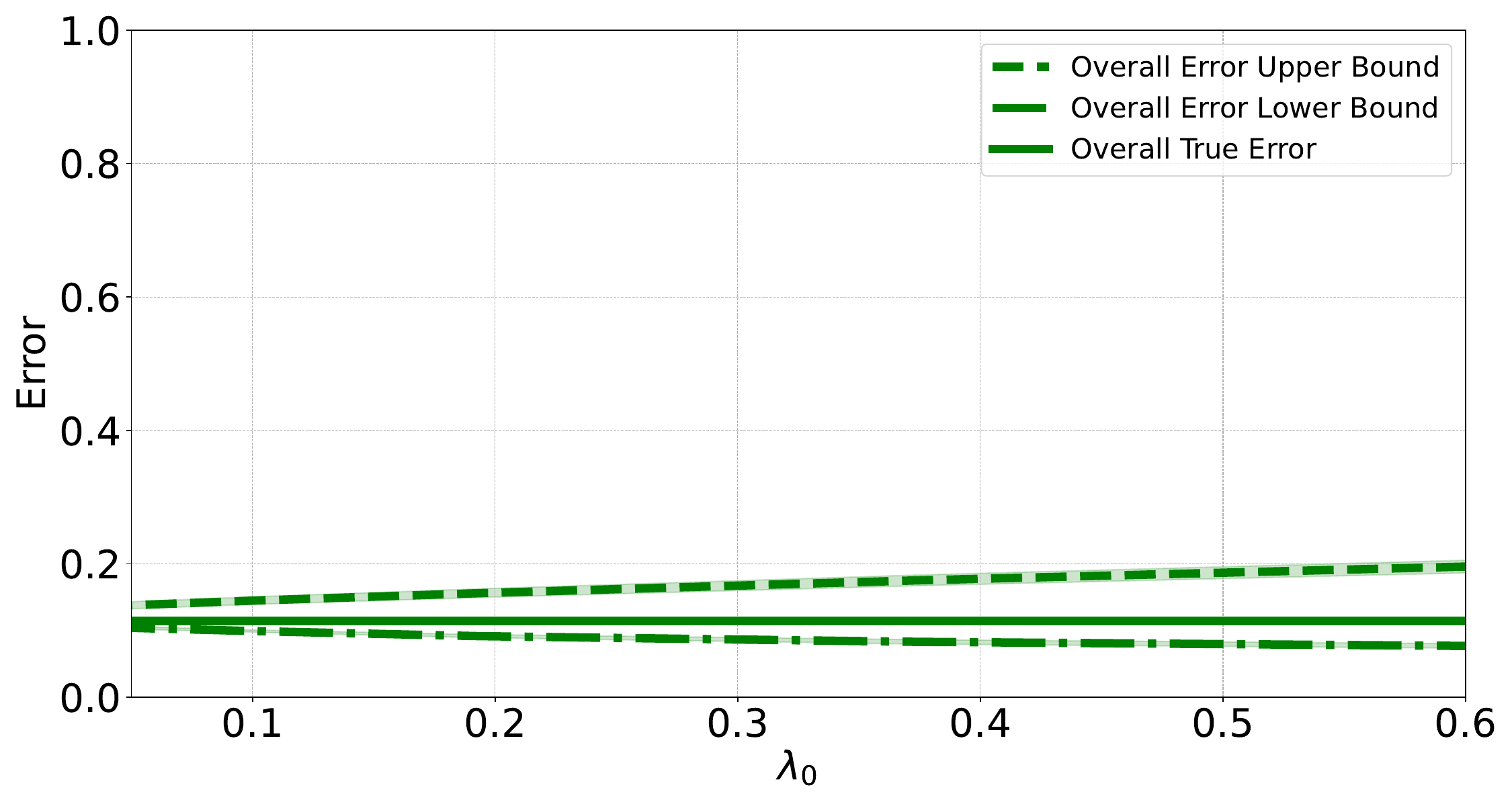}
            \label{fig:overall_bounds_lmbd}
        }
    \caption{Out-of-sample performance guarantees on overall risk for varying (a) $\sigma$ and (b) $\lambda_0$. For (a), we fix $\lambda_0 = 0.1$, and for (b), we fix $\sigma = 10$. 
    }
    \label{fig:overal_error_bounds}
\end{figure}

\subsection{Computational Cost}
\label{sec:spectre_cost}

In this section, we provide insights regarding the computational cost of training \texttt{SPECTRE} (as well as SOTA methods) along different perspectives: for varying numbers of instances (see Figure \ref{fig:time_vs_inst}), features (see Figure \ref{fig:time_vs_feat}) and Fourier components (see Figure \ref{fig:time_vs_comp}). 

Figure \ref{fig:time_vs_inst} shows how the performance of different methods scales with dataset size, using the \textsc{ACSIncome} data from various states (with different number of instances). These results demonstrate that \texttt{SPECTRE} scales well with dataset size. XGBoost and LR are the methods that require the lowest training time, running about 10 times faster than the next fastest approach, \texttt{SPECTRE}. Among fairness-enhancing methods, \texttt{SPECTRE} is the most efficient, with a computational cost similar to ARL but significantly lower than RLM, MMPF, and GDRO, which are approximately three times slower.

\begin{figure}[!h]
\begin{center}
\centerline{\includegraphics[width=0.7\textwidth]{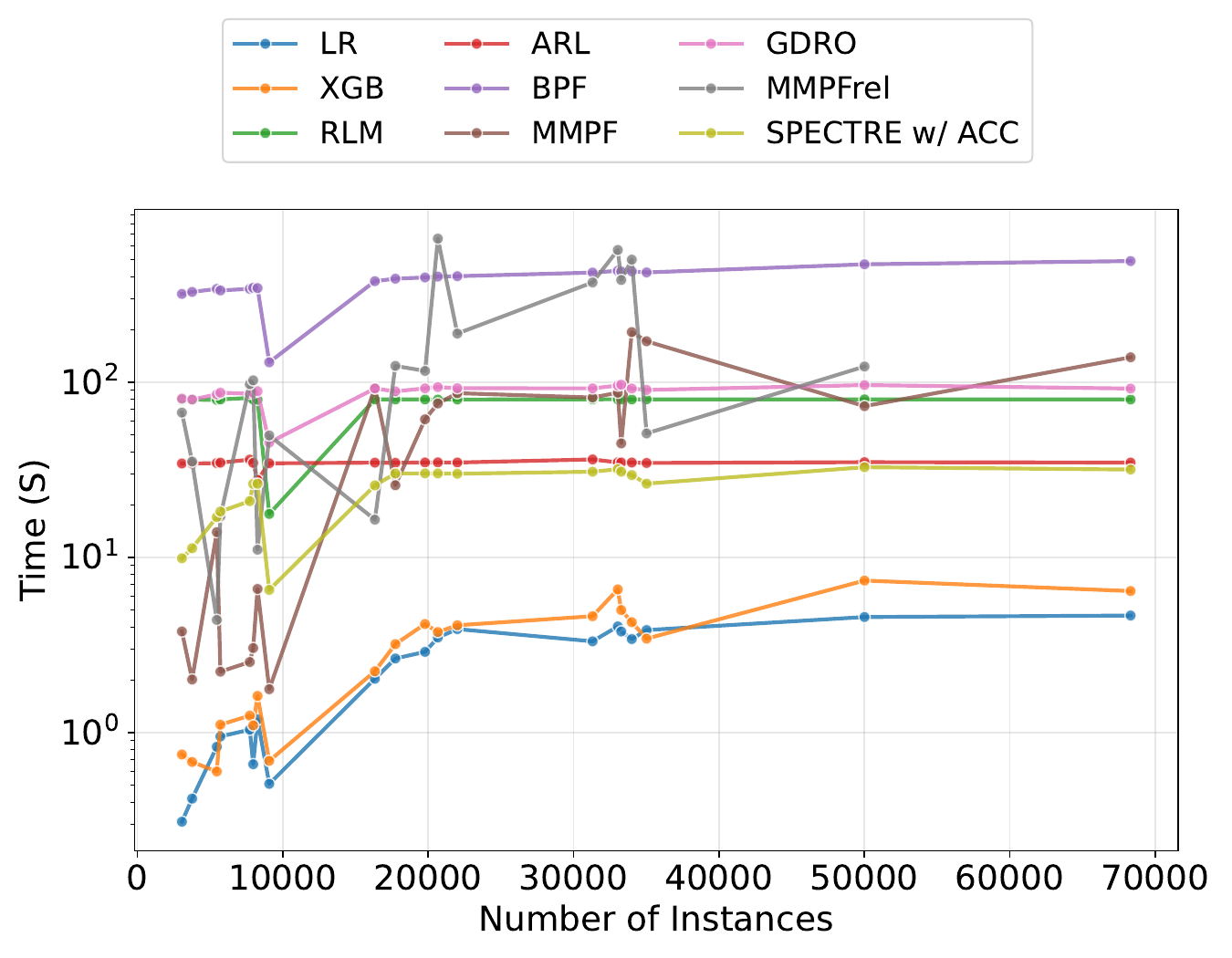}}
\caption{Training computational cost (y-axis) for various methods across different states with varying number of instances (x-axis) in the \textsc{ACSIncome} dataset. LR and XGBoost are the most efficient, followed by \texttt{SPECTRE} and ARL.}
\label{fig:time_vs_inst}
\end{center}
\end{figure}

Moreover, to study the effect of the number of features and Fourier components, we conducted experiments on the toy dataset.
First, we generated variants of the toy data by varying only the number of original features. As Figure \ref{fig:time_vs_feat} shows, increasing the number of features has minimal effect on \texttt{SPECTRE}’s training time. The reason behind this behavior is that the number of features only affects the cost of computing the RFF transformation, which grows linearly with the number of features. Nonetheless, the cost of performing the transformation is negligible compared to the time required to solve the mini-max optimization problem that derives the optimal classifier (Equation~\ref{eq:rrm}). Besides, the time required to solve this optimization problem heavily depends on the size of the feature mapping (i.e., the number of Fourier components in our case). This is clearly shown in Figure \ref{fig:time_vs_comp}, were we fix the original feature set and vary the number of Fourier components used in the RFF. That is, the predominant source of the training cost derives from solving the optimization problem, whose cost rises with the dimensionality of the feature map. In practice, we found that 600 Fourier features provide a favorable balance between efficiency and predictive performance, and we therefore adopt this setting throughout our empirical evaluation of \texttt{SPECTRE}.

\begin{figure}[!h]
\begin{center}
\centerline{\includegraphics[width=0.6\textwidth]{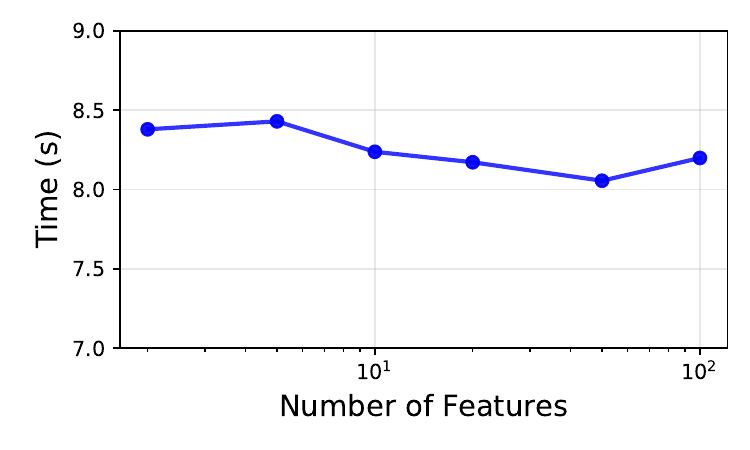}}
\caption{Training time for \texttt{SPECTRE} in the toy dataset as the number of features increases. While the RFF transformation scales linearly with features, its impact is negligible compared to the computational cost of solving the mini-max optimization problem to derive the optimal classification rule (Equation~\ref{eq:rrm}). }
\label{fig:time_vs_feat}
\end{center}
\end{figure}

\begin{figure}[!h]
\begin{center}
\centerline{\includegraphics[width=0.7\textwidth]{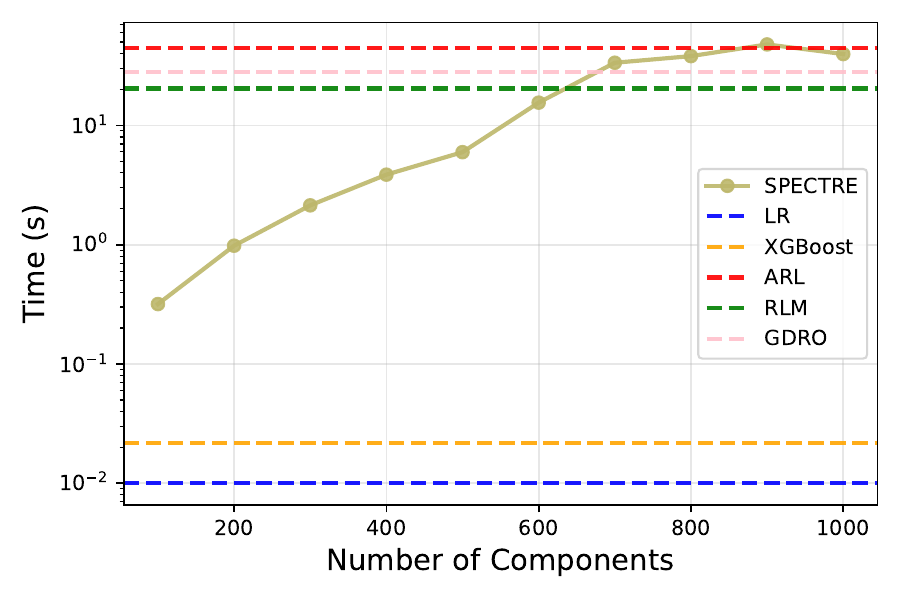}}
\caption{Training time of \texttt{SPECTRE} in the toy dataset for increasing number of Fourier components. As expected, computational cost increases with more Fourier components, but our experiments, as well as those by \citet{mazuelas2023minimax}, demonstrate that 600 components provide a strong balance between efficiency and performance, which is why we adopted this setting for our empirical analysis. }
\label{fig:time_vs_comp}
\end{center}
\end{figure}

\section{Conclusions and Future Work}
\label{sec:conclusions_spectre}

We present \texttt{SPECTRE}, a novel fairness-enhancing intervention that improves worst-group accuracy without demographics. In particular, \texttt{SPECTRE} leverage a novel strategy based on spectrum adjustment to enhance the performance of MRC on unprivileged groups through the uncertainty set. \texttt{SPECTRE} exhibits exceptional robustness performance in terms of worst-group-accuracy, outperforming SOTA minimax-fair approaches that operate without demographic information, and even methods that have access to sensitive information for all instances, with minimal loss in accuracy. 
This highlights how \texttt{SPECTRE} successfully overcomes the pessimism and over-reliance on outliers commonly seen in existing SOTA RRM-based minimax fairness methods. 
Moreover, when a portion of the training data includes sensitive information, it becomes possible to derive out-of-sample generalization guarantees for \texttt{SPECTRE}'s performance across different sensitive groups. 

As with all RRM-based approaches, SPECTRE can improve worst-group accuracy only if the model achieves satisfactory performance on training instances from that specific group; otherwise, such improvement is not guaranteed \cite{slowik2021algorithmic}.
Besides, it is important to note that the out-of-sample guarantees on group errors can only be computed if sensitive information is available for, at least, a subset of the training instances.


In future work, we plan to improve the hyperparameter tuning strategies by incorporating additional characteristics from data. We also aim to extend the applicability of this approach to scenarios involving distribution shifts, which are common in real-world applications. Lastly, we believe that developing methods to assess out-of-sample guarantees without relying on sensitive information is a promising and vital direction for future research.

\section*{Acknowledgments}
This research was funded by the European Union. Views and opinions expressed are however those of the author(s) only and do not necessarily reflect those of the European Union or the European Health and Digital Executive Agency (HaDEA). Neither the European Union nor the granting authority can be held responsible for them. This work is supported by the European Research Council under the European Union's Horizon 2020 research and innovation programme Grant Agreement no. 851538 - BayesianGDPR, Horizon Europe research and innovation programme Grant Agreement no. 101120763 - TANGO. This work is also supported by
the Basque Government under grant IT2109-26 and through the BERC 2022-2025 program; by the Spanish Ministry of Science and Innovation under the grant PID2022-137442NB-I00, and through
BCAM Severo Ochoa accreditation CEX2021-001142-S / MICIN / AEI / 10.13039/501100011033.

\bibliography{ref}
\bibliographystyle{plainnat}

\appendix

\counterwithin{figure}{section}
\counterwithin{table}{section}
\counterwithin{equation}{section}

\section{Uncertainty Set and Worst-Case Realization}
\label{ap:modify_uncertainty_set}

Adjusting the uncertainty set within a RRM framework directly influences the fairness guarantees of the trained classification rule, as it alters the emphasis on unprivileged instances. Consequently, the modification of the uncertainty set has an impact on the worst-group accuracy. In order to see this, let us reformulate the
results of the work by \citet{mazuelas2023minimax} on the worst-case risk for any classification rule  $f \in T(\mathcal{X}, \mathcal{Y})$ and the uncertainty set $\mathcal{U}$ defined as in Equation \ref{eq:uncertainty} as follows: 
    \begin{equation}
    \begin{aligned}
        \sup_{p\in \mathcal{U}} \mathcal{R}(f,p) = \min_{\boldsymbol \mu \in \mathbb{R}^m} \Bigg [ &  \boldsymbol\lambda^T |\boldsymbol \mu| + \\
        & \sup_{\boldsymbol{x} \in \mathcal{X}, y \in \mathcal{Y}} \{  (\Phi(\boldsymbol{x},y) - \mathbb{E}_{p_n} \{ \Phi \} )^T \boldsymbol\mu + \mathcal{L}(f(\boldsymbol{x}),y) \} \Bigg]
        \label{eq:worst_risk}
    \end{aligned}
    \end{equation}
This reformulation of the dual optimization problem of the worst case risk shows that protecting against worst case realization is equivalent to optimizing the tail performance of the model: the objective finds the point in which the classifier achieves high loss $ \mathcal{L}(f(\boldsymbol{x}),y)$ and is far from the expected value of the feature mapping $(\Phi(\boldsymbol{x},y) - \mathbb{E}_{p_n} \{ \Phi \} )$ in the direction pointed out by the optimal value $\boldsymbol{\mu}^*$ (i.e., the solution of the optimization problem \ref{eq:worst_risk}). Usually, the unprivileged instances that differ from the general population and in which the classifier incurs a high loss are associated with minority groups. 
Furthermore, this expression enables to directly see the effect of the confidence vector of the uncertainty set 
through the term $\boldsymbol\lambda^T |\boldsymbol \mu|$ within the worst-case loss. As $\boldsymbol{\lambda}$ decreases, the uncertainty set shrinks, leading to a less conservative worst-case risk (i.e., a smaller worst-case loss). 

\section{Summary of RRM-Based Approaches for Minimax Fairness}

In this section, we summarize the main characteristics of the RRM-based approaches proposed to enhance minimax fairness, including \texttt{SPECTRE} (see Table \ref{tab:summary_rrm}). Most existing methods construct the uncertainty set by reweighting the empirical distribution and do not impose constraints on the degree of divergence from it. In contrast, \texttt{SPECTRE} extends beyond simple reweighting: it models the expectations of distributions in the frequency domain and explicitly bounds how much the distributions within the uncertainty set can deviate from the empirical distribution (in that domain). Furthermore, while the other approaches rely on surrogate loss functions during training, \texttt{SPECTRE} is trained directly using the 0–1 loss.

\begin{table}[!h]
\caption{ Key characteristics of existing RRM-based minimax fairness approaches. The table reports each method’s name, whether it is blind to sensitive information, the construction of its uncertainty set, whether it constraints the maximum deviation from the empirical distribution, and the loss function (noting when a surrogate loss is used). }
\label{tab:summary_rrm}
\begin{center}
\begin{small}
\begin{sc}
\scalebox{0.95}{\begin{tabular}{lcccc}
\toprule
Method & Blind? & Uncertainty set & Constraint? & Loss function  \\
\midrule
MMPF \cite{martinez2020minimax} & \cross & Reweighing & \cross & Cross-entropy loss (surrogate)  \\
& &  & & and Brier Score (surrogate) \\
GDRO \cite{sagawa2020distributionally} & \cross & Reweighing & \cross & Cross-entropy loss (surrogate) \\
RLM \cite{hashimoto2018fairness} & \tick & Reweighing & \tick & Cross-entropy loss (surrogate) \\
ARL \cite{lahoti2020fairness} & \tick & Reweighing & \cross & Cross-entropy loss (surrogate) \\
BPF \cite{martinez2021blind} & \tick & Reweighing & \cross & Cross-entropy loss (surrogate) \\
\texttt{SPECTRE} (ours) & \tick & Spectrum expectation matching & \tick & 0-1 loss \\
\bottomrule
\end{tabular}}
\end{sc}
\end{small}
\end{center}
\end{table}

\section{Alternative General Feature Mappings}
\label{ap:other_featmap}

\begin{figure}[h!]
\begin{center}
\centerline{\includegraphics[width=0.55\textwidth]{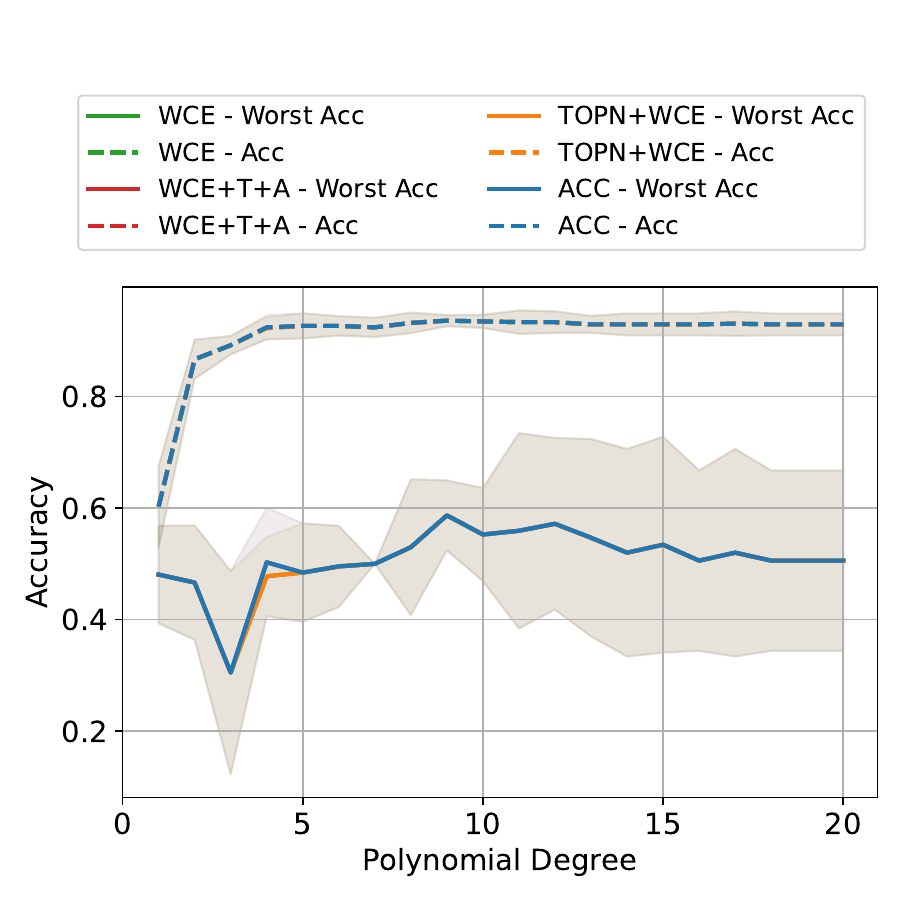}}
\caption{The worst-group accuracy and the overall accuracy of the MRC with a polynomial feature mapping of varying degrees on the toy dataset, employing different strategies to estimate the value of the hyperparameter $\lambda_0$.}
\label{fig:poly_vs_degree}
\end{center}
\end{figure}

In this section, we investigate an alternative to the Fourier feature mapping: the polynomial feature mapping. Figure \ref{fig:poly_vs_degree} presents the performance of applying polynomial feature mappings of varying degrees in combination with MRC on the toy dataset. As the degree of the polynomial increases, overall accuracy initially improves and then plateaus around degree 5. In contrast, the worst-group accuracy does not exhibit a clear trend with respect to the polynomial degree. 

Figure~\ref{fig:other_feat_map} further compares the performance of polynomial and Fourier feature mappings when combined with MRC. We focus on the polynomial mapping of degree 9, which achieves the highest worst-group accuracy among the polynomial variants (see Figure \ref{fig:poly_vs_degree}) and also yields optimal overall accuracy. However, despite its strong overall performance, the polynomial mapping still performs significantly worse than the Fourier mapping in terms of worst-group accuracy. Since our primary objective is to enhance fairness guarantees in terms of worst-group accuracy, the Fourier feature mapping is better suited for this purpose and will be the one employed in conjunction with MRC.

\begin{figure}[h!] 
    \centering
    \subfigure[]{
            \includegraphics[width=0.65\textwidth]{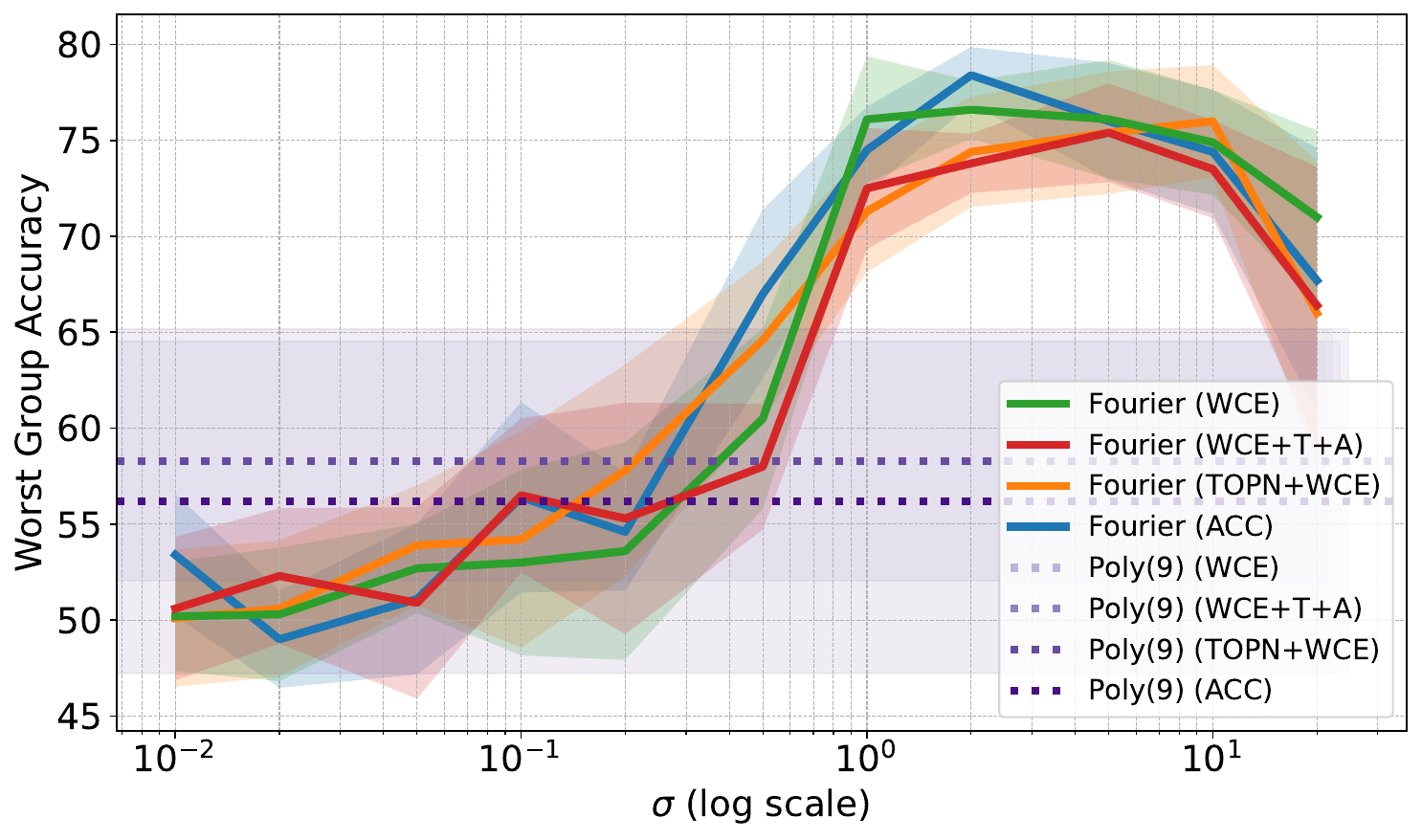}
            \label{fig:other_feat_map_worstacc}
        }
        \hfill
        \subfigure[]{
            \includegraphics[width=0.65\textwidth]{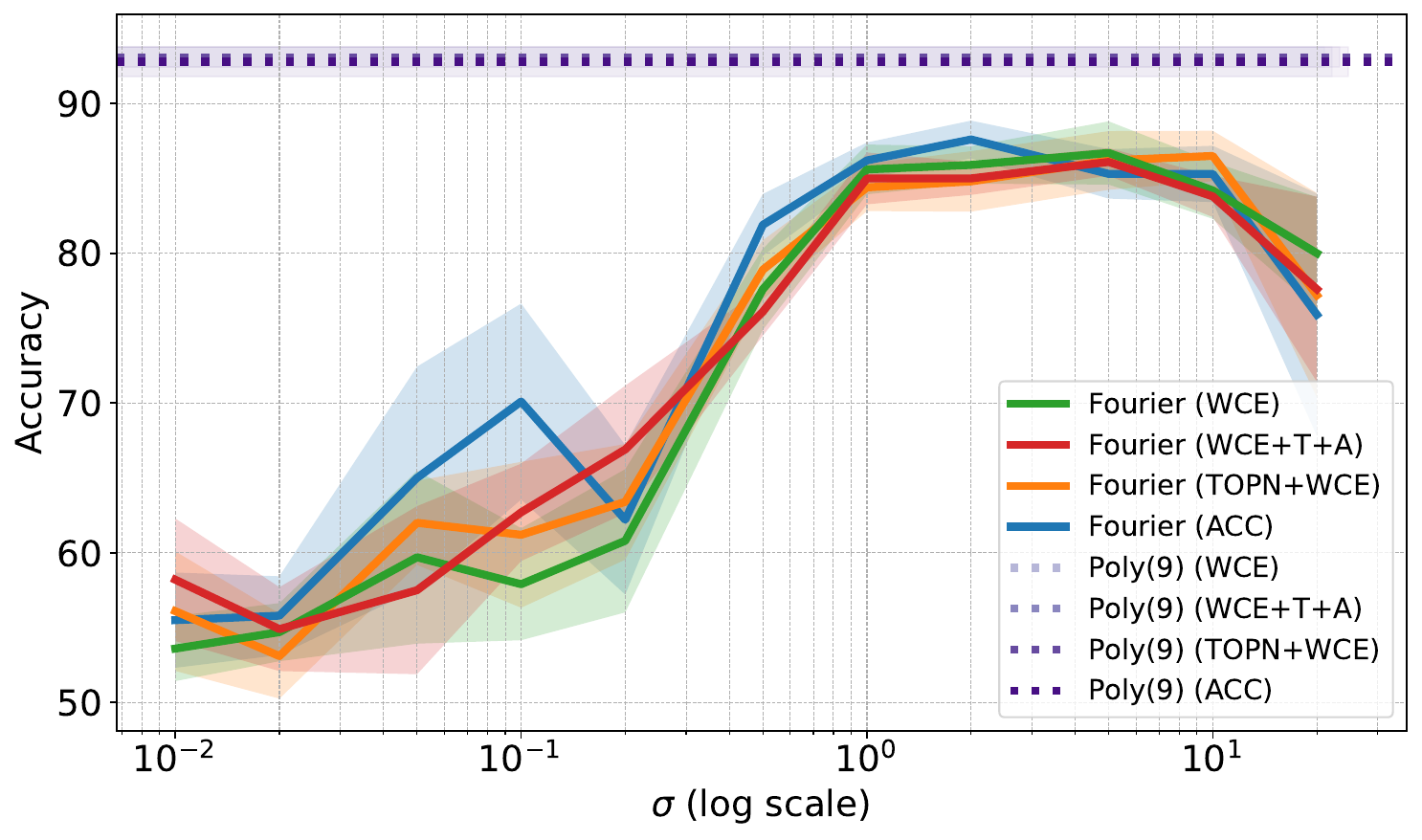}
            \label{fig:other_feat_map_acc}
        }
    \caption{ The (a) worst-group accuracy and (b) overall accuracy on the toy datasets for the Fourier feature mapping and the polynomial feature mapping with degree 9 with various strategies for hyperparameter tuning. The polynomial feature mapping offers better overall accuracy but suffers a significant drop in worst-group accuracy.}
    \label{fig:other_feat_map}
\end{figure}

\clearpage

\section{Additional Results on the Performance of \texttt{SPECTRE}}
\label{ap:additional_results}

\subsection{State-Level Results on ACS Data}
\label{ap:detailed_table}

This section provides additional insights into the performance of \texttt{SPECTRE} and SOTA approaches across various classification tasks using ACS data, as well as different characterizations of sensitive groups. Specifically, we summarize the results for different states in tables. Besides, since it has been demonstrated that improving a classifier's performance for subgroups based on one sensitive attribute can negatively impact the balance of subgroup accuracy for another sensitive attribute, in addition to evaluating sensitive groups based on \textit{race}, gender-based sensitive groups are also considered as a reference to assess this potential impact. In particular, when considering one-dimensional sensitive attributes, we report the average ($\pm$ standard deviation) for overall accuracy (AV. ACC),  the worst-group accuracy (W ACC (R)) maximum accuracy disparity (MAX $\Delta$ ACC (R)) across groups based on race and the worst-group accuracy (W. ACC (G)) maximum accuracy disparity (MAX $\Delta$ ACC (G)) across groups based on gender. 
On the other hand, when considering intersectional groups based on gender and race, we report the average ($\pm$ standard deviation) for overall accuracy (AV. ACC),  the worst-group accuracy (W ACC) and maximum accuracy disparity (MAX $\Delta$ ACC) across these groups.
The experimental setup follows the same methodology described in Section \ref{sec:experimental_setting}.

The overall accuracy and worst-group accuracy for groups based on gender the performance of \texttt{SPECTRE} decreases slightly relative to the most accurate method (XGBoost), with a maximum accuracy drop of 2–4\% (column 3) and a worst-group accuracy reduction of 1–6\% (columns 6–7). However, the fairness improvements observed for racial groups are substantially larger, reaching gains of up to 47\% (Table \ref{tab:results5}), 25\% (Table \ref{tab:results7}), 24\% (Table \ref{tab:results9}), and 20\% (Table \ref{tab:results2}), among others. Given these inherent trade-offs, we believe that a modest accuracy reduction of 2–4\% (or up to 6\% for worst-group performance among gender groups) is a reasonable cost for achieving fairness improvements of 20–40\% in the most vulnerable racial groups. Notably, in intersectional evaluations that jointly consider race and gender (see Tables \ref{tab:results8}–\ref{tab:results10} and Figures \ref{fig:genexp_subplot3}–\ref{fig:genexp_subplot4}), \texttt{SPECTRE} attains the best worst-group accuracy.

In addition to reporting the average values of the evaluation metrics, we also include their standard deviations across multiple repetitions of the experiment. This is particularly important for assessing the variability of performance, especially in smaller or underrepresented groups, where fluctuations tend to be more pronounced. 
Notably, \texttt{SPECTRE} consistently exhibits lower standard deviations compared to standard SOTA methods. This suggests that \texttt{SPECTRE} not only improves fairness but also achieves more stable and reliable performance across different demographic groups.

These results highlight \texttt{SPECTRE}'s ability to improve worst-group accuracy without the need for demographic information, even in scenarios with a large number of groups (up to 18 intersectional groups) and very small group sizes (just a few samples), all while incurring only a minimal reduction in overall accuracy. Furthermore, while it achieves substantial improvements for one sensitive attribute (e.g., race), the negative impact on fairness across other sensitive attributes (e.g., gender) remains minimal. 

\begin{table*}[!h]
\caption{ Average ($\pm$ std) results on the ACSEmployment dataset for state CT, with 8 sensitive groups and the smallest group containing 8 samples. Among the methods that are blind to demographics, the best result for each metric (column) is emphasized in \textbf{bold}, and the second-best result is \underline{underlined}. }
\label{tab:results1}
\begin{center}
\begin{small}
\begin{sc}
\scalebox{0.9}{\begin{tabular}{clccccc}
\toprule
& & av.acc  & w.acc  (r) & max$\Delta$acc(r) & w.acc(g)  & max$\Delta$acc(g)  \\
& Method & $(\uparrow)$ & $(\uparrow)$& $(\downarrow)$ & $(\uparrow)$ &  $(\downarrow)$  \\
\midrule
Base- & LR & $77.3 \pm 0.4$ & $41.2 \pm 25.5$ & $46.6 \pm 18.2$ & $75.0 \pm 0.9$ & $4.9 \pm 1.2$ \\
lines & XGB & $\boldsymbol{79.3 \pm 0.4}$ & $58.0 \pm 12.6$ & $35.4 \pm 7.9$ & $\boldsymbol{76.8 \pm 0.5}$ & $5.2 \pm 0.7$ \\
\midrule
& RLM & $52.3 \pm 5.3$ & $13.4 \pm 18.9$ & $54.5 \pm 14.9$ & $49.9 \pm 4.3$ & $4.8 \pm 2.4$ \\
& ARL & $75.8 \pm 0.8$ & $35.0 \pm 21.5$ & $55.7 \pm 19.4$ & $73.0 \pm 0.9$ & $5.9 \pm 0.8$ \\
& BPF & $70.1 \pm 2.7$ & $39.9 \pm 14.8$ & $54.2 \pm 18.2$ & $66.7 \pm 3.0$ & $7.2 \pm 1.1$ \\
w/o & SURE & $77.2 \pm 0.3$ & $55.8 \pm 7.6$ & $37.4 \pm 7.6$ & $74.7 \pm 0.4$ & $5.1 \pm 0.6$ \\ 
& FairEns & $53.9 \pm 8.8$ & $39.9 \pm 11.2$ & $39.9 \pm 19.0$ & $50.4 \pm 9.1$ & $7.3 \pm 2.4$ \\
Demo- & \texttt{SPECTRE}+WCE & $77.5 \pm 0.3$ & $61.7 \pm 9.6$ & $31.4 \pm 3.2$  & $\underline{75.6 \pm 0.5}$ & $\boldsymbol{3.8 \pm 0.7}$ \\
graphics & \texttt{SPECTRE}+WCE+T+A & $77.4 \pm 0.4$ & $62.4 \pm 10.1$ & $\underline{29.6 \pm 1.3}$ & $75.4 \pm 0.7$ & $\underline{4.2 \pm 0.7}$ \\
& \texttt{SPECTRE}+TOPN+WCE & $\underline{77.7 \pm 0.4}$ & $\underline{62.6 \pm 10.3}$ & $30.6 \pm 3.5$ & $75.6 \pm 0.7$ & $4.4 \pm 0.7$ \\
& \texttt{SPECTRE}+ACC & $77.5 \pm 0.5$ & $\boldsymbol{62.7 \pm 10.4}$ & $\boldsymbol{29.2 \pm 1.3}$ & $75.4 \pm 0.5$ & $4.5 \pm 0.6$ \\
\midrule
S in & MMPF & $72.5 \pm 0.3$ & $45.6 \pm 7.9$ & $40.5 \pm 8.5$ & $69.3 \pm 0.4$  & $6.7 \pm 0.6$ \\
train & GDRO & $74.2 \pm 0.4$ & $54.1 \pm 5.6$ & $37.0 \pm 8.8$ & $71.2 \pm 0.4$ & $6.3 \pm 0.3$ \\
\& val & MMPFrel & $71.4 \pm 2.1$ & $44.1 \pm 15.8$ & $41.5 \pm 14.9$ & $68.2 \pm 2.1$ & $6.7 \pm 0.8$ \\
\end{tabular}}
\end{sc}
\end{small}
\end{center}
\end{table*}

\begin{table*}[!h]
\caption{ Average ($\pm$ std) results on the ACSEmployment dataset for state LA, with 9 sensitive groups and the smallest group containing 1 sample. Among the methods that are blind to demographics, the best result for each metric (column) is emphasized in \textbf{bold}, and the second-best result is \underline{underlined}. }
\label{tab:results2}
\begin{center}
\begin{small}
\begin{sc}
\scalebox{0.9}{\begin{tabular}{clccccc}
\toprule
& & av.acc  & w.acc  (r) & max$\Delta$acc(r) & w.acc(g)  & max$\Delta$acc(g)  \\
& Method & $(\uparrow)$ & $(\uparrow)$& $(\downarrow)$ & $(\uparrow)$ &  $(\downarrow)$  \\
\midrule
Base - & LR & $75.4 \pm 0.3$ & $41.9 \pm 26.1$ & $47.4 \pm 27.8$ & $72.9 \pm 0.2$ & $5.2 \pm 0.4$ \\
lines & XGB & $\boldsymbol{77.5 \pm 0.2}$ & $46.6 \pm 26.6$ & $46.5 \pm 32.9$ & $\boldsymbol{74.6 \pm 0.2}$ & $6.0 \pm 0.2$ \\
\midrule
& RLM & $50.5 \pm 2.9$ & $25.1 \pm 20.7$ & $53.6 \pm 25.1$ & $49.1 \pm 2.9$ & $\boldsymbol{2.9 \pm 2.3}$ \\
& ARL & $75.1 \pm 0.8$ & $31.7 \pm 26.0$ & $54.3 \pm 37.4$ & $72.6 \pm 0.9$ & $5.1 \pm 0.4$ \\
& BPF & $68.5 \pm 0.3$ & $21.4 \pm 21.1$ & $61.1 \pm 33.6$ & $65.7 \pm 0.3$ & $5.8 \pm 0.7$ \\
w/o & SURE & $75.8 \pm 0.2$ & $42.1 \pm 26.3$ & $42.4 \pm 23.2$ & $73.0 \pm 0.3$ & $5.9 \pm 0.3$ \\
Demo- & FairEns & $49.9 \pm 2.2$ & $20.7 \pm 18.1$ & $37.8 \pm 16.7$ & $47.1 \pm 2.7$ & $5.4 \pm 1.9$ \\
graphics & \texttt{SPECTRE}+WCE & $75.4 \pm 0.7$ & $66.4 \pm 3.8$ & $22.0 \pm 9.5$ & $73.5 \pm 0.5$ & $\underline{4.0 \pm 0.9}$ \\
& \texttt{SPECTRE}+WCE+T+A & $\underline{76.0 \pm 0.4}$ & $\underline{67.1 \pm 2.7}$ & $\boldsymbol{20.0 \pm 10.4}$ & $\underline{73.6 \pm 0.5}$ & $4.9 \pm 0.5$ \\
& \texttt{SPECTRE}+TOPN+WCE & $75.5 \pm 0.7$ & $\boldsymbol{67.3 \pm 3.8}$ & $24.7 \pm 8.2$ & $73.5 \pm 0.6$ & $4.3 \pm 1.1$ \\
& \texttt{SPECTRE}+ACC & $76.0 \pm 0.3$ & $66.1 \pm 2.2$ & $\underline{21.9 \pm 10.2}$ & $73.6 \pm 0.4$ & $5.0 \pm 0.5$ \\
\midrule
S in & MMPF & $72.8 \pm 0.5$ & $31.3 \pm 22.6$ & $56.6 \pm 29.0$ & $70.6 \pm 0.5$ & $4.5 \pm 0.6$ \\
train  & GDRO & $73.7 \pm 0.3$ & $21.0 \pm 18.5$ & $63.3 \pm 23.1$ & $71.5 \pm 0.2$ & $4.6 \pm 0.3$ \\
\& val & MMPFrel & $71.4 \pm 1.3$ & $31.9 \pm 23.5$ & $54.3 \pm 31.6$ & $68.8 \pm 1.5$ & $5.6 \pm 1.0$ \\
\end{tabular}}
\end{sc}
\end{small}
\end{center}
\end{table*}

\begin{table*}[ht!]
\caption{ Average ($\pm$ std) results on the ACSEmployment dataset for state MI, with 8 sensitive groups and the smallest group containing 21 samples. Among the methods that are blind to demographics, the best result for each metric (column) is emphasized in \textbf{bold}, and the second-best result is \underline{underlined}. }
\label{tab:results3}
\begin{center}
\begin{small}
\begin{sc}
\scalebox{0.9}{\begin{tabular}{clccccc}
\toprule
& & av.acc  & w.acc  (r) & max$\Delta$acc(r) & w.acc(g)  & max$\Delta$acc(g)  \\
& Method & $(\uparrow)$ & $(\uparrow)$& $(\downarrow)$ & $(\uparrow)$ &  $(\downarrow)$  \\
\midrule
Base- & LR & $77.5 \pm 0.1$ & $61.0 \pm 2.1$ & $17.8 \pm 2.1$ & $75.7 \pm 0.2$ & $3.7 \pm 0.3$ \\
lines & XGB & $\boldsymbol{79.5 \pm 0.2}$ & $65.6 \pm 5.7$ & $18.3 \pm 9.8$ & $\boldsymbol{77.6 \pm 0.3}$ & $3.8 \pm 0.3$ \\
\midrule
& RLM & $50.0 \pm 5.0$ & $38.5 \pm 5.7$ & $26.4 \pm 8.8$ & $48.5 \pm 6.1$ & $3.0 \pm 2.5$ \\
& ARL & $76.0 \pm 1.0$ & $51.7 \pm 8.1$ & $25.8 \pm 7.5$ & $74.4 \pm 1.1$ & $3.1 \pm 0.5$ \\
 & BPF & $70.2 \pm 0.3$ & $34.7 \pm 12.4$ & $39.6 \pm 11.3$ & $67.7 \pm 0.4$ & $5.1 \pm 0.2$ \\
 w/o & SURE & $\underline{78.6 \pm 0.2}$ & $41.5 \pm 17.9$ & $37.7 \pm 17.8$ & $76.5 \pm 0.2$ & $4.1 \pm 0.3$ \\
 Demo- & FairEns & $46.1 \pm 1.9$ & $36.0 \pm 8.4$ & $32.8 \pm 11.5$ & $43.3 \pm 2.2$ & $5.5 \pm 1.2$ \\
 graphics & \texttt{SPECTRE}+WCE & $77.7 \pm 0.4$ & $\boldsymbol{66.2 \pm 5.7}$ & $22.4 \pm 5.2$ & $76.4 \pm 0.4$ & $\boldsymbol{2.5 \pm 0.2}$ \\
 & \texttt{SPECTRE}+WCE+T+A & $77.8 \pm 0.3$ & $65.1 \pm 5.0$ & $\underline{14.4 \pm 5.0}$ & $\underline{76.5 \pm 0.3}$ & $\underline{2.6 \pm 0.4}$ \\
& \texttt{SPECTRE}+TOPN+WCE & $77.7 \pm 0.4$ & $\underline{66.1 \pm 5.6}$ & $\boldsymbol{14.3 \pm 4.9}$ & $76.4 \pm 0.2$ & $2.7 \pm 0.4$ \\
& \texttt{SPECTRE}+ACC & $77.8 \pm 0.3$ & $64.5 \pm 5.5$ & $16.4 \pm 5.4$ & $76.4 \pm 0.3$ & $3.0 \pm 0.1$ \\
\midrule
S in & MMPF & $73.5 \pm 0.6$ & $34.0 \pm 16.8$ & $40.8 \pm 17.2$ & $71.7 \pm 0.6$ & $3.6 \pm 0.6$ \\
train & GDRO & $74.0 \pm 0.5$ & $49.1 \pm 16.0$ & $29.0 \pm 15.5$ & $72.6 \pm 0.5$ & $3.0 \pm 0.5$ \\
\& val & MMPFrel & $70.1 \pm 1.4$ & $52.1 \pm 16.3$ & $23.9 \pm 15.7$ & $67.5 \pm 1.5$ & $5.2 \pm 0.4$ \\
\end{tabular}}
\end{sc}
\end{small}
\end{center}
\end{table*}

\begin{table*}[ht!]
\caption{ Average ($\pm$ std) results on the ACSEmployment dataset for state OR, with 9 sensitive groups and the smallest group containing 21 samples. Among the methods that are blind to demographics, the best result for each metric (column) is emphasized in \textbf{bold}, and the second-best result is \underline{underlined}. }
\label{tab:results4}
\begin{center}
\begin{small}
\begin{sc}
\scalebox{0.9}{\begin{tabular}{clccccc}
\toprule
& & av.acc  & w.acc  (r) & max$\Delta$acc(r) & w.acc(g)  & max$\Delta$acc(g)  \\
& Method & $(\uparrow)$ & $(\uparrow)$& $(\downarrow)$ & $(\uparrow)$ &  $(\downarrow)$  \\
\midrule
Base- & LR & $77.3 \pm 0.2$ & $67.8 \pm 3.4$ & $17.6 \pm 6.0$ & $74.7 \pm 0.3$ & $5.3 \pm 0.7$ \\
lines & XGB & $\boldsymbol{78.9 \pm 0.3}$ & $63.2 \pm 6.5$ & $20.1 \pm 5.8$ & $\boldsymbol{76.3 \pm 0.4}$ & $5.4 \pm 0.6$ \\
\midrule
& RLM & $47.2 \pm 3.3$ & $36.6 \pm 11.0$ & $32.8 \pm 9.6$ & $45.8 \pm 3.0$ & $\boldsymbol{2.8 \pm 2.0}$ \\
& ARL & $76.5 \pm 0.4$ & $60.4 \pm 6.1$ & $25.1 \pm 7.8$ & $74.5 \pm 0.3$ & $4.1 \pm 0.8$ \\
& BPF & $69.3 \pm 0.3$ & $50.8 \pm 7.1$ & $33.3 \pm 7.4$ & $67.6 \pm 0.1$ & $\underline{3.5 \pm 0.7}$ \\
w/o & SURE & $77.5 \pm 0.3$ & $62.0 \pm 5.5$ & $21.3 \pm 4.1$ & $75.1 \pm 0.3$ & $4.9 \pm 0.5$ \\
Demo- & FairEns & $48.8 \pm 6.4$ & $26.2 \pm 7.3$ & $37.2 \pm 7.3$ & $46.2 \pm 6.1$ & $5.3 \pm 2.4$ \\
graphics & \texttt{SPECTRE}+WCE & $\underline{77.7 \pm 0.2}$ & $\underline{69.4 \pm 3.8}$ & $15.3 \pm 8.0$ & $\underline{75.9 \pm 0.2}$ & $3.8 \pm 0.6$ \\
 & \texttt{SPECTRE}+WCE+T+A & $77.6 \pm 0.3$ & $69.3 \pm 3.7$ & $14.8 \pm 4.8$ & $75.7 \pm 0.4$ & $3.8 \pm 0.3$ \\
& \texttt{SPECTRE}+TOPN+WCE & $77.6 \pm 0.3$ & $69.3 \pm 3.8$ & $\underline{13.9 \pm 4.5}$ & $75.7 \pm 0.6$ & $4.0 \pm 0.7$ \\
& \texttt{SPECTRE}+ACC & $77.5 \pm 0.3$ & $\boldsymbol{69.9 \pm 2.7}$ & $\boldsymbol{12.3 \pm 4.5}$ & $75.6 \pm 0.4$ & $3.8 \pm 0.5$ \\
\midrule
S in & MMPF & $74.1 \pm 0.3$ & $53.4 \pm 11.5$ & $28.9 \pm 13.0$ & $71.8 \pm 0.5$ & $4.8 \pm 0.6$ \\
train & GDRO & $74.3 \pm 0.4$ & $58.0 \pm 13.3$ & $31.6 \pm 10.2$ & $72.3 \pm 0.8$ & $4.0 \pm 1.0$ \\
\& val & MMPFrel & $70.3 \pm 1.9$ & $62.4 \pm 4.2$ & $21.9 \pm 6.5$ & $67.3 \pm 2.2$ & $6.0 \pm 1.1$ \\
\end{tabular}}
\end{sc}
\end{small}
\end{center}
\end{table*}


\begin{table*}[ht!]
\caption{ Average ($\pm$ std) results on the ACSIncome dataset for state MT, with 8 sensitive groups and the smallest group containing 3 samples. Among the methods that are blind to demographics, the best result for each metric (column) is emphasized in \textbf{bold}, and the second-best result is \underline{underlined}. }
\label{tab:results5}
\begin{center}
\begin{small}
\begin{sc}
\scalebox{0.9}{\begin{tabular}{clccccc}
\toprule
& & av.acc  & w.acc  (r) & max$\Delta$acc(r) & w.acc(g)  & max$\Delta$acc(g)  \\
& Method & $(\uparrow)$ & $(\uparrow)$& $(\downarrow)$ & $(\uparrow)$ &  $(\downarrow)$  \\
\midrule
Base- & LR & $\underline{76.5 \pm 0.9}$ & $51.0 \pm 12.1$ & $49.0 \pm 12.1$ & $72.1 \pm 0.7$ & $9.6 \pm 1.1$ \\
lines & XGB & $\boldsymbol{77.8 \pm 0.7}$ & $15.2 \pm 30.4$ & $84.8 \pm 30.4$ & $\underline{73.4 \pm 1.3}$ & $9.5 \pm 1.6$ \\
\midrule
& RLM & $54.2 \pm 3.5$ & $7.9 \pm 10.2$ & $87.1 \pm 11.2$ & $52.6 \pm 2.8$ & $\boldsymbol{3.6 \pm 2.4}$ \\
& ARL & $76.2 \pm 0.8$ & $60.6 \pm 10.2$ & $39.4 \pm 10.2$ & $71.4 \pm 0.4$ & $10.2 \pm 1.1$ \\
& BPF & $63.1 \pm 4.1$ & $16.7 \pm 21.1$ & $79.7 \pm 19.4$ & $59.6 \pm 2.8$ & $7.4 \pm 3.6$ \\
w/o & SURE & $76.4 \pm 1.4$ & $48.4 \pm 26.1$ & $51.6 \pm 26.1$ & $\boldsymbol{73.6 \pm 1.4}$ & $\underline{5.7 \pm 1.8}$ \\
Demo- & FAIREns & $51.0 \pm 16.4$ & $32.0 \pm 28.6$ & $53.0 \pm 16.0$ & $46.3 \pm 16.9$ & $9.3 \pm 5.3$ \\
graphics & \texttt{SPECTRE}+WCE & $73.0 \pm 0.8$ & $\boldsymbol{62.5 \pm 10.5}$ & $\boldsymbol{34.3 \pm 13.3}$ & $67.7 \pm 1.1$ & $11.4 \pm 0.9$ \\
graphics & \texttt{SPECTRE}+WCE+T+A & $72.8 \pm 0.7$ & $\underline{62.3 \pm 10.3}$ & $37.7 \pm 10.3$ & $67.2 \pm 1.0$ & $11.9 \pm 0.9$ \\
& \texttt{SPECTRE}+TOPN+WCE & $73.0 \pm 0.7$ & $\boldsymbol{62.5 \pm 10.5}$ & $34.6 \pm 12.9$ & $67.7 \pm 1.0$ & $11.4 \pm 0.8$ \\
& \texttt{SPECTRE}+ACC & $72.9 \pm 0.7$ & $\boldsymbol{62.5 \pm 10.4}$ & $\underline{34.4 \pm 13.2}$ & $67.7 \pm 1.0$ & $11.3 \pm 0.8$ \\
\midrule
S in & MMPF & $73.8 \pm 0.6$ & $44.5 \pm 16.7$ & $55.5 \pm 16.7$ & $68.6 \pm 1.2$ & $11.2 \pm 1.9$ \\
train & GDRO & $75.9 \pm 0.2$ & $53.4 \pm 28.2$ & $46.6 \pm 28.2$ & $71.2 \pm 0.6$ & $10.1 \pm 1.5$ \\
\& val & MMPFrel & $72.7 \pm 1.4$ & $57.1 \pm 10.9$ & $42.9 \pm 10.9$ & $66.7 \pm 2.0$ & $12.9 \pm 1.9$ \\
\end{tabular}}
\end{sc}
\end{small}
\end{center}
\end{table*}

\begin{table*}[ht!]
\caption{ Average ($\pm$ std) results on the ACSIncome dataset for state RI, with 8 sensitive groups and the smallest group containing 1 sample. Among the methods that are blind to demographics, the best result for each metric (column) is emphasized in \textbf{bold}, and the second-best result is \underline{underlined}. }
\label{tab:results6}
\begin{center}
\begin{small}
\begin{sc}
\scalebox{0.9}{\begin{tabular}{clccccc}
\toprule
& & av.acc  & w.acc  (r) & max$\Delta$acc(r) & w.acc(g)  & max$\Delta$acc(g)  \\
& Method & $(\uparrow)$ & $(\uparrow)$& $(\downarrow)$ & $(\uparrow)$ &  $(\downarrow)$  \\
\midrule
Base- & LR & $\underline{78.7 \pm 0.6}$ & $61.9 \pm 18.7$ & $38.1 \pm 18.7$ & $\underline{77.3 \pm 0.9}$ & $2.8 \pm 1.9$ \\
lines & XGB & $\boldsymbol{79.2 \pm 0.3}$ & $59.1 \pm 29.7$ & $40.9 \pm 29.7$ & $\boldsymbol{78.3 \pm 0.1}$ & $1.7 \pm 0.7$ \\
\midrule
& RLM & $51.5 \pm 2.0$ & $6.7 \pm 13.3$ & $79.2 \pm 17.2$ & $49.5 \pm 2.8$ & $4.0 \pm 2.9$ \\
& ARL & $78.5 \pm 0.4$ & $66.3 \pm 8.8$ & $33.7 \pm 8.8$ & $76.9 \pm 0.7$ & $3.1 \pm 0.9$ \\
 & BPF & $71.4 \pm 0.7$ & $50.7 \pm 14.4$ & $45.6 \pm 16.0$ & $70.7 \pm 0.8$ & $\underline{1.4 \pm 0.6}$ \\
 w/o & SURE & $76.4 \pm 0.5$ & $62.0 \pm 18.7$ & $38.0 \pm 18.7$ & $75.3 \pm 1.1$ & $2.1 \pm 1.4$ \\
 Demo- & FairEns & $53.7 \pm 4.3$ & $43.8 \pm 11.3$ & $50.5 \pm 11.9$ & $48.6 \pm 2.6$ & $10.0 \pm 4.6$ \\
graphics & \texttt{SPECTRE}+WCE & $75.5 \pm 0.6$ & $\underline{72.2 \pm 3.1}$ & $\underline{27.8 \pm 3.1}$ & $74.5 \pm 0.6$ & $1.9 \pm 1.1$ \\
 & \texttt{SPECTRE}+WCE+T+A & $76.0 \pm 1.4$ & $70.8 \pm 3.2$ & $29.2 \pm 3.2$ & $75.1 \pm 1.3$ & $1.8 \pm 0.9$ \\
& \texttt{SPECTRE}+TOPN+WCE & $75.7 \pm 2.0$ & $71.6 \pm 4.1$ & $28.4 \pm 4.1$ & $75.3 \pm 2.1$ & $\boldsymbol{0.8 \pm 0.9}$ \\
& \texttt{SPECTRE}+ACC & $76.8 \pm 0.5$ & $\boldsymbol{73.3 \pm 4.0}$ & $\boldsymbol{26.7 \pm 4.0}$ & $75.9 \pm 0.7$ & $2.0 \pm 1.5$ \\
\midrule
S in & MMPF & $76.0 \pm 0.9$ & $34.8 \pm 20.6$ & $60.8 \pm 16.5$ & $75.3 \pm 0.8$ & $1.4 \pm 0.9$ \\
train & GDRO & $78.5 \pm 0.6$ & $59.4 \pm 19.7$ & $40.6 \pm 19.7$ & $77.0 \pm 0.5$ & $3.0 \pm 1.2$ \\
\& val & MMPFrel & $76.4 \pm 0.7$ & $59.4 \pm 16.8$ & $39.3 \pm 15.1$ & $75.7 \pm 0.7$ & $1.4 \pm 1.5$ \\
\end{tabular}}
\end{sc}
\end{small}
\end{center}
\end{table*}

\begin{table*}[ht!]
\caption{ Average ($\pm$ std) results on the ACSIncome dataset for state VT, with 7 sensitive groups and the smallest group containing 2 samples. Among the methods that are blind to demographics, the best result for each metric (column) is emphasized in \textbf{bold}, and the second-best result is \underline{underlined}. }
\label{tab:results7}
\begin{center}
\begin{small}
\begin{sc}
\scalebox{0.9}{\begin{tabular}{clccccc}
\toprule
& & av.acc  & w.acc  (r) & max$\Delta$acc(r) & w.acc(g)  & max$\Delta$acc(g)  \\
& Method & $(\uparrow)$ & $(\uparrow)$& $(\downarrow)$ & $(\uparrow)$ &  $(\downarrow)$  \\
\midrule
Base- & LR & $\underline{76.4 \pm 0.9}$ & $38.4 \pm 31.9$ & $61.6 \pm 31.9$ & $\underline{74.6 \pm 2.2}$ & $\boldsymbol{3.6 \pm 3.1}$ \\
lines & XGB & $\boldsymbol{77.7 \pm 1.0}$ & $41.9 \pm 34.3$ & $58.1 \pm 34.3$ & $\boldsymbol{76.0 \pm 1.1}$ & $\underline{3.6 \pm 1.6}$ \\
\midrule
& RLM & $55.7 \pm 2.5$ & $16.6 \pm 20.3$ & $73.9 \pm 15.9$ & $52.8 \pm 3.2$ & $5.9 \pm 2.1$ \\
& ARL & $75.6 \pm 0.7$ & $32.4 \pm 29.1$ & $66.3 \pm 29.3$ & $73.5 \pm 2.1$ & $4.2 \pm 3.1$ \\
& BPF & $63.0 \pm 2.2$ & $46.1 \pm 10.0$ & $45.7 \pm 13.2$ & $60.3 \pm 2.2$ & $5.5 \pm 2.8$ \\
w/o & SURE & $66.2 \pm 1.0$ & $52.4 \pm 10.3$ & $43.1 \pm 14.3$ & $61.4 \pm 2.7$ & $9.7 \pm 3.6$ \\
Demo- & FairEns & $61.4 \pm 4.0$ & $27.5 \pm 24.3$ & $54.5 \pm 23.4$ & $57.7 \pm 3.8$ & $7.6 \pm 3.6$ \\
graphics & \texttt{SPECTRE}+WCE & $74.8 \pm 1.1$ & $64.9 \pm 9.1$ & $35.1 \pm 9.1$ & $72.8 \pm 2.3$ & $4.1 \pm 2.9$ \\
 & \texttt{SPECTRE}+WCE+T+A & $74.9 \pm 1.1$ & $64.4 \pm 8.6$ & $\underline{34.3 \pm 9.3}$ & $72.4 \pm 2.1$ & $5.0 \pm 2.6$ \\
& \texttt{SPECTRE}+TOPN+WCE & $75.3 \pm 1.1$ & $60.5 \pm 16.3$ & $39.5 \pm 16.3$ & $73.2 \pm 1.1$ & $4.3 \pm 2.7$ \\
& \texttt{SPECTRE}+ACC & $75.8 \pm 1.4$ & $\boldsymbol{67.1 \pm 9.0}$ & $\boldsymbol{31.5 \pm 9.3}$ & $73.0 \pm 2.6$ & $5.6 \pm 3.5$ \\
\midrule
S in & MMPF & $73.5 \pm 1.6$ & $65.0 \pm 8.7$ & $31.8 \pm 9.4$ & $71.1 \pm 3.1$ & $4.9 \pm 3.5$ \\
train & GDRO & $75.6 \pm 0.6$ & $66.4 \pm 8.7$ & $33.6 \pm 8.7$ & $73.4 \pm 1.5$ & $4.7 \pm 2.1$ \\
\& val & MMPFrel & $74.2 \pm 3.0$ & $66.4 \pm 9.6$ & $31.3 \pm 9.1$ & $72.0 \pm 4.6$ & $4.6 \pm 3.7$ \\
\end{tabular}}
\end{sc}
\end{small}
\end{center}
\end{table*}


\begin{table*}[ht!]
\caption{ Average ($\pm$ std) results on the ACSEmployment dataset for state IA, with 16 sensitive groups and the smallest group containing 6 samples. Among the methods that are blind to demographics, the best result for each metric (column) is emphasized in \textbf{bold}, and the second-best result is \underline{underlined}. }
\label{tab:results8}
\begin{center}
\begin{small}
\begin{sc}
\scalebox{0.9}{\begin{tabular}{clccc}
\toprule
& & av. acc  & w acc & max $\Delta$ acc \\
& Method & $(\uparrow)$ & $(\uparrow)$& $(\downarrow)$  \\
\midrule
Base- & LR & $80.0 \pm 0.2$ & $41.4 \pm 22.0$ & $53.5 \pm 20.9$ \\
lines & XGB & $\boldsymbol{81.3 \pm 0.3}$ & $50.1 \pm 9.3$ & $\underline{47.9 \pm 6.0}$ \\
\midrule
& RLM & $53.7 \pm 6.3$ & $28.6 \pm 5.9$ & $58.8 \pm 11.6$ \\
& ARL & $79.3 \pm 0.9$ & $31.0 \pm 17.7$ & $65.0 \pm 23.9$ \\
 & BPF & $68.1 \pm 1.9$ & $34.6 \pm 5.8$ & $55.1 \pm 13.9$ \\
 w/o & SURE & $78.6 \pm 0.5$ & $25.7 \pm 23.2$ &  $74.3 \pm 23.2$ \\
Demo- & FairEns & $46.2 \pm 9.9$ & $18.6 \pm 15.4$ & $61.1 \pm 24.9$ \\
graphics & \texttt{SPECTRE}+WCE & $\underline{80.6 \pm 0.3}$ & $51.5 \pm 9.7$ & $48.5 \pm 9.7$ \\
&  \texttt{SPECTRE}+WCE+T+A & $80.5 \pm 0.3$ & $\underline{51.7 \pm 10.2}$ & $48.3 \pm 10.2$ \\
&  \texttt{SPECTRE}+TOPN+WCE & $80.4 \pm 0.3$ & $51.7 \pm 10.2$ & $48.3 \pm 10.2$ \\
&  \texttt{SPECTRE}+ACC & $80.3 \pm 0.4$ & $\boldsymbol{52.4 \pm 7.9}$ & $\boldsymbol{47.6 \pm 7.9}$ \\
\midrule
S in & MMPF & $72.2 \pm 0.5$ & $42.2 \pm 15.1$ & $53.4 \pm 18.6$ \\
train & GDRO & $76.6 \pm 0.4$ & $38.0 \pm 11.2$ & $59.3 \pm 12.8$ \\
\& val & MMPFrel & $72.0 \pm 0.5$ & $41.6 \pm 14.9$ & $54.0 \pm 18.2$ \\
\end{tabular}}
\end{sc}
\end{small}
\end{center}
\end{table*}


\begin{table*}[ht!]
\caption{ Average ($\pm$ std) results on the ACSIncome dataset for state AZ, with 18 sensitive groups and the smallest group containing 2 samples. Among the methods that are blind to demographics, the best result for each metric (column) is emphasized in \textbf{bold}, and the second-best result is \underline{underlined}. }
\label{tab:results9}
\begin{center}
\begin{small}
\begin{sc}
\scalebox{0.9}{\begin{tabular}{clccc}
\toprule
& & av. acc  & w acc & max $\Delta$ acc \\
& Method & $(\uparrow)$ & $(\uparrow)$& $(\downarrow)$  \\
\midrule
Base- & LR & $\underline{81.3 \pm 0.4}$ & $35.5 \pm 30.7$ & $62.1 \pm 33.2$ \\
lines & XGB & $\boldsymbol{81.6 \pm 0.2}$ & $41.2 \pm 35.2$ & $55.5 \pm 36.6$ \\
\midrule
& RLM & $47.1 \pm 3.5$ & $19.4 \pm 16.9$ & $71.3 \pm 12.3$ \\
& ARL & $81.1 \pm 0.5$ & $37.6 \pm 31.9$ & $60.0 \pm 32.7$ \\
 & BPF & $78.8 \pm 0.7$ & $42.2 \pm 21.3$ & $54.1 \pm 23.1$ \\
 w/o & SURE & $81.1 \pm 0.4$ & $34.7 \pm 29.6$ & $60.8 \pm 30.3$ \\
Demo- & FairEns & $46.5 \pm 10.5$ & $8.0 \pm 10.7$ & $62.2 \pm 12.8$ \\ 
graphics & \texttt{SPECTRE}+WCE & $76.9 \pm 0.4$ & $43.4 \pm 22.1$ & $50.0 \pm 25.6$ \\
& \texttt{SPECTRE}+WCE+T+A & $77.2 \pm 0.4$ & $\underline{62.2 \pm 10.3}$ & $\underline{35.4 \pm 11.1}$ \\
& \texttt{SPECTRE}+TOPN+WCE & $76.9 \pm 0.8$ & $51.4 \pm 26.8$ & $45.1 \pm 24.0$ \\
& \texttt{SPECTRE}+ACC & $77.2 \pm 0.1$ & $\boldsymbol{65.7 \pm 10.9}$ & $\boldsymbol{34.3 \pm 10.9}$ \\
\midrule
S in & MMPF & $75.9 \pm 0.3$ & $27.8 \pm 29.8$ & $68.2 \pm 30.7$ \\
train & GDRO & $79.7 \pm 0.3$ & $45.3 \pm 23.5$ & $51.3 \pm 26.2$ \\
\& val & MMPFrel & $75.4 \pm 0.2$ & $31.1 \pm 28.2$ & $65.5 \pm 29.4$ \\
\end{tabular}}
\end{sc}
\end{small}
\end{center}
\end{table*}

\begin{table*}[ht!]
\caption{ Average ($\pm$ std) results on the ACSIncome dataset for state RI, with 15 sensitive groups and the smallest group containing 1 sample. Among the methods that are blind to demographics, the best result for each metric (column) is emphasized in \textbf{bold}, and the second-best result is \underline{underlined}. }
\label{tab:results10}
\begin{center}
\begin{small}
\begin{sc}
\scalebox{0.9}{\begin{tabular}{clccc}
\toprule
& & av. acc  & w acc & max $\Delta$ acc \\
& Method & $(\uparrow)$ & $(\uparrow)$& $(\downarrow)$  \\
\midrule
Base- & LR & $\underline{78.7 \pm 0.6}$ & $57.0 \pm 16.3$ & $43.0 \pm 16.3$ \\
lines & XGB & $\boldsymbol{79.2 \pm 0.3}$ & $55.2 \pm 27.9$ & $44.8 \pm 27.9$ \\
\midrule
& RLM & $53.3 \pm 2.4$ & $14.7 \pm 18.6$ & $82.0 \pm 17.1$ \\
& ARL & $78.6 \pm 0.7$ & $50.5 \pm 17.7$ & $49.5 \pm 17.7$ \\
 & BPF & $71.4 \pm 0.7$ & $35.9 \pm 20.7$ & $61.8 \pm 22.6$ \\
 w/o & SURE & $75.9 \pm 1.2$ & $44.8 \pm 27.8$ & $55.2 \pm 27.8$ \\
Demo- & FairEns & $ 50.1 \pm 7.1$ & $18.6 \pm 22.8$ & $76.4 \pm 20.9$ \\
graphics & \texttt{SPECTRE}+WCE & $73.8 \pm 0.9$ & $65.3 \pm 4.0$ & $34.7 \pm 4.0$ \\
 & \texttt{SPECTRE}+WCE+T+A & $76.5 \pm 0.3$ & $\underline{66.6 \pm 4.8}$ & $\underline{33.4 \pm 4.8}$ \\
& \texttt{SPECTRE}+TOPN+WCE & $74.9 \pm 1.9$ & $\boldsymbol{66.8 \pm 6.1}$ & $\boldsymbol{33.2 \pm 6.1}$ \\
& \texttt{SPECTRE}+ACC & $75.4 \pm 0.9$ & $63.4 \pm 8.9$ & $36.6 \pm 8.9$ \\
\midrule
S in & MMPF & $77.4 \pm 0.8$ & $64.2 \pm 8.2$ & $34.8 \pm 8.3$ \\
train & GDRO & $78.2 \pm 0.5$ & $65.1 \pm 8.0$ & $34.9 \pm 8.0$ \\
\& val & MMPFrel & $76.7 \pm 0.8$ & $63.8 \pm 8.7$ & $35.1 \pm 8.2$ \\
\end{tabular}}
\end{sc}
\end{small}
\end{center}
\end{table*}

\clearpage

\subsection{Statistical Test}
\label{ap:statistical_test}

To gain further insight into the empirical results reported, we also conduct a statistical significance analysis. The primary objective is to determine whether the improvements observed in the worst-group accuracy metric, when comparing \texttt{SPECTRE} against SOTA methods that operate without demographics, are statistically significant across the considered ACS datasets and tasks. We follow a two-step, non-parametric methodology suitable for comparing multiple classifiers on multiple datasets. First, we apply the Friedman test \cite{friedman1937use} to test the null hypothesis that all methods perform equivalently. This determines if a statistically significant performance difference exists among the set of methods being compared. In case the Friedman test reveals there is sufficient evidence to reject the null hypothesis, we then perform the Nemenyi post-hoc test \cite{nemenyi1963distribution}. This test provides a critical difference (CD) value, allowing us to identify and study the statistical differences for all pairwise comparisons between methods. In both tests, we adopt a significance level of $\alpha = 0.05$.

\begin{table*}[ht]
\centering
\caption{Nemenyi post-hoc test p-values for worst-group accuracy ($\alpha = 0.05$).  Statistically significant differences (p$< 0.05$) are shown in \textbf{bold}.}
\label{tab:nemenyi_worst_group}
\begin{tabular}{@{\hspace{0pt}}lccccccc}
\toprule
Method  & LR & XGBoost & RLM & ARL & BPF & SURE & FAIREns \\
\midrule
\texttt{SPECTRE} & \textbf{0.005} & \textbf{0.012} & \textbf{0.000} & \textbf{0.000} & \textbf{0.000} & \textbf{0.000} & \textbf{0.000}  \\
LR & -- & 1.000 & \textbf{0.000} & 0.762 & \textbf{0.034} & 0.847 & \textbf{0.000}  \\
XGBoost & -- & -- & \textbf{0.000} & 0.608 & \textbf{0.016} & 0.713 & \textbf{0.000}  \\
RLM &  -- &  -- & -- & \textbf{0.000} & 0.084 & \textbf{0.000} & 1.000  \\
ARL &  -- &  -- &  -- & -- & 0.762 & 1.000 & \textbf{0.001}  \\
BPF &  -- &  -- &  -- &  -- & -- & 0.661 & 0.183  \\
SURE &  -- &  -- &  -- &  -- &  -- & -- & \textbf{0.001}  \\
\bottomrule
\end{tabular}
\end{table*}

\subsection{Additional fairness metrics}
\label{ap:other_fairness_metrics}

Although the primary objective of developing \texttt{SPECTRE} is to provide strong minimax fairness guarantees without relying on demographic information, we also evaluate its performance with respect to other fairness metrics—namely, equality of opportunity (EOp) and statistical parity (SP) (see Figure \ref{fig:other_metrics_results}). The results indicate that \texttt{SPECTRE} attains the second-lowest EOp and SP values (lower is better), closely following FairEns, among the baseline and fairness-enhancing methods that operate without sensitive information. However, it is important to note that FairEns exhibits an overall accuracy approximately 30\% lower than that of \texttt{SPECTRE}. This observation underscores a key limitation of parity-based fairness metrics: parity can sometimes be achieved through uniformly poor performance, meaning that low disparity does not necessarily imply equitable or effective outcomes. Therefore, it is essential to assess algorithmic performance from multiple perspectives and not rely solely on parity-based measures.
In general, the SP and EOp metrics exhibit higher variability because the studied methods were designed to optimize fairness in terms of worst-group accuracy. Since different fairness metrics can conflict, this trade-off is expected.

\begin{figure}[h!] 
    \centering
    \subfigure[]{
    \includegraphics[width=0.85\textwidth]{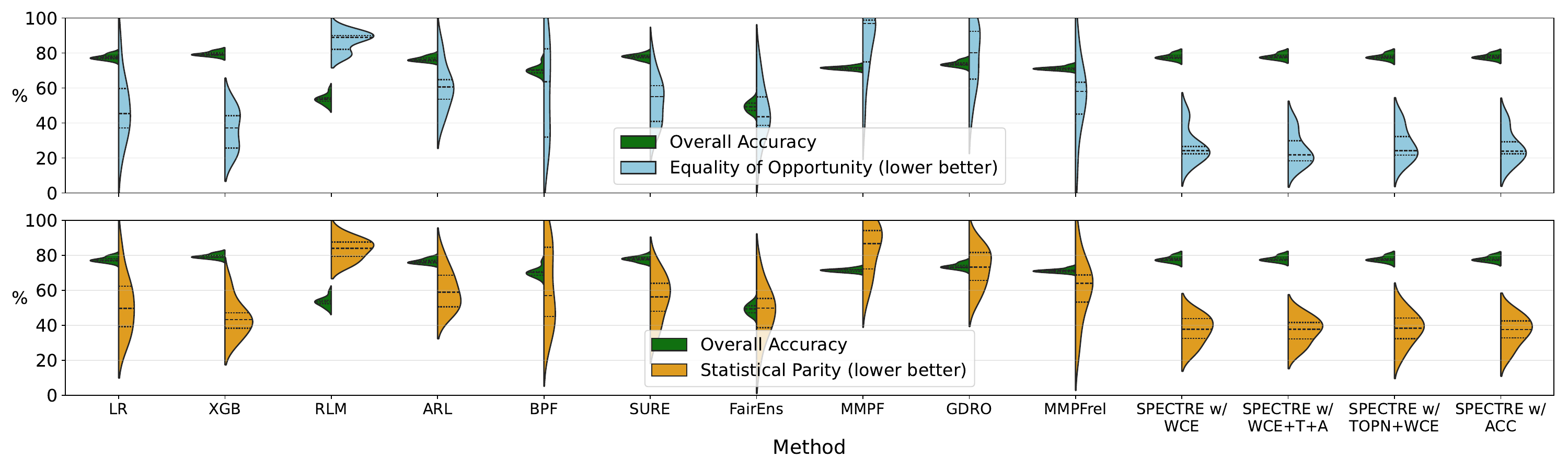}
    \label{fig:ACSEmp_race_other_metrics}
  }
  \hfill 
  \subfigure[]{
    \includegraphics[width=0.85\textwidth]{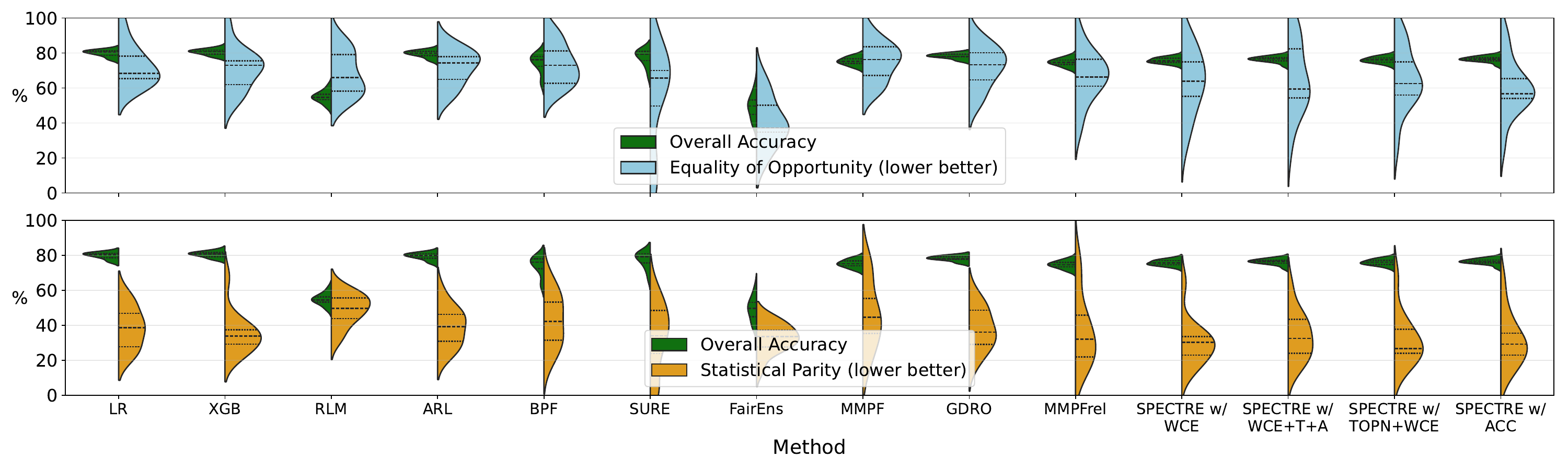}
    \label{fig:ACSInc_race_other_metrics}
  }
  \hfill 
  \subfigure[]{
    \includegraphics[width=0.85\textwidth]{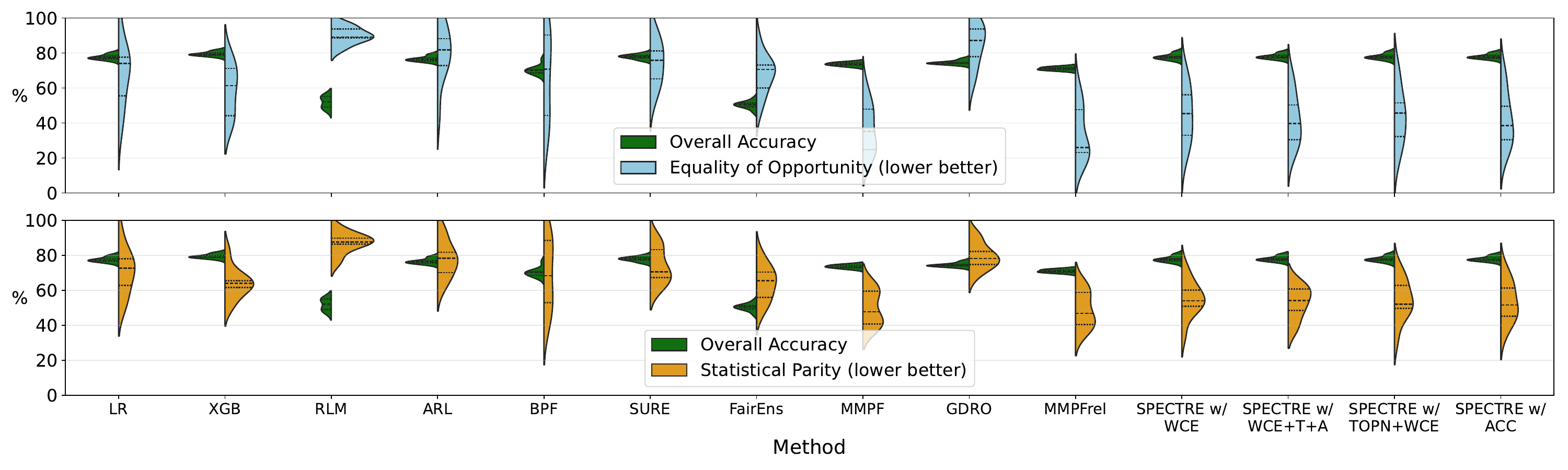}
    \label{fig:ACSEmp_int_other_metrics}
  }
  \hfill 
  \subfigure[]{
    \includegraphics[width=0.85\textwidth]{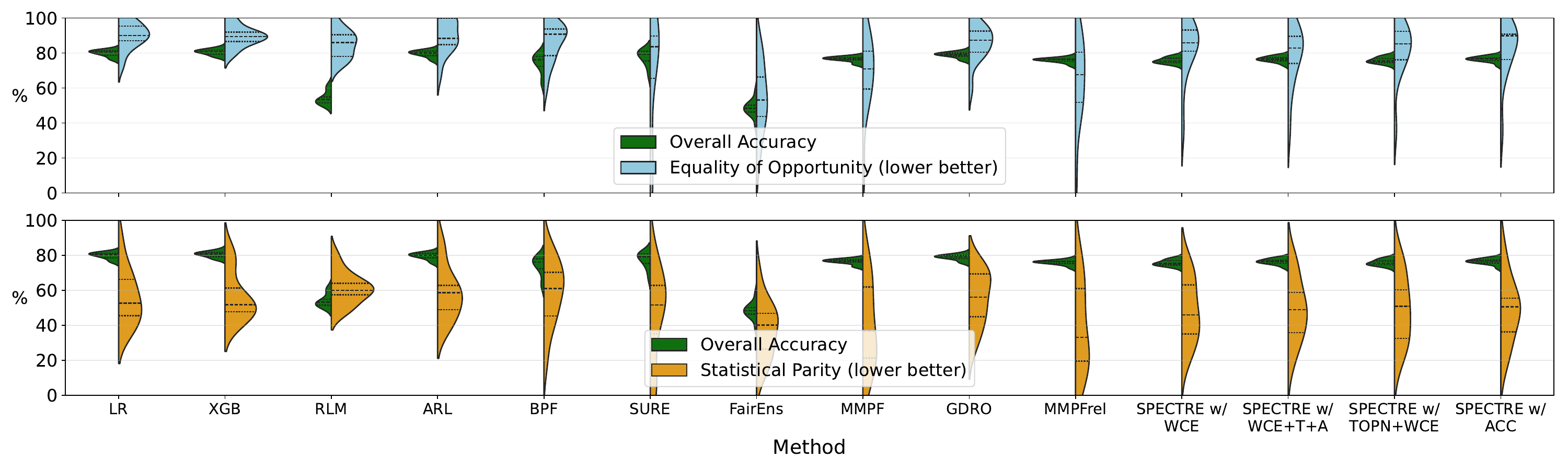}
    \label{fig:ACSInc_int_other_metrics}
  }
    \caption{The distributions of the average \color{OliveGreen} \textbf{overall accuracy} \color{black},  \color{SkyBlue} \textbf{equality of opportunity} \color{black} and \color{YellowOrange} \textbf{statistical parity} \color{black} of SOTA approaches and \texttt{SPECTRE} in (a,c) ACSEmployment and (b,d) ACSIncome datasets, across different collections of 20 randomly selected US states, with (a,b) \textit{race} as the sensitive attribute and (c,d) considering intersectional groups based on \textit{race} and \textit{gender}.
    }
    \label{fig:other_metrics_results}
\end{figure}

\clearpage

\subsection{The Effect of $\lambda_0$}
\label{ap:lambda_effect}

SPECTRE has two hyperparameters: $\sigma$ and $\lambda_0$. In Section~\ref{sec:sigma_toy}, we focused on the effect of $\sigma$, as it directly determines the frequency spectrum of the new representation and, consequently, the uncertainty set. To complement that analysis, we conduct additional experiments to study the impact of $\lambda_0$ on performance for both the toy dataset (Figure~\ref{fig:lambda_exp_toy}) and real-world datasets (Figure~\ref{fig:lambda_exp_ACS}).

Results on the toy dataset show that the influence of $\lambda_0$ diminishes as the quality of the representation (i.e., the choice of $\sigma$) improves. When $\sigma$ yields good performance, variations in $\lambda_0$ have little effect. However, when $\sigma$ is suboptimal, larger deviations from the empirical distribution (higher $\lambda_0$ values) can negatively impact performance.

This pattern is consistent in the real-world experiments (Figure~\ref{fig:lambda_exp_ACS}), where once $\sigma$ is optimally selected, the effect of $\lambda_0$ remains minor but non-negligible. Therefore, while $\sigma$ has a more substantial influence—justifying why we optimize it first in Algorithm~\ref{alg:spectre}—tuning $\lambda_0$ remains important to mitigate potential performance degradation, particularly when $\sigma$ is not ideally chosen.

\begin{figure}[h!] 
    \centering
    \subfigure[]{
    \includegraphics[width=0.6\textwidth]{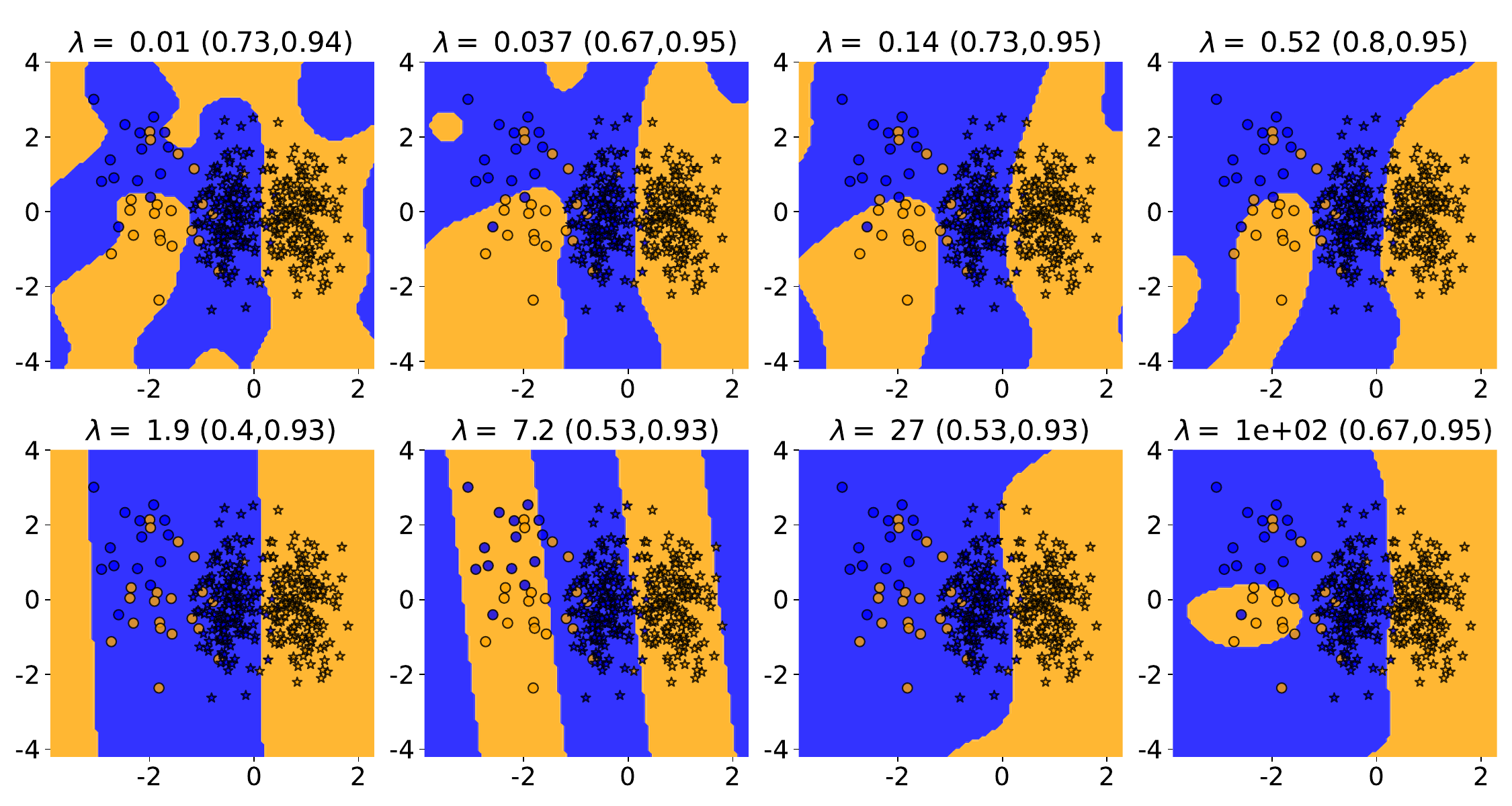}
    \label{fig:lambda_DB_toy}
  }
  \hfill 
  \subfigure[]{
    \includegraphics[width=0.45\textwidth]{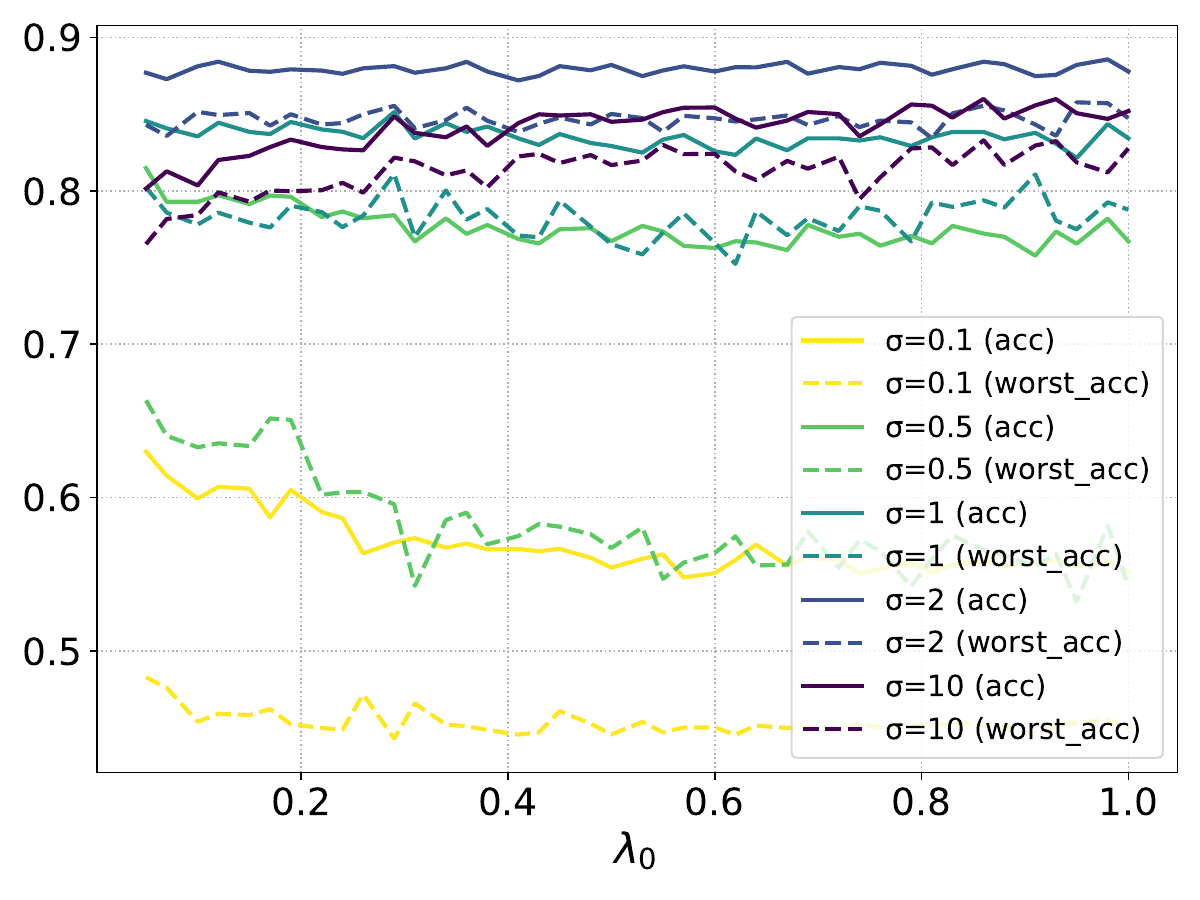}
    \label{fig:lambda_perf_toy}
  }
    \caption{(a) Decision boundary and (b) average performance (in terms of accuracy -acc- and minimax fairness -worst\_acc-) of the resulting classifier over 10 runs for different values of $\lambda_0$. In (a), we fix $\sigma = 1.0$, one of the optimal values, to illustrate the decision boundary. In (b), we show performance curves for multiple $\sigma$ values. 
    }
    \label{fig:lambda_exp_toy}
\end{figure}

\begin{figure}[!h]
\begin{center}
\centerline{\includegraphics[width=0.5\textwidth]{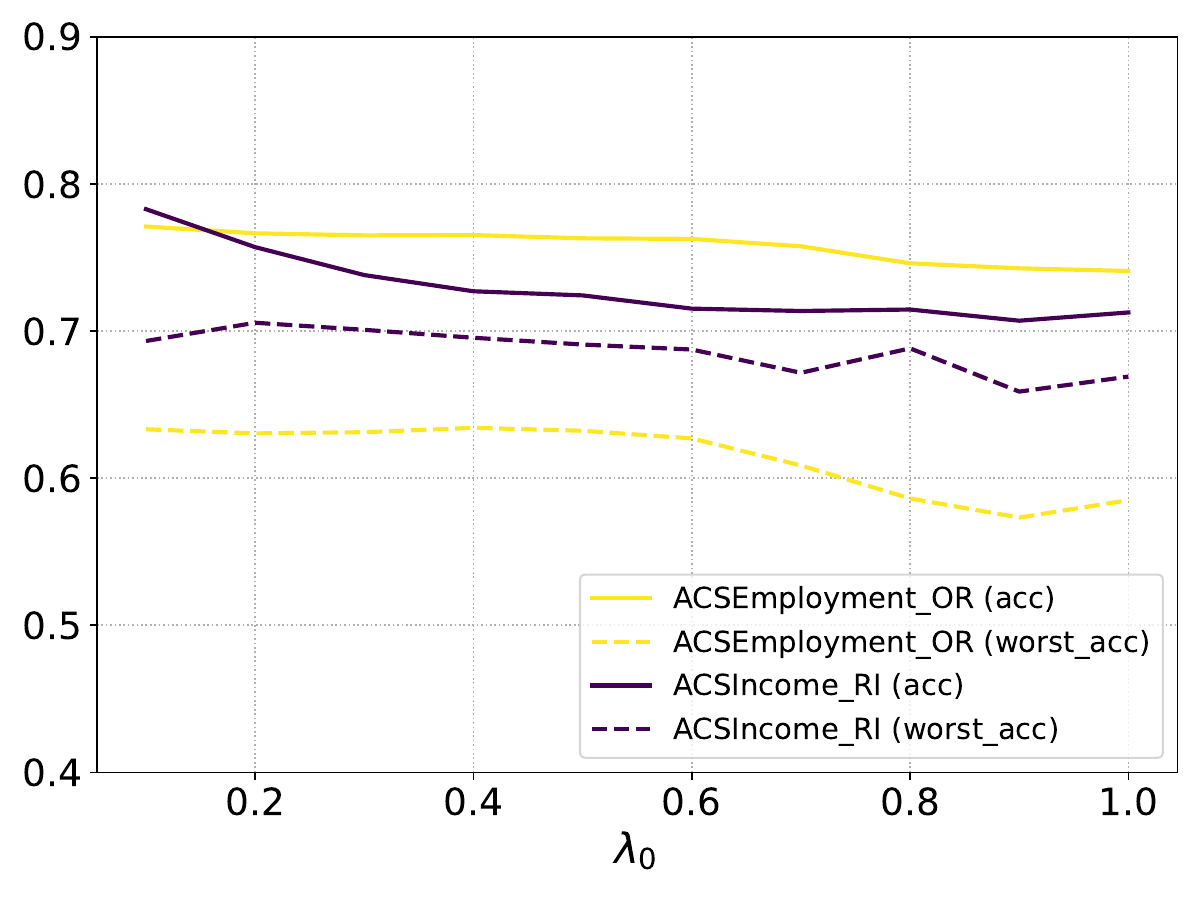}}
\caption{Average performance (in terms of acc and worst\_acc) of the resulting classifier across different values of $\lambda_0$, averaged over 10 repetitions, using the optimal $\sigma$ selected through hyperparameter tuning, for two different datasets.  }
\label{fig:lambda_exp_ACS}
\end{center}
\end{figure}

\section{Empirical Evaluation of the Out-of-Sample Guarantees}

\subsection{Real-World Datasets}
\label{ap:bounds_real}

In this section, we evaluate the bounds using real-world datasets (see Figure \ref{fig:bounds_acs}). Specifically, we employ the ACSIncome dataset, focusing on the states of Hawaii (HI) and Alaska (AK). To compute the bounds, we use 20\% of the data for each group while ensuring that each group contains at least 40 instances. That is, groups that do not retain at least 40 instances after the 20\% reduction are not taken into account in this analysis. As a result, in Alaska (AK), we analyze four racial groups:  
\begin{itemize}
    \item \( S = 1 \) with 424 instances  
    \item \( S = 4 \) with 155 instances  
    \item \( S = 6 \) with 40 instances  
    \item \( S = 9 \) with 48 instances  
\end{itemize}
For Hawaii (HI), three racial groups meet the criteria:  
\begin{itemize}
    \item \( S = 1 \) with 414 instances  
    \item \( S = 6 \) with 577 instances  
    \item \( S = 9 \) with 339 instances  
\end{itemize}

\begin{figure}[h!] 
    \centering
    \subfigure[]{
    \includegraphics[width=0.45\textwidth]{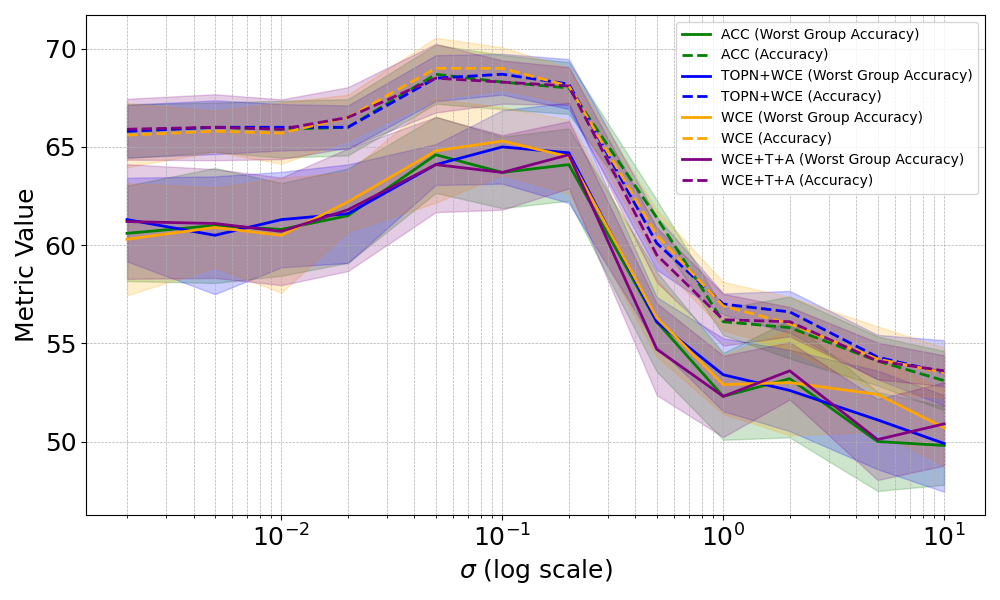}
    \label{fig:bounds_acs_s1}
  }
  \hfill 
  \subfigure[]{
    \includegraphics[width=0.49\textwidth]{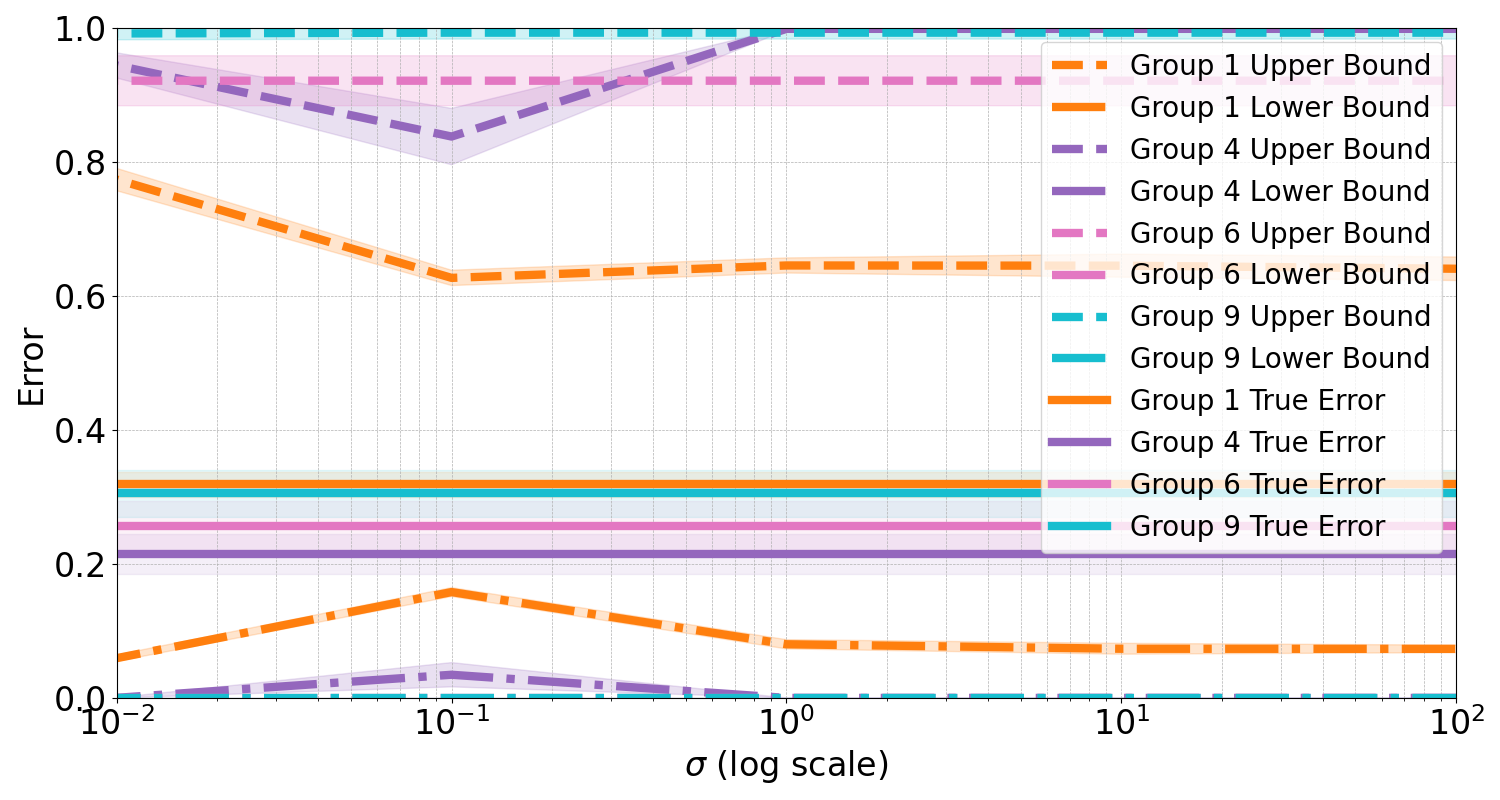}
    \label{fig:bounds_acs_s2}
  }
  \hfill 
  \subfigure[]{
    \includegraphics[width=0.45\textwidth]{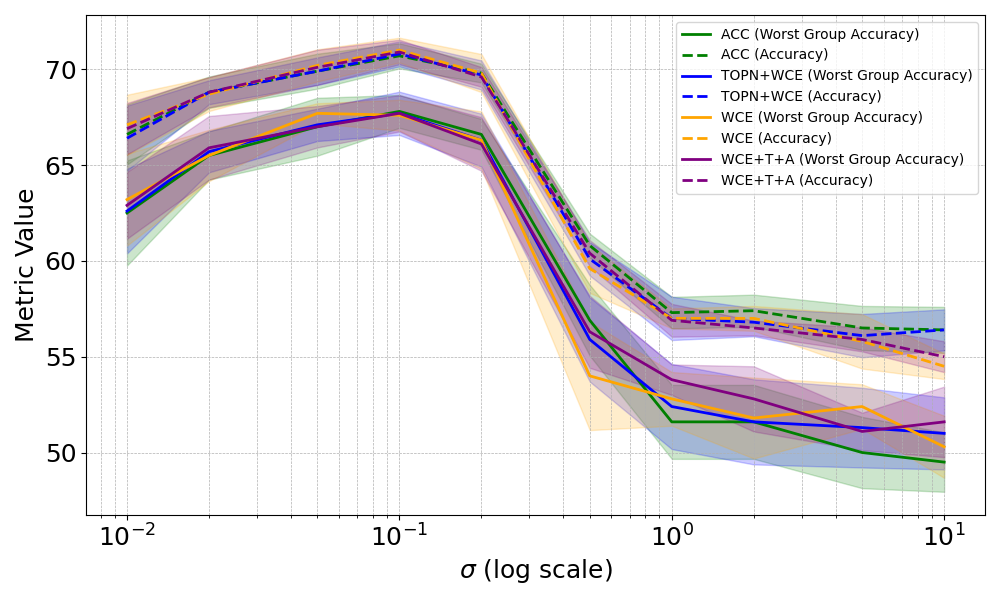}
    \label{fig:bounds_acs_s3}
  }
  \hfill 
  \subfigure[]{
    \includegraphics[width=0.49\textwidth]{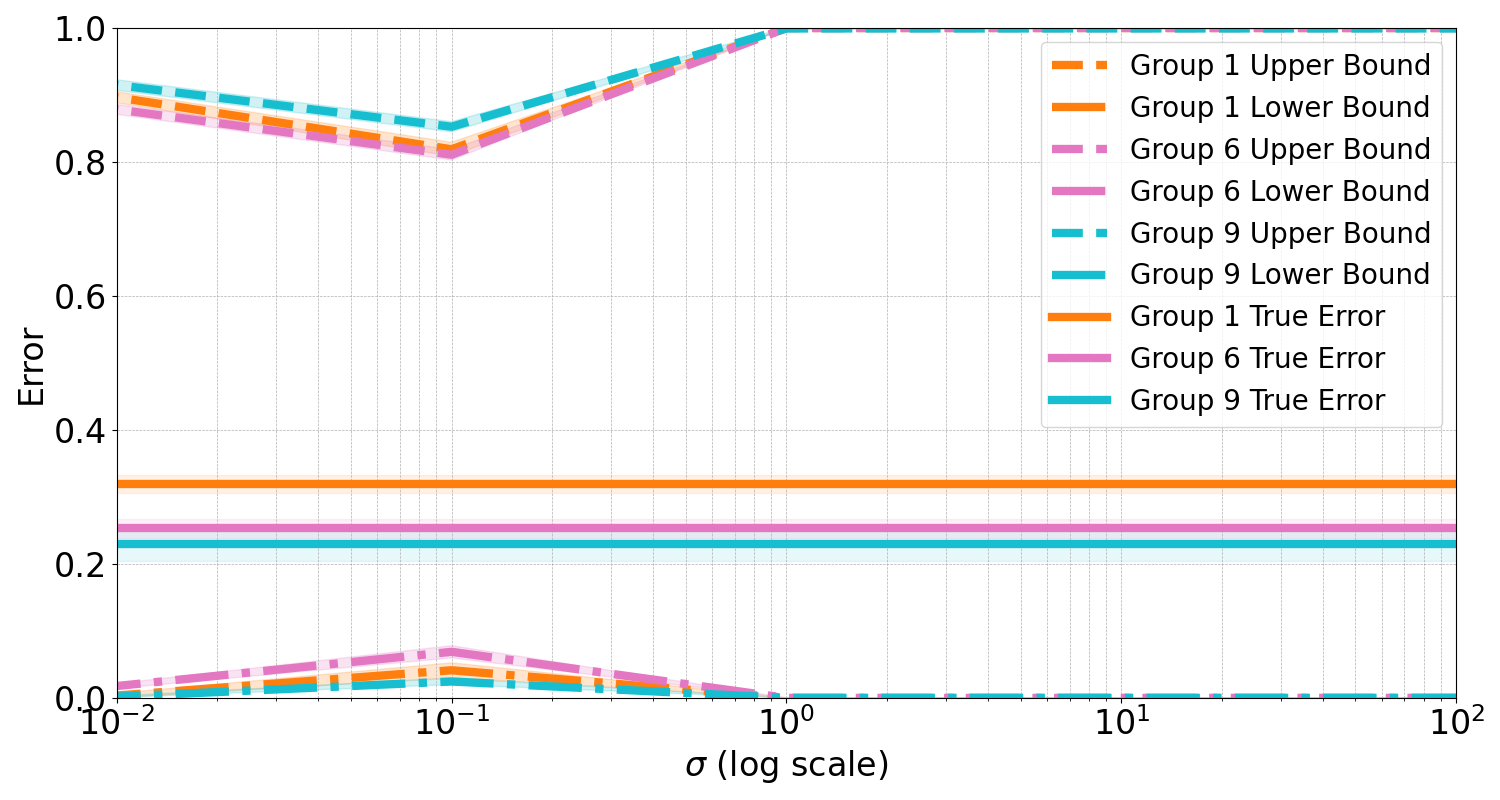}
    \label{fig:bounds_acs_s4}
  }
    \caption{ The overall and worst-group accuracy (left) and the performance bounds of \texttt{SPECTRE} on the different sensitive groups (right) for different states. There is a correspondence between the optimal $\sigma$ for \texttt{SPECTRE} and the $\sigma$ that provides the tightest bounds.
    }
    \label{fig:bounds_acs}
\end{figure}

On the one hand, these results further validates the conclusions drawn from the toy dataset in Section \ref{sec:experiments_bounds}: (a) when a small portion of the training data includes available sensitive information, it is possible to reliably estimate a feasible range for \texttt{SPECTRE}'s performance across different sensitive groups; and (b) the values that produce the tightest bounds closely align with the near-optimal $\sigma$ values for \texttt{SPECTRE}.

Furthermore, the results indicate that, within a given state and at the optimal $\sigma$ (the value that yields the tightest bounds), larger group sizes correspond to tighter bounds. In AK, Group 1 has the tightest bounds, followed by Group 4, whereas in HI, all groups exhibit similar bounds.  A key insight emerges when comparing Group 1 in AK to the groups in HI. Despite having a similar group size and true error, Group 1 in AK has significantly tighter bounds than any group in HI. This difference arises from the relative size of the group within the overall dataset. In AK, Group 1 constitutes a substantially larger proportion of the dataset compared to other groups. In contrast, in HI, no single group dominates—each of the three groups represents a similar proportion of the dataset.  The proportion of a group relative to the entire dataset plays a crucial role in determining the flexibility of worst-case distribution shifts. When a group makes up a significant majority (as Group 1 does in AK), modifying its distribution while still satisfying overall constraints is difficult. However, for smaller groups, substantial variations in their distributions have a limited impact on the overall dataset, as the majority groups can compensate to maintain the required constraints on the overall distribution.  In HI, since all three groups have comparable proportions, altering their distributions has a similar effect across the board, leading to similar bound tightness. Additionally, because their relative proportions within the full HI dataset are smaller than that of Group 1 in AK, more variations can be induced in their distributions while still adhering to the overall constraints. This explains why the groups in HI have looser bounds compared to Group 1 in AK.

\section{Using RFF with Existing Baseline and Blind Fairness-Enhancing Interventions}
\label{ap:RFF_sota}

In this section, we replicate the experiments from Section~\ref{sec:sigma_toy}, applying RFFs and hyperparameter tuning to other SOTA methods: LR, RLM, and ARL. For RLM and ARL, we consider both linear and non-linear versions. The results are presented in Figures~\ref{fig:toy_BD_LR}–\ref{fig:toy_BD_ARL} and they demonstrate that while employing spectrum adjustment and hyperparameter tuning with different strategies modifies the decision boundaries, it does not lead to a significant positive impact (in terms of minimax fairness) when applied to LR, RLM and ARL. This limitation arises primarily from their training framework. In LR, performance is not considerably improved because worst-case configurations are not considered during optimization. In RLM and ARL, their re-weighting nature, combined with the use of steep surrogate losses during training, significantly limits their effectiveness. As a result, spectrum adjustment does not meaningfully improve their robustness or reduce pessimism. This limitation is particularly evident as even substantial variations in spectrum adjustment result in minimal impact on classifier performance. 

\begin{figure}[h!] 
    \centering
    \subfigure[]{
    \includegraphics[width=0.25\textwidth]{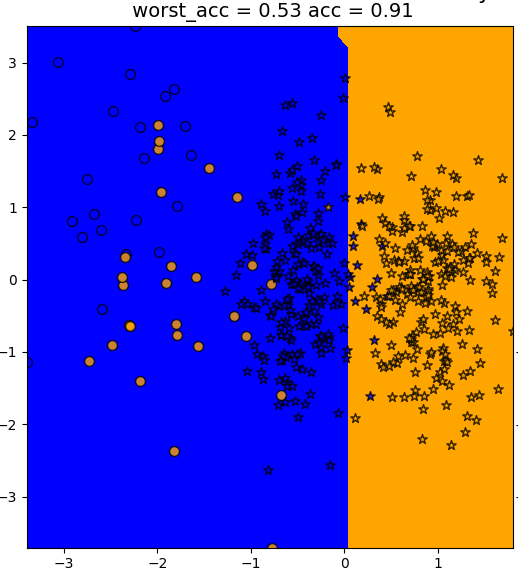}
    \label{fig:DB_LR_subplot1}
  }
  \hfill 
  \subfigure[]{
    \includegraphics[width=0.9\textwidth]{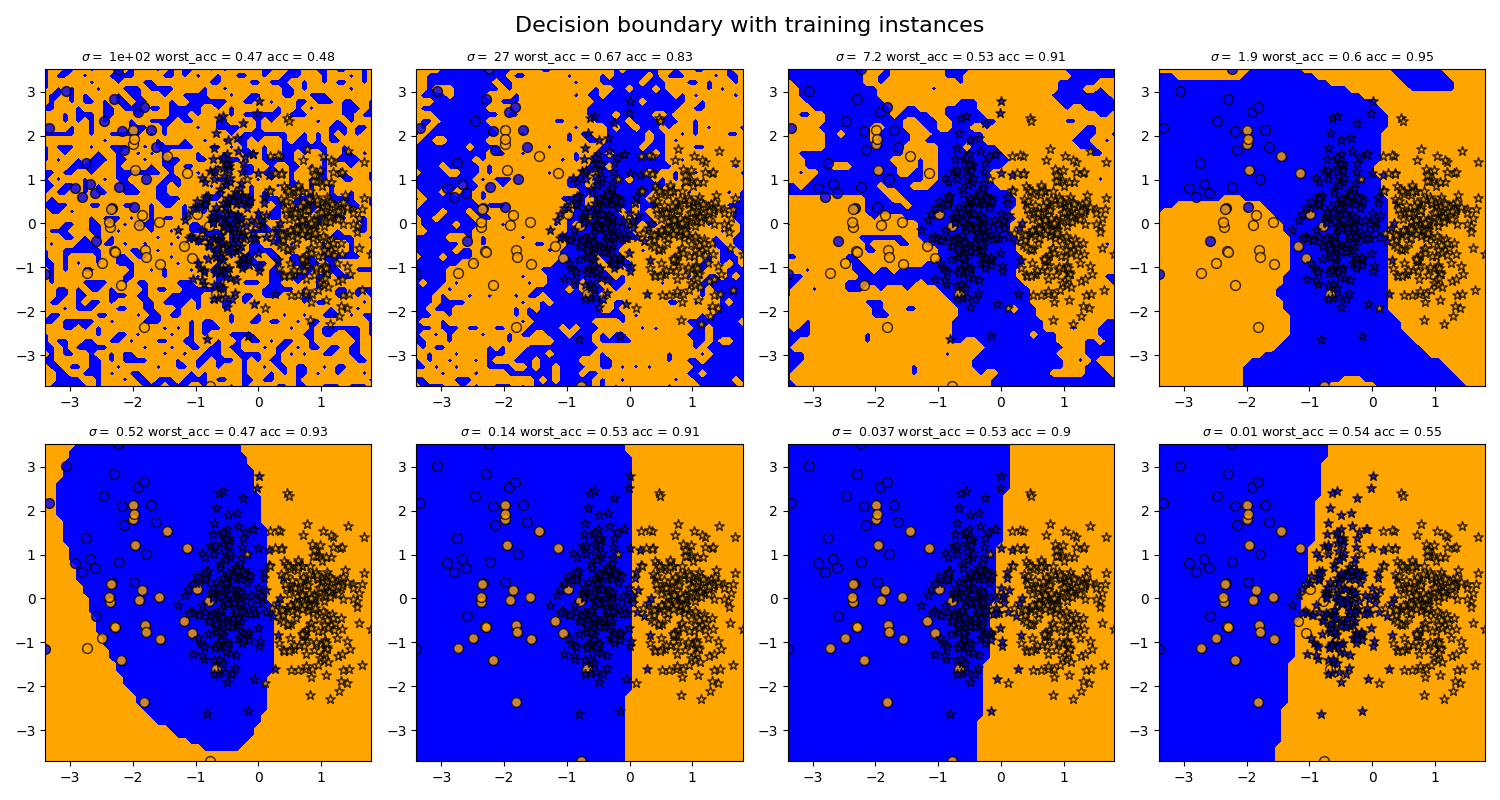}
    \label{fig:DB_LR_subplot2}
  }
  \hfill 
  \subfigure[]{
    \includegraphics[width=0.45\textwidth]{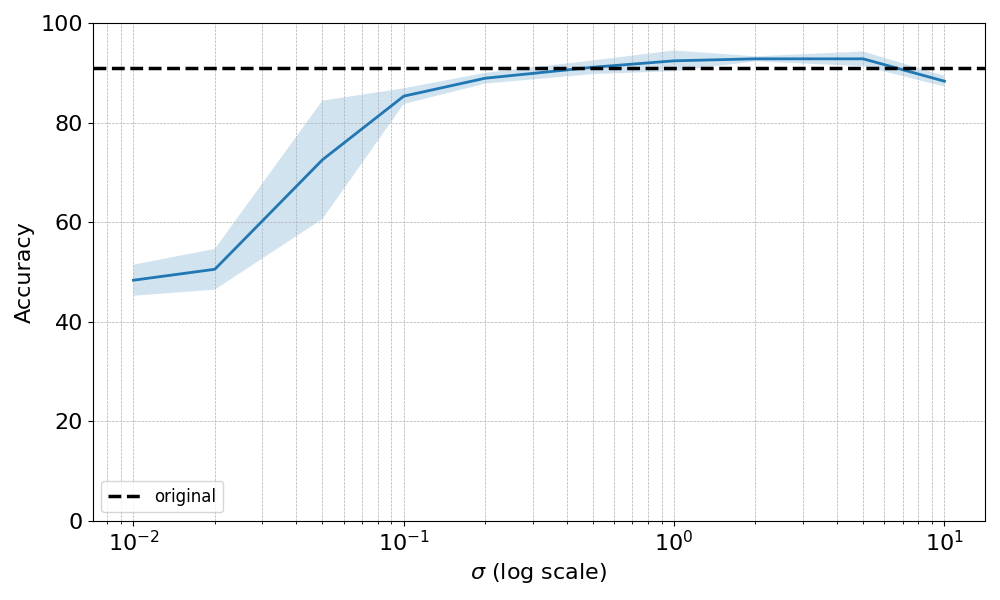}
    \label{fig:DB_LR_subplot3}
  }
  \hfill 
  \subfigure[]{
    \includegraphics[width=0.45\textwidth]{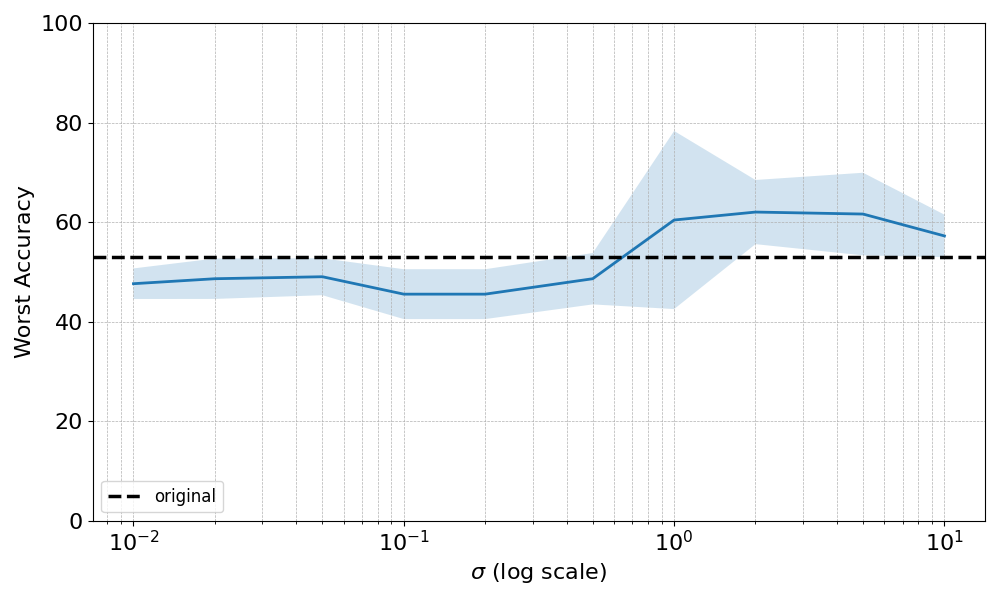}
    \label{fig:DB_LR_subplot4}
  }
    \caption{ Logistic Regression (LR). The (a) original decision boundary obtained with the LR method and the effect of the scaling parameter on (b) the decision boundaries (where training instances are also shown), (c) overall accuracy and (d) worst-group accuracy.
    }
    \label{fig:toy_BD_LR}
\end{figure}

\begin{figure}[h!] 
    \centering
    \subfigure[]{
    \includegraphics[width=0.25\textwidth]{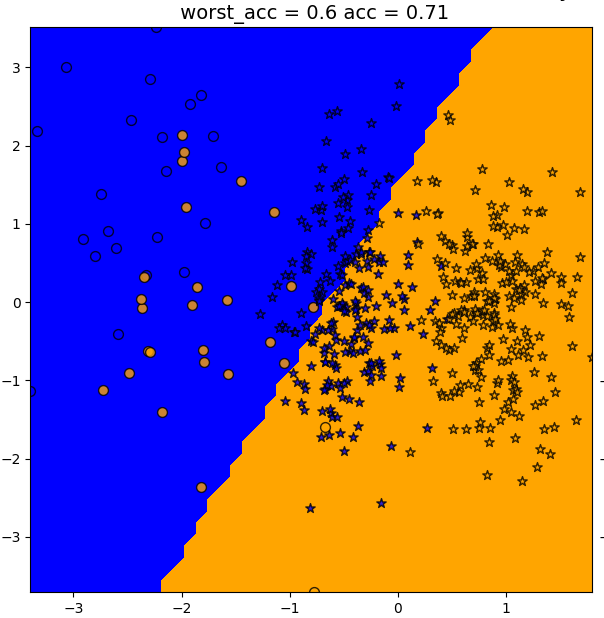}
    \label{fig:DB_linRLM_subplot1}
  }
  \hfill 
  \subfigure[]{
    \includegraphics[width=0.9\textwidth]{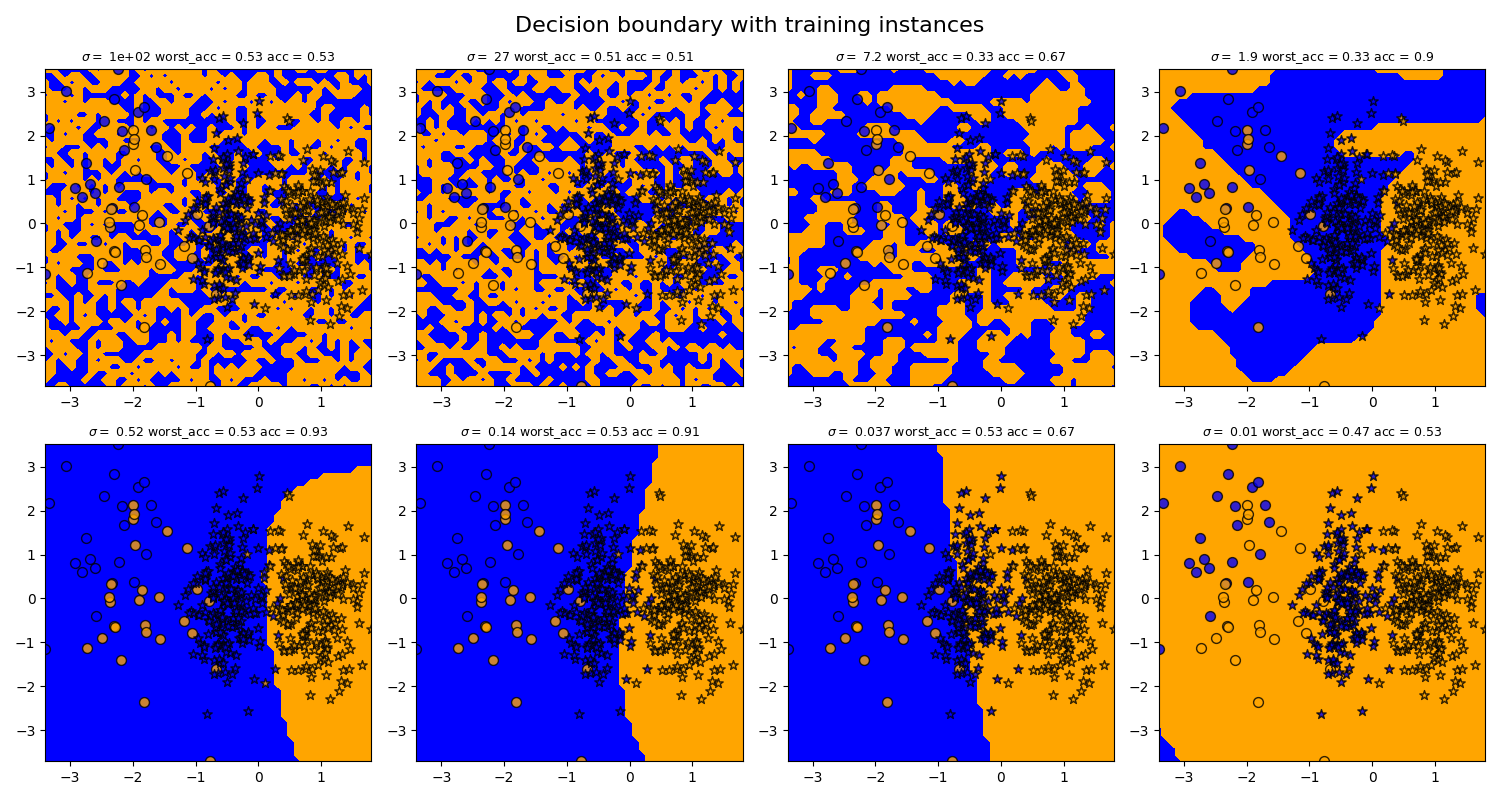}
    \label{fig:DB_linRLM_subplot2}
  }
  \hfill 
  \subfigure[]{
    \includegraphics[width=0.45\textwidth]{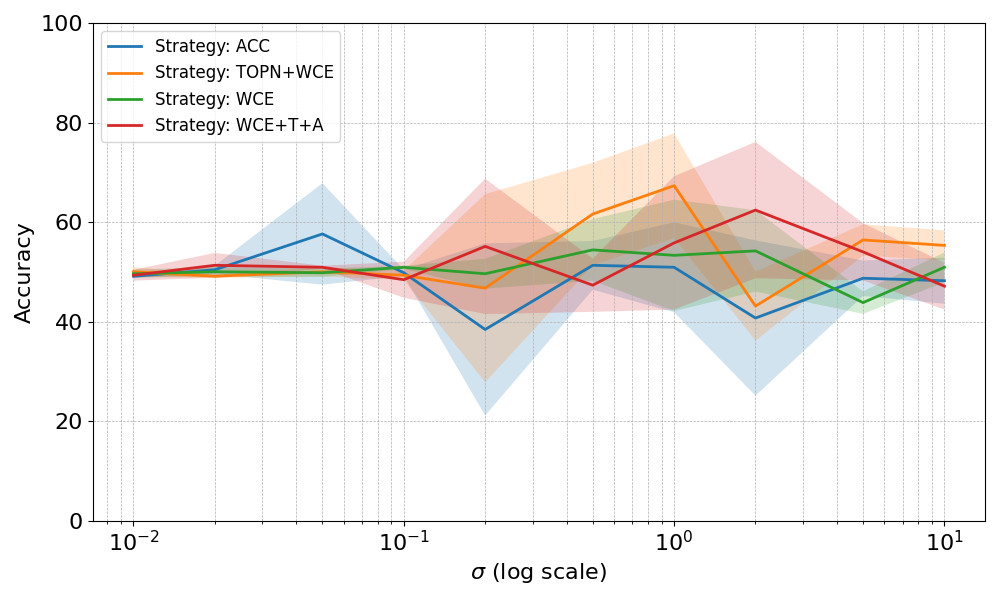}
    \label{fig:DB_linRLM_subplot3}
  }
  \hfill 
  \subfigure[]{
    \includegraphics[width=0.45\textwidth]{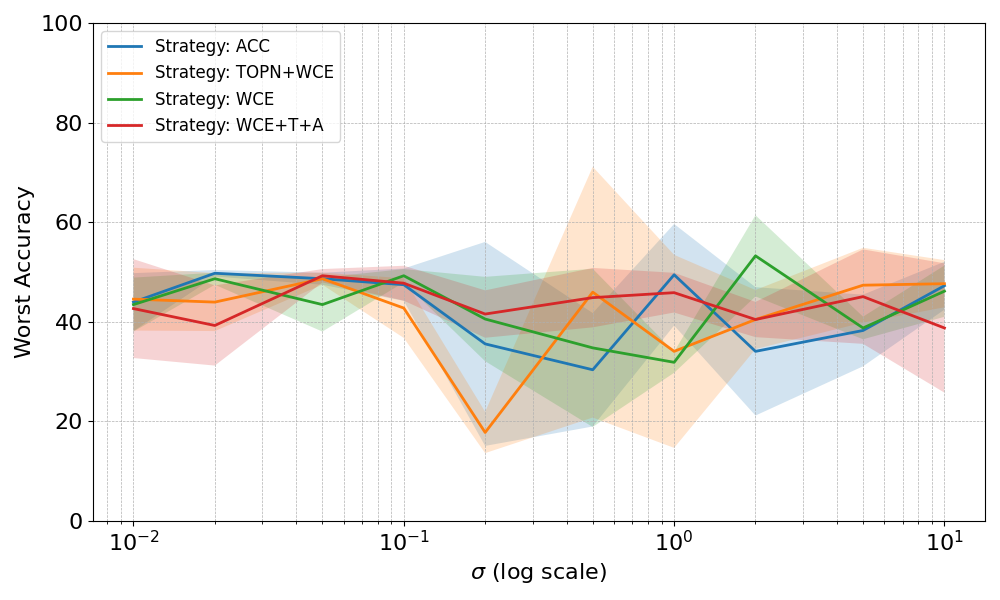}
    \label{fig:DB_linRLM_subplot4}
  }
    \caption{ Linear RLM. The (a) decision boundary obtained with the original RLM and the effect of the scaling parameter on (b) the decision boundaries (where training instances are also shown), (c) overall accuracy and (d) worst-group accuracy. For the decision boundaries we set $\eta = 0.6$ fixed. When studying the worst-group and overall accuracy we consider different strategies to get the value of $\eta$.
    }
    \label{fig:toy_BD_linRLM}
\end{figure}

\begin{figure}[h!] 
    \centering
    \subfigure[]{
    \includegraphics[width=0.25\textwidth]{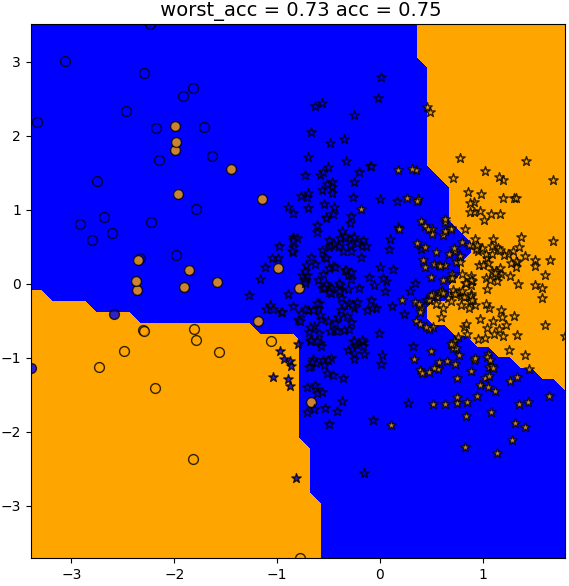}
    \label{fig:DB_RLM_subplot1}
  }
  \hfill 
  \subfigure[]{
    \includegraphics[width=0.9\textwidth]{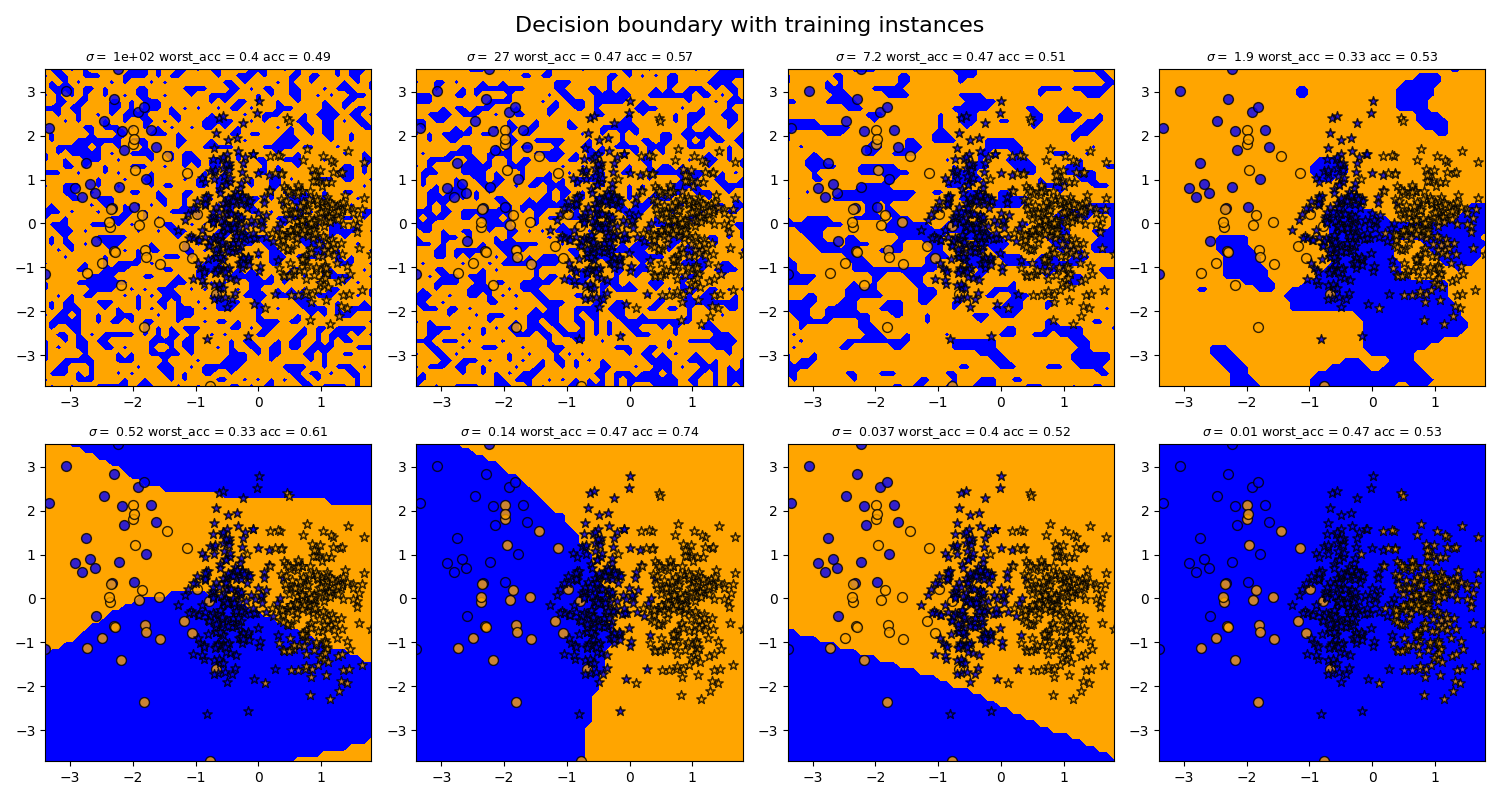}
    \label{fig:DB_RLM_subplot2}
  }
  \hfill 
  \subfigure[]{
    \includegraphics[width=0.45\textwidth]{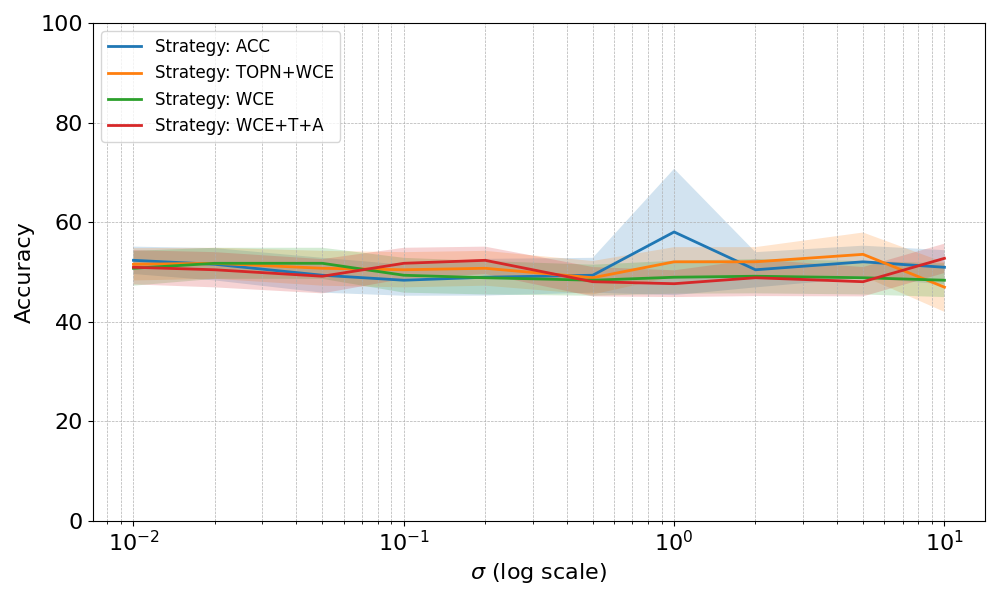}
    \label{fig:DB_RLM_subplot3}
  }
  \hfill 
  \subfigure[]{
    \includegraphics[width=0.45\textwidth]{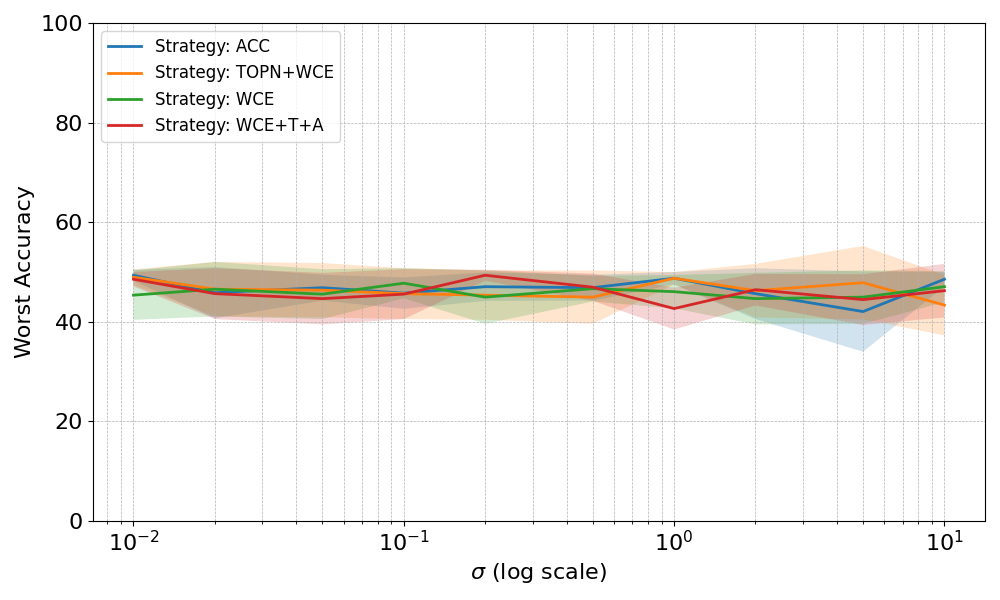}
    \label{fig:DB_RLM_subplot4}
  }
    \caption{ Non-linear RLM. The (a) decision boundary obtained with the original RLM and the effect of the scaling parameter on (b) the decision boundaries (where training instances are also shown), (c) overall accuracy and (d) worst-group accuracy. For the decision boundaries we set $\eta = 0.6$ fixed. When studying the worst-group and overall accuracy we consider different strategies to get the value of $\eta$.
    }
    \label{fig:toy_BD_RLM}
\end{figure}

\begin{figure}[h!] 
    \centering
    \subfigure[]{
    \includegraphics[width=0.25\textwidth]{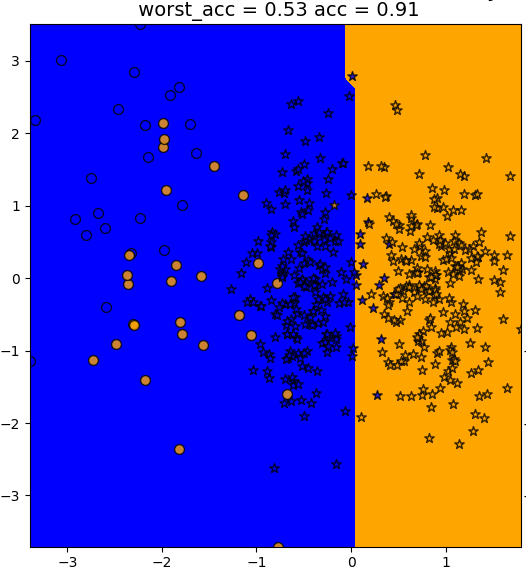}
    \label{fig:DB_linARL_subplot1}
  }
  \hfill 
  \subfigure[]{
    \includegraphics[width=0.9\textwidth]{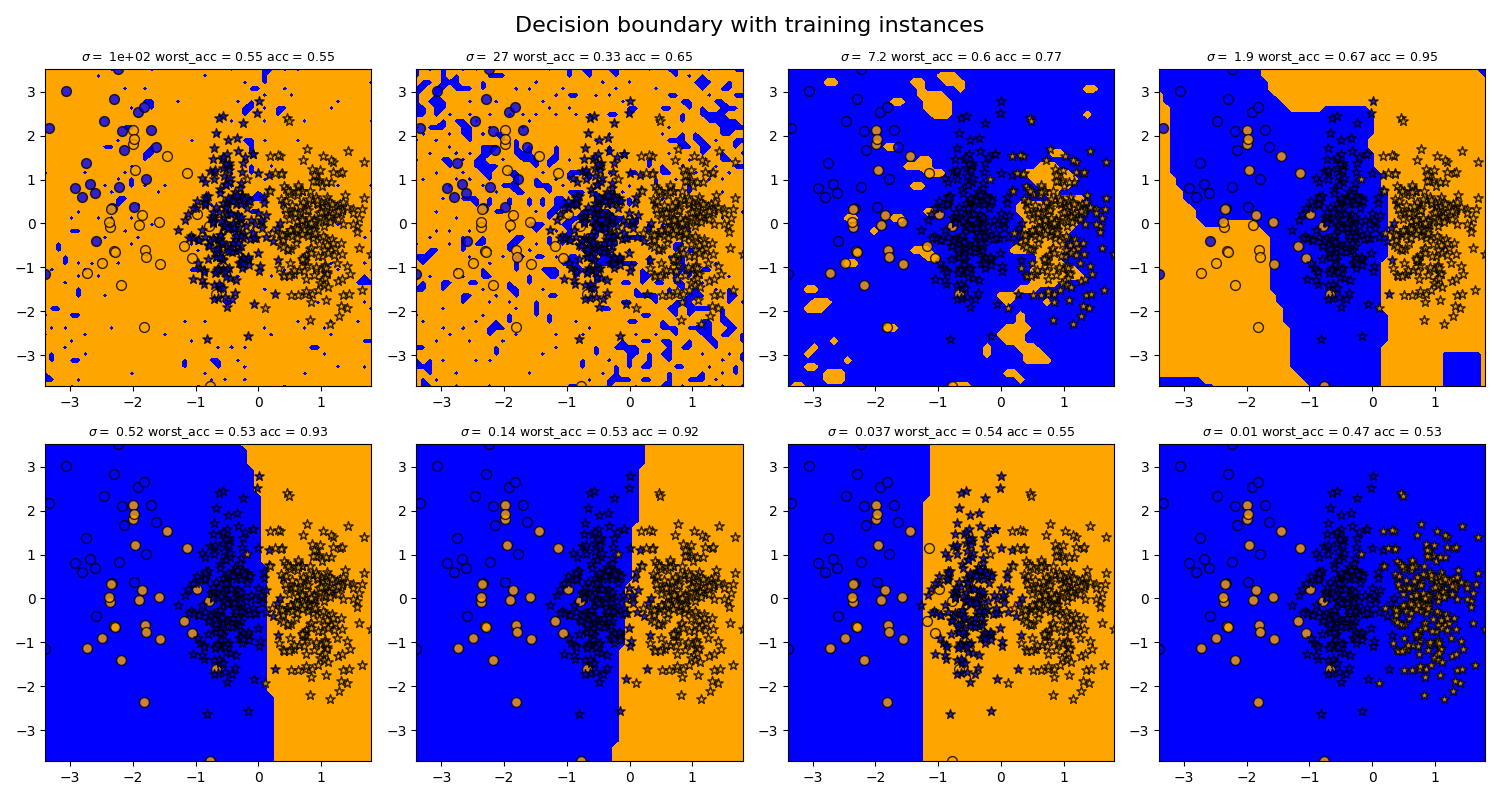}
    \label{fig:DB_linARL_subplot2}
  }
  \hfill 
  \subfigure[]{
    \includegraphics[width=0.45\textwidth]{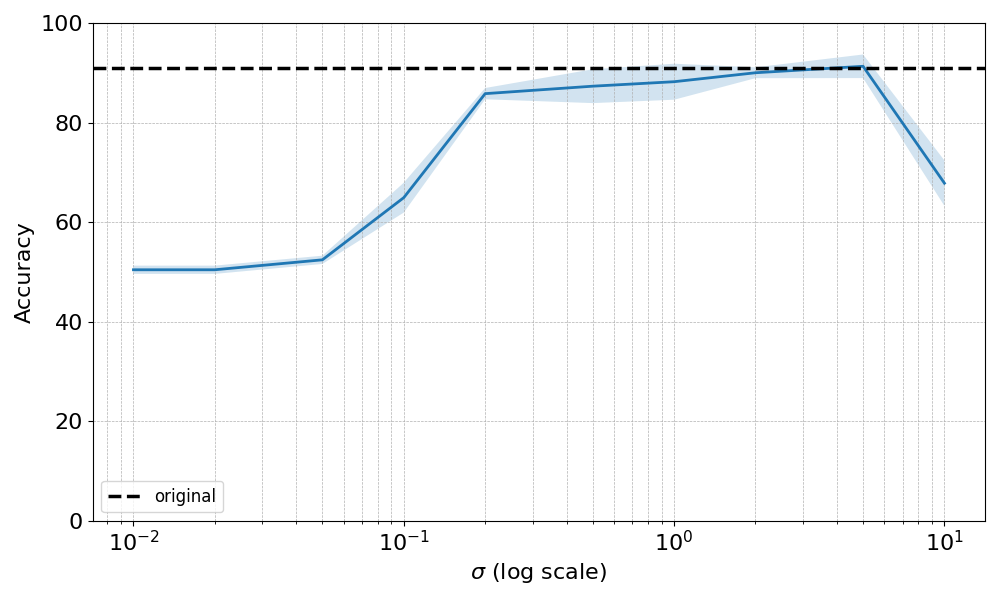}
    \label{fig:DB_linARL_subplot3}
  }
  \hfill 
  \subfigure[]{
    \includegraphics[width=0.45\textwidth]{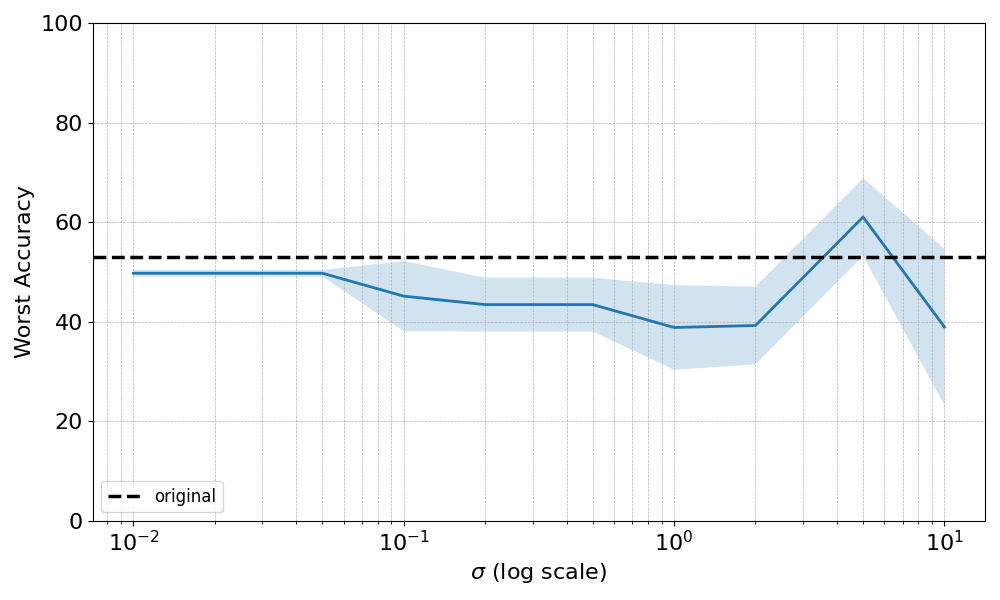}
    \label{fig:DB_linARL_subplot4}
  }
    \caption{ Linear ARL. The (a) decision boundary obtained with the linear ARL trained with original instances and the effect of the scaling parameter on (b) the decision boundaries (where training instances are also shown), (c) overall accuracy and (d) worst-group accuracy. For the decision boundaries we set $\eta = 0.6$ fixed. When studying the worst-group and overall accuracy we consider different strategies to get the value of $\eta$.
    }
    \label{fig:toy_BD_linARL}
\end{figure}

\begin{figure}[h!] 
    \centering
    \subfigure[]{
    \includegraphics[width=0.25\textwidth]{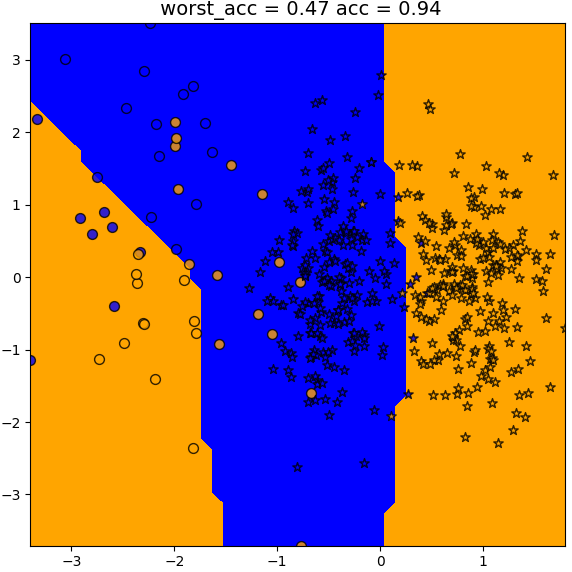}
    \label{fig:DB_ARL_subplot1}
  }
  \hfill 
  \subfigure[]{
    \includegraphics[width=0.9\textwidth]{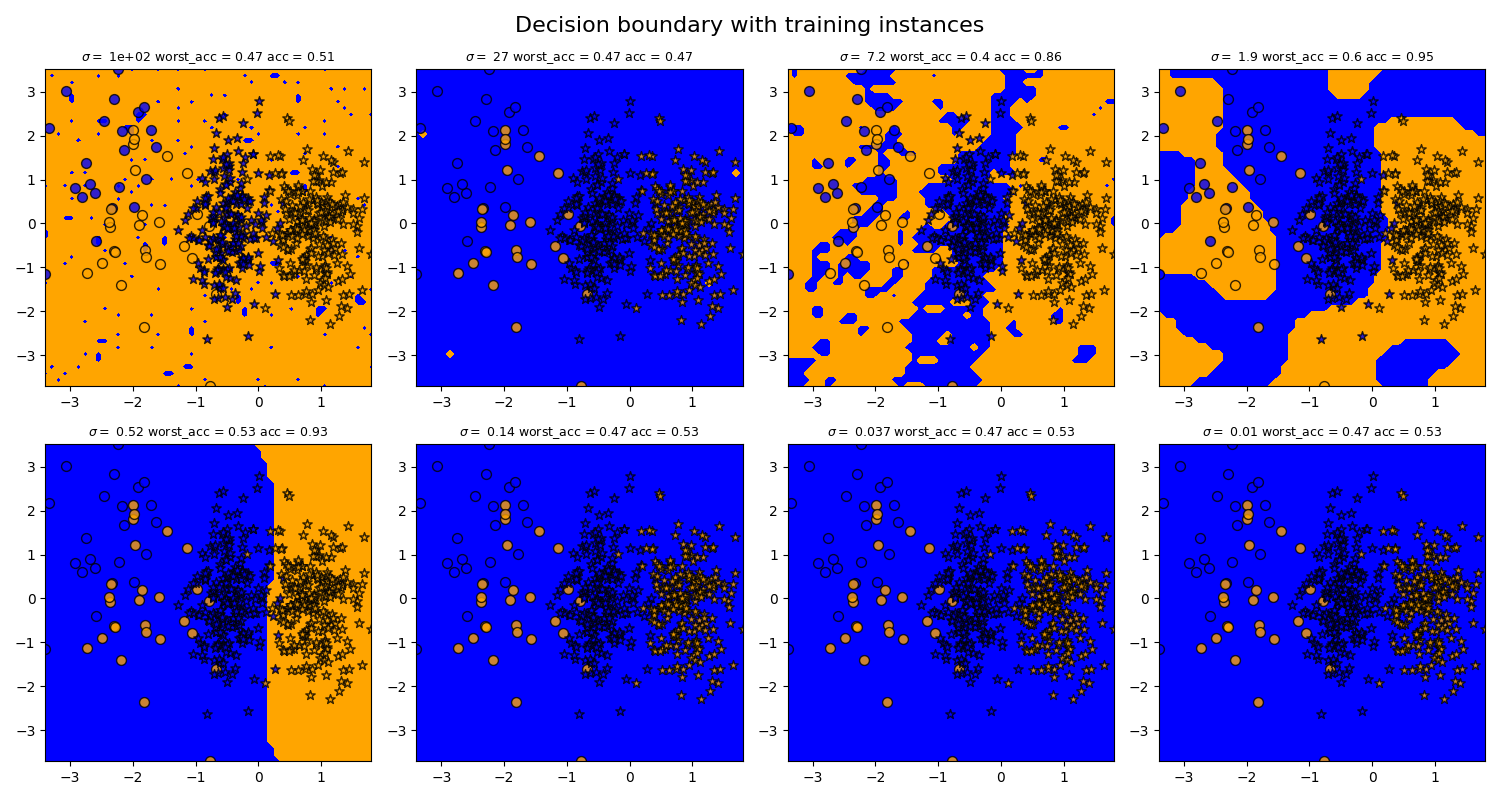}
    \label{fig:DB_ARL_subplot2}
  }
  \hfill 
  \subfigure[]{
    \includegraphics[width=0.45\textwidth]{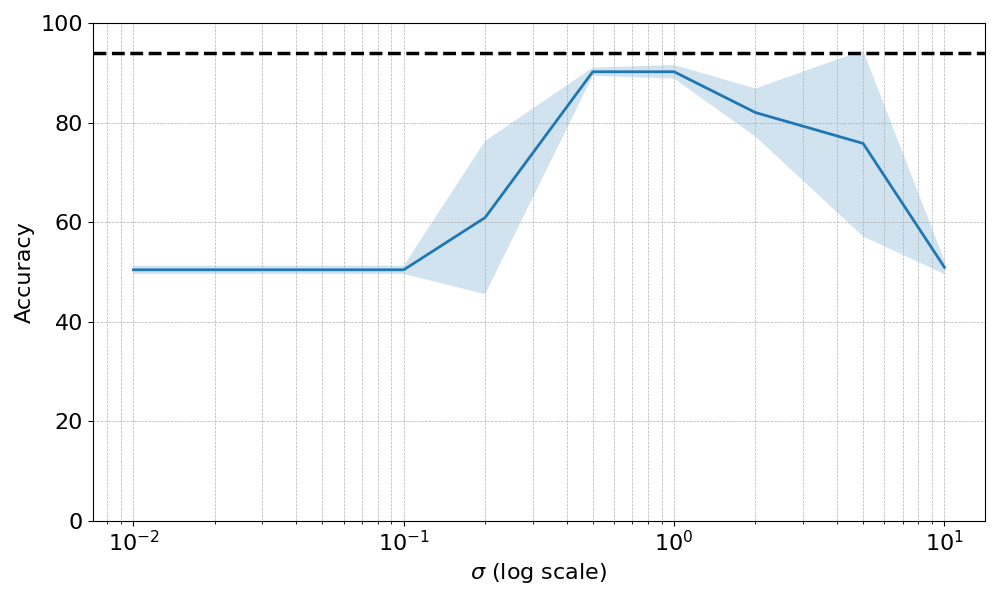}
    \label{fig:DB_ARL_subplot3}
  }
  \hfill 
  \subfigure[]{
    \includegraphics[width=0.45\textwidth]{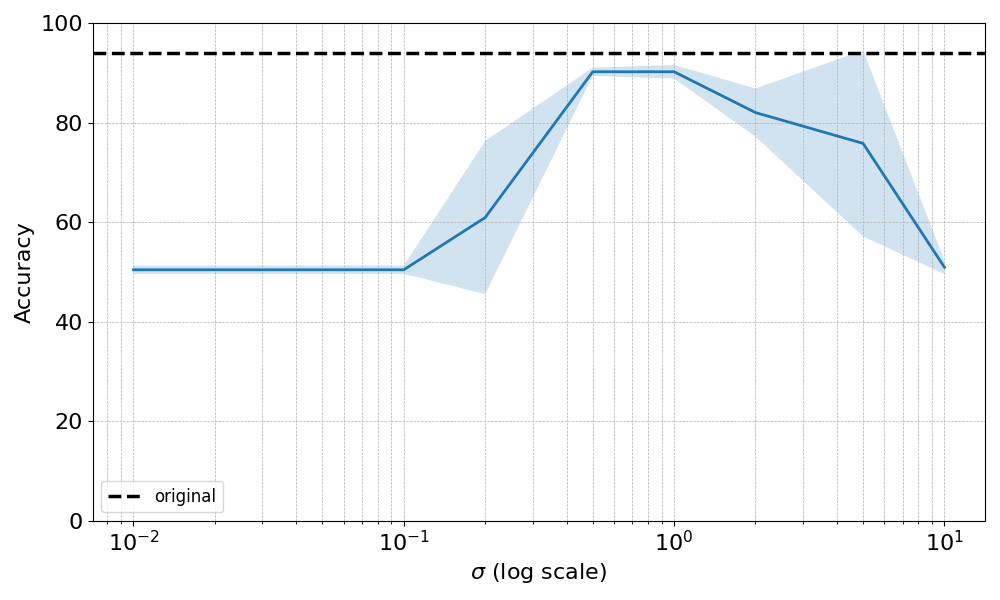}
    \label{fig:DB_ARL_subplot4}
  }
    \caption{ ARL. The (a) decision boundary obtained with the original ARL and the effect of the scaling parameter on (b) the decision boundaries (where training instances are also shown), (c) overall accuracy and (d) worst-group accuracy. For the decision boundaries we set $\eta = 0.6$ fixed. When studying the worst-group and overall accuracy we consider different strategies to get the value of $\eta$. 
    }
    \label{fig:toy_BD_ARL}
\end{figure}

\end{document}